\documentclass[11pt]{article}

\usepackage{amsthm,amsmath,bbm,amssymb,multirow,natbib,graphicx,makecell,subfigure,booktabs,array,fullpage,url,mathtools,wrapfig,lipsum,mathrsfs,dsfont,titling,epstopdf,bm,bbm,relsize,smile,caption,color,algorithm,enumitem}

\usepackage[noend]{algpseudocode}
\usepackage[useLove2s,dvipsLove2s,svgLove2s,table]{xcolor}
\usepackage[colorlinks=true,
linkcolor=blue,
urlcolor=blue,
citecolor=blue]{hyperref}

\setcellgapes{5pt}

\DeclareFontFamily{U}{mathx}{\hyphenchar\font45}
\DeclareFontShape{U}{mathx}{m}{n}{
	<5> <6> <7> <8> <9> <10>
	<10.95> <12> <14.4> <17.28> <20.74> <24.88>
	mathx10
}{}
\DeclareSymbolFont{mathx}{U}{mathx}{m}{n}
\DeclareFontSubstitution{U}{mathx}{m}{n}
\DeclareMathAccent{\widecheck}{0}{mathx}{"71}

\newtheorem{conj}{Conjecture}
\newtheorem{thm}[conj]{Theorem}
\newtheorem{cor}[conj]{Corollary}
\newtheorem{prop}[conj]{Proposition}
\newtheorem{lemma}[conj]{Lemma}
\newtheorem{ass}{Assumption}

{\endinnercustomass}
\theoremstyle{remark}\newtheorem{remark}{Remark}
\theoremstyle{remark}\newtheorem{example}{Example}

\def\A{\mathcal{A}}

\def\RR{\mathbb{R}}

\def\wh{\widehat}
\def\wt{\widetilde}
\def\bl{\backslash}
\def\E{\mathcal{E}}
\def\u{\mu}
\def\I{\mathcal{I}}

\def\g{\gamma}
\def\EE{\mathbb{E} }
\def\PP{\mathbb{P} }
\def\NN{\mathbb{N} }
\def\eps{\varepsilon}
\def\S{\mathcal{S}}
\def\diag{\textrm{diag}}
\def\i{\infty}
\def\r{{\infty,1}}
\def\c{{1, \infty}}
\def\H{\mathcal{H}}

\def\1{\mathbbm{1}}
\def\M{\Pi}
\def\rs{\!\!\!}
\def\KL{\textrm{KL}}
\def\oa{\overline{\alpha}}
\def\ua{\underline{\alpha}}
\def\og{\overline{\gamma}}
\def\ug{\underline{\gamma}}
\def\red{\color{red}}
\def\blue{\color{blue}}
\def\green{\color{green}}
\def\od{\overline{d}}
\def\ud{\underline{d}}

\let\emptyset\varnothing

\makeatletter
\def\BState{\State\hskip-\ALG@thistlm}
\makeatother
\title{A fast algorithm with minimax optimal guarantees for topic models with an unknown number of topics.}
\author{Xin Bing\thanks{Department of Statistical Science, Cornell University, Ithaca, NY. E-mail: \texttt{xb43@cornell.edu}.}~~~~~Florentina Bunea\thanks{Department of Statistical Science, Cornell University, Ithaca, NY. E-mail: \texttt{fb238@cornell.edu}.}~~~~~Marten Wegkamp\thanks{Department of Mathematics and Department of Statistical Science, Cornell University, Ithaca, NY. E-mail: \texttt{mhw73@cornell.edu}. }}

\begin{document}
	\maketitle
	
	\begin{abstract}
		\noindent
				Topic models have become popular for the analysis of data that consists in a collection of n independent multinomial observations, with parameters $N_i\in \NN$ and $\Pi_i\in [0,1]^p$ for $i=1,\ldots,n$.  The model links all cell probabilities, collected in a  $p\times n$ matrix $\Pi$, via the assumption that $\Pi$ can be factorized as the product of  two nonnegative matrices $A\in [0,1]^{p\times K}$ and  $W\in [0,1]^{K\times n}$. Topic models have been originally developed in text mining, when one browses through $n$ documents, based on a dictionary of $p$ 
		words, and covering $K$ topics. In this terminology, the matrix $A$ is called the word-topic matrix, and is the main target of estimation. It can be viewed as a matrix of conditional probabilities, and it is uniquely defined, under appropriate separability assumptions, discussed in detail in this work. Notably, the unique $A$ is required to satisfy what is commonly known as the anchor word assumption, under which $A$ has an unknown  number of rows respectively proportional to the  canonical basis vectors in $\RR^K$. The indices of such rows are referred to as anchor words. Recent computationally feasible algorithms, with theoretical guarantees, 
		utilize constructively this assumption by linking the estimation of the set of anchor words with that of estimating the $K$  vertices of a simplex. This crucial step in the estimation of $A$ requires $K$ to be known, and  cannot be easily extended to the more realistic set-up when $K$ is unknown.  
		
		This work takes a different view on anchor word estimation, and on the estimation of $A$. We propose a new method of estimation in topic models, that is not a variation 
		on the existing simplex finding algorithms, 
		and that estimates $K$ from the observed data. We derive new finite sample minimax lower bounds for the estimation of $A$, as well as new upper bounds for our proposed estimator. We describe the scenarios where our estimator is minimax adaptive. Our finite sample analysis is valid for any $n, N_i, p$ and  $K$, and both $p$ and $K$ are allowed to increase with $n$, a situation not handled well by previous analyses.  
		
		We complement our theoretical results with a detailed simulation study.  We illustrate that the new algorithm is faster and more accurate than the current ones, although  we start out with a computational and theoretical disadvantage of not knowing  the  correct number of topics $K$, while we provide the competing methods with the correct value in our simulations.\\
		
		\noindent {\bf Keywords:}  { \small  Topic model,
			latent model, overlapping clustering, 
			identification,
			high dimensional estimation,
			minimax estimation,
			anchor words,
			separability,
			nonnegative matrix factorization,
			adaptive estimation} 
	\end{abstract}

	\section{Introduction}

\subsection{Background} 	
Topic models have been developed during the last two decades in   natural language processing and machine learning for discovering the themes, or  ``topics'',  that occur in a collection of documents. They have also been successfully used to explore structures in data from genetics, neuroscience and  computational social science, to name just a few areas of application.  Earlier works on versions of these models, called latent semantic indexing models,  appeared mostly in the computer science and information science literature, for instance  \cite{Deerwester,Papa98,Hofmann99, Papa}. Bayesian solutions, involving latent Dirichlet allocation models, have been introduced in \cite{BleiLDA} and MCMC-type solvers have been considered by \cite{MCMC}, to give a very limited number of earlier references.  We refer to \cite{blei-intro} for a in-depth  overview of this field. One weakness of the earlier work on topic models was of 
computational nature, which motivated further, more recent, research on the development of algorithms with polynomial running time, see, for instance, 
\cite{Anandkumar, arora2013practical,Bansal,Tracy}.
Despite these recent advances, {\em fast} algorithms leading to estimators with {\em sharp statistical properties} are still lacking, and motivates this work.  

We begin by describing the topic model, using the terminology employed for its original usage, that of text mining. It is assumed that we observe a collection of  $n$ independent documents, and that  each document is written  using the same dictionary of $p$ words.  For each document $i\in[n]:=\{1,\ldots,n\}$, we sample $N_i$ words and record their frequencies in the vector $X_i\in \RR^p$. 
It is further assumed that the probability $\Pi_{ji}$ with which a word $j$ appears in a document $i$ depends on the topics covered in the document, justifying the following informal application of the Bayes' theorem, 
\begin{align*}  
\Pi_{ji}  :=  \PP_i&(\text{Word} \ j) = \sum_{k = 1}^K \PP_i(\text{Word} \  j| \text{Topic} \ k) \PP_i(\text{Topic} \ k) . 
\end{align*} 
The topic model assumption is that the conditional probability of the occurrence of a word, given the topic, is the same for all documents. 
This leads to the topic model specification: 
\begin{align} \label{basic}
\Pi_{ji}  =   \sum_{k = 1}^K \PP(\text{Word} \  j| \text{Topic} \ k) \PP_i(\text{Topic} \ k), \qquad \text{for each }j\in [p],\ i\in [n].  
\end{align} 
We collect the above conditional probabilities in the $p\times K$ word-topic matrix $A$ and we let $W_i \in \RR^K$ denote the vector containing the probabilities of each of the $K$ topics occurring in document $i\in[n]$.  With this notation, data generated from topic models  are  observed count frequencies $X_i$ corresponding to independent 
\begin{equation}\label{model_multinomial}
Y_i := N_iX_i \sim \text{Multinomial}_p\left(N_i,  AW_i\right),\quad \text{ for each  }i\in[n].
\end{equation} 
Let $X$ be the $p \times n$ observed data matrix,  $W$ be the $K \times n$ matrix with columns $W_i$, and $\Pi$ be the $p \times n$ matrix with entries $\Pi_{ji}$ satisfying (\ref{basic}). The topic model therefore postulates that the expectation of the word-document  frequency matrix $X$ has the non-negative factorization
\begin{equation}\label{model}  
\Pi := \EE[X] = AW, 
\end{equation} 
and the goal is to borrow strength across the $n$ samples to estimate the common matrix of conditional probabilities, $A$. Since the columns in $\Pi, A$ and $W$ are probabilities specified by (\ref{basic}), they have non-negative entries and satisfy
\begin{equation}\label{col_sum_one}
\sum_{j=1}^p\M_{ji} = 1, \quad \sum_{j=1}^pA_{jk}=1, \quad \sum_{k=1}^KW_{ki} = 1, \quad \text{for any $k\in [K]$ and  $i\in[n]$.}
\end{equation}
In Section \ref{sec_iden} we discuss in detail separability conditions on $A$ and $W$ that ensure the uniqueness of the factorization in (\ref{model}). 	

In this context, the main goal of this  work is to estimate $A$ optimally, both computationally and from a minimax-rate perspective,  in identifiable topic models, with an {\em unknown} number $K$ of topics, that is allowed to depend on $n,N_i,p$.

\subsection{Outline and contributions}\label{sec_contri}

In this section we describe  the outline of this paper and give a precise summary of our results which are developed via the following overall strategy: (i) We first show that $A$ can be derived, uniquely, at the population level, from quantities that can be estimated independently of $A$. 
(ii) We use the constructive procedure in (i) for estimation, and replace population level quantities by appropriate estimates, tailored to our final goal of minimax optimal estimation of $A$ in (\ref{model}), via  fast computation.  
\paragraph{Recovery of $A$ at the population level.}

We prove in   Propositions \ref{prop_anchor} and \ref{prop_A} of Section \ref{sec_idenA}  that the target word-topic matrix  $A$ can be uniquely derived from $\Pi$, and give the resulting procedure in Algorithm \ref{alg}. The proofs require the separability Assumptions \ref{ass_sep} and \ref{ass_pd}, common in the topic model literature, when $K$ is known. 
All model assumptions are stated and discussed in Section \ref{sec_iden}, and  informally described here. Assumption \ref{ass_sep} is placed on the word-topic matrix $A$, and is known as the anchor-word assumption as it requires the existence of words that are solely associated with one topic.  In Assumption \ref{ass_pd} we require that  $W$ have full rank.

To the best of our knowledge, parameter identifiability in topic models received a limited amount of attention. If model (\ref{model}) and Assumptions \ref{ass_sep} and \ref{ass_pd} hold, 
and provided that the index set $I$ corresponding to anchor words, as well as the number of topics  $K$, are known, Lemma 3.1 of \cite{arora2012learning} shows that  $A$ can be constructed uniquely via $\Pi$. If $I$ is unknown, but $K$ is known,  Theorem 3.1 of \cite{rechetNMF} further shows that the matrices $A$ and $W$  
can  be constructed uniquely via $\Pi$, by connecting the problem of finding  $I$ with that of finding the $K$ vertices of an appropriately defined simplex.
Methods that utilize simplex structures
are common in the topic models literature, such as the simplex structure in the word-word co-occurrence matrix
\citep{arora2012learning,arora2013practical}, in the original matrix $\Pi$ \citep{ding}, and in the singular vectors of $\Pi$ \citep{Tracy}.

In this work we provide a solution to the open problem of constructing $I$, and then $A$, in topic models,  in the more realistic situation when $K$ is unknown.  For this, we develop a method that is not a  variation of the existing simplex-based constructions. Under the additional Assumption \ref{ass_w} of Section \ref{sec_iden}, but without a priori knowledge of $K$, we recover the index set  $I$ of all anchor words, as well as its partition $\mathcal I$. This constitutes   Proposition \ref{prop_anchor}. 
Our proof only requires the existence of  one anchor word for each topic, but we allow for the existence of more, as this is typically the case in practice, see for instance, \cite{blei-intro}. Our method is optimization-free. It involves comparisons between row and column maxima of a scaled version of the matrix $\Pi\Pi^T$, specifically  of the  matrix $R$ given by (\ref{def_R}). Example \ref{eg_2} of Section \ref{sec_idenA} illustrates our procedure, whereas a  contrast with simplex-based approaches is given in Remark \ref{simplex} of Section \ref{sec_idenA}. 


\paragraph{Estimation of $A$.}
In Section \ref{sec_est_A}, we follow the steps of Algorithm \ref{alg} of Section \ref{sec_idenA}, to develop Algorithms \ref{alg1} and \ref{alg2} for estimating $A$ from the data.

We show first how to construct estimators of $I$, $\mathcal{I}$ and $K$, and summarize this construction in Algorithm \ref{alg1} of Section \ref{sec_est_I}, with theoretical guarantees provided in Theorem \ref{thm_anchor}. Since we follow Algorithm \ref{alg}, this step of our estimation procedure does not involve any of the previously  used simplex recovery algorithms, such as those mentioned above. 

The estimators of $I$,  $\mathcal{I}$ and $K$  are employed in the second step of our procedure, summarized in Algorithm \ref{alg2} of Section \ref{sec_est_A}. This step yields the estimator  $\widehat{A}$ of $A$, and only requires solving a system of equations under linear restrictions, which, in turn, requires the estimation of the inverse of a matrix. For the latter, we develop a  fast and stable algorithm, tailored to this model, which reduces to solving $K$ linear programs, each optimizing over a $K$-dimensional space.  This is  less involved, computationally, than the next best competing estimator of $A$, albeit developed for $K$ known, in \cite{arora2013practical}.  After estimating $I$, their estimate of  $A$  requires solving $p$ restricted convex optimization problems, each optimizing over a $K$-dimensional parameter space. 

We assess the theoretical performance of our estimator $\wh A$ with respect to the $L_{1, \infty}$ and $L_{1}$ losses defined below, by providing finite sample lower and upper bounds on these quantities, that hold for all   $p$, $K$, $N_i$ and $n$. In particular, we allow $K$ and $p$ to grow with  $n$, as we expect that when the number of available documents $n$ increases, so will the number $K$ of topics that they cover, and possibly the number $p$  of words used in these documents.  Specifically, we let $\mathcal{H}_{K}$ denote the set of all $K \times K$ permutation matrices and define:  
\begin{align*}
\| \wh A - A\|_{\c} &:= \max_{1\le k\le K} \sum_{j=1}^p|\wh A_{jk}-A_{jk}|, \quad
&\| \wh A - A\|_1 &:= \sum_{j=1}^p\sum_{k =1}^K|\wh A_{jk}-A_{jk}|,  
\\
L_\c(\wh A, A) &:= \min_{P\in \H_K}\|\wh A - AP\|_\c, \quad	& L_{1}(\wh A, A) &:= \min_{P\in \H_K}\| \wh A - AP\|_1.
\end{align*}
We provide upper bounds for $L_{1}(\wh A, A)$ and  	$L_\c(\wh A, A)$ in Theorem \ref{thm_rate_Ahat} of Section \ref{sec_upperbound}. 
To benchmark these upper bounds,   Theorem \ref{thm_lb} in Section \ref{sec_lowerbound} shows that the corresponding lower bounds are: 
\begin{align}\label{ours}
&\inf_{\wh A} \sup_{A}\PP_{A}\left\{ L_\c(\wh A, A) \ge c_0\sqrt{K(|I_{\max}| + |J|)\over nN} \right\} \ge c_1,\\\nonumber
&\inf_{\wh A} \sup_{A}\PP_{A}\left\{ L_{1}(\wh A, A) \ge c_0K\sqrt{|I| + K |J|\over nN} \right\} \ge c_1,
\end{align}
for absolute constants $c_0>0$ and $c_1\in (0, 1]$ and assuming $N:=N_1 = \cdots N_n$ for ease of presentation. The infimum is taken over all estimators $\wh A$, while the supremum is taken over all matrices $A$ in a prescribed class $\A$, defined in (\ref{classA}). The lower bounds depend on the largest number of anchor words within each topic  ($|I_{\max}|$), the total number of anchor words ($|I|$), and the number of non-anchor words ($|J|$) with $J := [p]\setminus I$.  In Section \ref{sec_upperbound} we discuss conditions under which  our estimator $\widehat A$  is minimax optimal, up to a logarithmic factor, under both losses. To the best of our knowledge, these lower and upper bounds on the $L_{1, \infty}$ loss of our estimators are new, and valid for growing $K$ and $p$. They imply the more commonly studied bounds on the  $L_{1}$ loss. 

Our estimation procedure and the analysis of the resulting estimator $\widehat{A}$ are tailored to count data, and utilize  the restrictions (\ref{col_sum_one}) on the parameters of  model (\ref{model}). Consequently, both the estimation method and the properties of the estimator differ from those developed for general identifiable latent variable models, for instance those in    \cite{LOVE}, and we refer to the latter for further references  and a recent  overview of estimation in such models.

To the best of our knowledge, computationally efficient estimators of the word-topic matrix $A$ in (\ref{model}), that are also accompanied  by a  theoretical analysis, have only been developed for the situation in which $K$ is known in advance. Even in that case, the existing results are limited. 

\cite{arora2012learning,arora2013practical} are the first to analyze theoretically, from a  rate perspective,	estimators of $A$ in the  topic model. They establish  upper bounds    on the global $L_{1}$ loss of their estimators, and their analysis  
allows $K$ and $p$ to grow with $n$. Unfortunately, these bounds   differ by at least a factor of order $p^{3/2}$ from the minimax optimal rate given by our Theorem \ref{thm_rate_Ahat}, even when $K$ is fixed and does not grow with $n$.

The recent work of \cite{Tracy} is tailored  
to topic models with a small, known, number of topics $K$, which is independent of the number of documents $n$. 
Their procedure makes clever use of the geometric simplex structure in the singular vectors of $\M$. To the best of our knowledge, \cite{Tracy} is the first work that proves a  minimax lower bound for  the estimation of $A$ in topic models, with respect to the $L_{1}$ loss,  over a different parameter space than the one we consider. We discuss in detail the corresponding rate over this space, and compare it with ours, in Remark \ref{minimax_Ke} in Section \ref{sec_lowerbound}.
The procedure developed by \cite{Tracy} is rate optimal for fixed $K$, under suitable conditions tailored to their set-up (see pages 13 -- 14 in \cite{Tracy}).

We defer a  detailed rate comparison with existing results to Remark \ref{minimax_Ke} of Section \ref{sec_lowerbound} and to Section \ref {sec_compare_rate}.

In Section \ref{sec_sim} we present a simulation study, in which 
we compare numerically the quality 
of our  estimator with   that of the best performing estimator to date, developed in \cite{arora2013practical}, which also comes with theoretical guarantees, albeit not minimax optimal.
We found that the competing estimator   is generally fast and accurate when $K$ is known, but it is very sensitive to the misspecification of $K$, as we illustrate in Appendix \ref{sec_sim_K} of the supplement. Further, extensive comparisons are presented in Section  \ref{sec_sim}, in terms of the estimation of $I$, $A$ and the computational running time of the algorithms. We found that our procedure dominates on all these counts.

Finally, the proofs of Propositions \ref{prop_orth_AW} and \ref{prop_anchor} of Section \ref{sec_idenA} and the results of Sections \ref{sec_est_I} and \ref{sec_A} are deferred to the appendices. 

\paragraph*{Summary of new contributions.}
We propose a new method that estimates
\begin{enumerate}[noitemsep]
	\item[(a)]  the number of topics $K$;
	\item[(b)]  the anchor words and their partition;
	\item[(c)]  the word-topic matrix $A$; 
\end{enumerate}
and  provide an analysis under a {\em finite sample setting}, that allows $K$, in addition to $N_i$ and $p$ to grow with the sample size 
(number of documents) $n$. In this regime, 
\begin{enumerate}[noitemsep]
	\item[(d)] we establish a minimax lower bound for estimating the word-topic matrix $A$; 
	\item [(e)]  we show that the number of topics can be estimated correctly, with high probability; 
	\item [(f)] we show that $A$ can be estimated at the minimax-optimal rate. 
\end{enumerate}
Furthermore,
\begin{enumerate}[noitemsep]
	\item[(g)] the estimation of $K$  is optimization free;
	\item[(h)] the estimation of the anchor words and that of $A$ is scalable in $n,N_i$, $p$ and $K$. 
\end{enumerate}
To the best of our knowledge, 
estimators of $A$ that are scalable not only with $p$, but also with $K$, and for which (a), (b) and (d) - (f) hold are new in the literature.

\subsection{Notation}
The following notation will be used throughout the 
entire paper. 

The integer set $\{1,\ldots,n\}$ is denoted by $[n]$. For a generic set $S$, we denote $|S|$ as its cardinality. For a generic vector $v\in \RR^d$, we let $\|v\|_q$ denote the vector $\ell_q$ norm, for $q= 0,1,2,\ldots,\infty$ and $\textrm{supp}(v)$ denote its support. We denote by $\diag(v)$ a $d\times d$ diagonal matrix with diagonal elements equal to $v$. For a generic matrix $Q\in \RR^{d\times m}$, we write $\|Q\|_{\infty}=\max_{1\le i \le d, 1\le j \le m}|Q_{ij}|$, $\|Q\|_{{1}} = \sum_{1\le i\le d,1\le j\le m}|Q_{ij}|$ and $\|Q\|_{\r}=\max_{1\le i\le d}\sum_{1\le j\le m}|Q_{ij}|$.  For the submatrix of $A$, we let $Q_{i\cdot}$ and $Q_{\cdot j}$ be the $i$th row and $j$th column of $Q$. For a set $S$, we let $Q_S$ denote its $|S|\times m$ submatrix. We write the $d\times d$ diagonal matrix $$D_Q = \diag(\|Q_{1\cdot}\|_1, \ldots, \|Q_{d\cdot}\|_1)$$ and let $(D_Q)_{ii}$ denote the $i$th diagonal element. 

We use $a_n \lesssim b_n$ to denote there exists an absolute constant $c>0$ such that $a_n \le cb_n$, and write $a_n \asymp b_n$ if there exists two absolute constants $c, c'>0$ such that 
$c b_n \le a_n \le c'b_n$.

We let $n$ stand for the number of documents and $N_i$ for the number of randomly drawn words at  document $i\in[n]$.  Furthermore,  $p$ is the total number of words (dictionary size) and $K$ is the number of topics. We define $M := \max_i N_i \vee n\vee p $. Finally, $I$ is the (index) set of anchor words, and its complement $J:= [p]\setminus I$ forms the (index) set of non-anchor words.

\section{Preliminaries}\label{sec_iden}
In this section we introduce and discuss the assumptions under which $A$ in  model (\ref{model}) can be uniquely determined via $\Pi$, although $W$ is not observed.  

\subsection{An information bound perspective on model assumptions}\label{sec_fisher}

If  we had access to $W$ in model 
(\ref{model}), then the problem  of estimating $A$ would become the more standard problem of estimation in multivariate response regression under the constraints (\ref{col_sum_one}), and dependent  errors. In that case,  $A$ is uniquely defined if $W$ has full rank, which  is our  Assumption \ref{ass_pd} below. Since $W$ is not observable, we mentioned earlier that the identifiability of $A$ requires extra assumptions. We provide insight into their nature, via a classical information bound calculation. We view $W$ as a nuisance parameter and ask when the estimation of $A$ can be done with the  same precision  whether $W$ is known or not. In classical information bound jargon \citep{CoxReid}, we study when the parameters $A$ and $W$ are orthogonal. The latter  is equivalent with verifying 	\begin{eqnarray}\label{fisher}
\EE\left[- {\partial^2 \ell(X_1,\ldots, X_n) \over \partial A_{jk}\partial W_{k'i}} \right] = 0\quad \text{ for all $j\in [p]$, $i\in [n]$ and $k, k'\in [K]$,}
\end{eqnarray}
where $\ell(X_1, \ldots, X_n)$ is the log-likelihood of $n$ independent multinomial vectors.  Proposition \ref{prop_orth_AW} below 
gives necessary and sufficient conditions for parameter orthogonality. 

\begin{prop}\label{prop_orth_AW}
	If $X_1, \ldots X_n$ are an independent sample from (\ref{model_multinomial}),  and (\ref{model}) holds, then  $A$ and $W$ are orthogonal parameters, in the sense (\ref{fisher}) above, if and only if the following holds:
	\begin{equation}\label{cond_orth}
	\Big| \textrm{supp}(A_{j\cdot}) \cap \textrm{supp}(W_{\cdot i}) \Big| \le 1, \quad \text{ for all }j\in[p], i\in [n].
	\end{equation}
\end{prop}
\medskip

\noindent We observe that condition (\ref{cond_orth})  is implied by either of the two following extreme  conditions:  
\begin{enumerate}
	\item[(1)]  
	All rows in $A$ are proportional to canonical vectors in $\RR^K$, which is equivalent to assuming  that all words are anchor words. 
	\item[(2)] $C:=n^{-1}WW^T$ is diagonal.
\end{enumerate} 
In the first scenario, each topic is described via words exclusively used for that topic, which is unrealistic. 
In the second case, the topics are totally unrelated to one another, an assumption that is not generally met, but is  perhaps more plausible than (1). Proposition \ref{prop_orth_AW} above shows that 
one cannot expect the estimation of $A$ in (\ref{model}), when $W$ is not observed, to be as easy as that when $W$ is observed, unless the very stringent conditions of this proposition hold. However, it points 
towards quantities that play a crucial role in the estimation of $A$: the anchor words 
and the rank of $W$. 
This motivates the study of this model, with both $A$ and $W$ unknown,  under the  more realistic assumptions introduced in the next section and used throughout this paper. 

\subsection{Main assumptions}
We make   the following three main  assumptions:  
\begin{ass}\label{ass_sep}
	For each topic $k = 1,\ldots, K$, there exists at least one word $j$ such that $A_{jk} >0$ and $A_{j\ell} = 0$ for any $\ell \ne k$.
\end{ass}
\begin{ass}\label{ass_pd}
	The matrix $W$ has   rank $K\le n$.
\end{ass}
\begin{ass}\label{ass_w}
	The inequality 
	\[ \cos\left( \angle( W_{i\cdot},   W_{j\cdot} ) \right)< \frac{\zeta_i}{\zeta_j} \wedge  \frac{\zeta_j}{\zeta_i}\qquad \text{for all $1\le i \ne j\le K$},\]
	holds, with $\zeta_i:= \| W_{i\cdot} \|_2 / \|W_{i\cdot} \|_1$.
\end{ass}      
Conditions on $A$ and $W$ under which $A$ can be uniquely determined from $\Pi$ are generically known as separability conditions, and were first introduced by \cite{donohoNMF}, for the identifiability of the factors in general nonnegative matrix factorization (NMF) problems. Versions of such conditions have been subsequently adopted in most of the literature  on topic models, which are particular instances of NMF, see, for instance, \cite{rechetNMF,arora2012learning,arora2013practical}.

In the context and interpretation of the topic model, the commonly accepted  Assumption \ref{ass_sep}    postulates that for each topic $k$ there exists at least one word solely associated with that topic. Such words are called {\it anchor words}, as the appearance of an anchor word is a clear indicator of the occurrence of its corresponding topic, and typically more than one anchor word is present. For future reference, for a given word-topic matrix $A$, we let  $I := I(A)$ be the set of anchor words, and $\mathcal{I}$ be its partition relative to topics: 
\begin{eqnarray}\label{def_pure}
I_k := \{j\in[p]: A_{jk}  >  0,  \ A_{j\ell} = 0\ \text{ for all } \ell \ne k\},\quad 
I := \bigcup_{k=1}^{K} I_k,\quad  \I :=  \left\{I_1, \ldots, I_K \right\}.
\end{eqnarray}
Earlier work  \cite{Anandkumar} proposes a tensor-based approach that does not require the anchor word assumption, but assumes that the topics are uncorrelated. \cite{LM-dag,blei2007} showed that, in practice, there is strong evidence against
the lack of correlation between topics.  We therefore relax the orthogonality conditions on the	matrix $W$ in our Assumption \ref{ass_pd}, similar to \cite{arora2012learning, arora2013practical}. 
We note that in  Assumption \ref{ass_pd} we have $K \leq n$, which translates as:  the total number $K$ of topics covered by $n$ documents is smaller than the number of documents.  

Assumption \ref{ass_pd} guarantees that 
the rows of $ W$, viewed as vectors in $\RR^n$, are not parallel, and Assumption \ref{ass_w}
strengthens this, by placing  a mild condition on the angle 
between any two rows of $W$.  
If, for instance,  $WW^T$ is a diagonal matrix, or if  $\zeta_i$ is the same for all $i\in[K]$, 
then Assumption \ref{ass_pd} implies 
Assumption \ref{ass_w}. However, the two assumptions are not equivalent, and neither implies the other, in general. We illustrate this in the   examples of  Section \ref{app_ex} in  the supplement. It is worth mentioning that, when the columns of $W$ are i.i.d. samples from the Dirichlet distribution as commonly assumed in the topic model literature \citep{BleiLDA,blei2007,blei-intro},  Assumption \ref{ass_w} holds with high probability under mild conditions on the hyper-parameter of the Dirichlet distribution. We defer their precise expressions  to Lemma \ref{lem_dir} in Appendix \ref{app_dir} of the supplement.

We discuss these assumptions further  in Remark \ref{simplex} of Section \ref{sec_idenA} and Remark \ref{remark_nu} of Section \ref{sec_est_I} below.

\section{Exact recovery of $I$, $\mathcal{I}$  and $A$ at the population level}\label{sec_idenA}
In this section we construct  $A$ via $\Pi$. Under Assumptions \ref{ass_sep} and \ref{ass_w}, we show first that the set of anchor words $I$ and its partition $\I$ can be determined, from the matrix $R$ given in (\ref{def_R}) below. 
We begin by re-normalizing  the three matrices involved in model (\ref{model}) such that their rows sum up to 1: 
\begin{equation}\label{normal}  
\wt W :=  D_W^{-1}W, \quad  \wt \Pi:= D_\Pi^{-1}\Pi, \quad \wt A :=  D_\Pi^{-1}AD_W. 
\end{equation}
Then
\begin{equation}\label{scaled}  
\wt \Pi= \wt A \wt W, 
\end{equation}
and 
\begin{equation}\label{def_R}
R := n\wt \Pi \wt \Pi^T = \wt A \wt C \wt A^T, 
\end{equation}
with $\wt C = n \wt W \wt W^T$.	
This normalization is standard in the topic model literature \citep{arora2012learning,arora2013practical}, and it preserves the anchor word structure:  matrices $A$ and $\wt A$ have the same support, and  Assumption \ref{ass_sep} is equivalent with the existence, for each $k\in[K]$, of at least one word $j$ such that $\wt A_{jk} =1$ and $\wt A_{j\ell} = 0$ for any $\ell \ne k$. 
Therefore $A$ and $\wt A$ have the same $I$ and $\mathcal{I}$. We differ from the existing literature in the way we make use of this normalization and explain this in Remark \ref{simplex} below. 	
Let  
\begin{equation}\label{comp}  
T_i := \max_{1\le j\le p}R_{ij}, \qquad 	S_i := \left\{j\in [p]: R_{ij} = T_i\right\}, \qquad\text{for any $i\in [p].$}
\end{equation}  
In words, $T_i$ is the maximum entry of row $i$, and $S_i$ is the set of column indices of those entries in row $i$ that equal to the row maximum value. The following proposition shows the exact recovery of $I$ and $\I$ from $R$.


\begin{prop}
	\label{prop_anchor}
	Assume that model (\ref{model}) and Assumptions \ref{ass_sep} and \ref{ass_w} hold. Then:
	\begin{enumerate}
		\item[(a)] 
		$i\in I \iff T_i = T_j, \quad \text{for all }j\in S_i.$ 
		\item[(b)] The anchor word set $I$ can be determined uniquely from $R$. Moreover, its partition $\I$ is unique and can be determined from $R$ up to label permutations.
	\end{enumerate}
\end{prop}
\noindent  The proof of this proposition  is given in Appendix \ref{sec_iden}, and its success relies on the equivalent formulation of Assumption \ref{ass_w},
\[ \min_{1\le i<j\le K} \left( \wt C_{ii}\wedge \wt C_{jj} - \wt C_{ij} \right) > 0. \]

\medskip	

The short proof of Proposition  \ref{prop_A} below gives an explicit construction of  $A$ from 
\begin{equation}\label{teta} \Theta := {1\over n}\Pi\Pi^T, \end{equation} 
using the unique partition $\I$ of $I$ given by Proposition \ref{prop_anchor} above. 

\begin{prop}\label{prop_A}
	Under model (\ref{model}) and Assumptions \ref{ass_sep}, \ref{ass_pd} and \ref{ass_w}, $A$ can be uniquely recovered from $\Theta$ with given $\I$, up to column permutations.
\end{prop}
\begin{proof}
	Given the partition of anchor words $\I = \{I_1, \ldots, I_K\}$, we construct a set $L = \{i_1, \ldots, i_K\}$ by selecting one anchor word $i_k\in I_k$ for each topic $k\in [K]$. We let $A_L$ be the diagonal matrix 
	\begin{equation}\label{def_AL}
	A_L= \diag(A_{i_11}, \ldots, A_{i_KK}). 
	\end{equation}
	We  show first that      
	$B:= A A_L^{-1}$  can be constructed from $\Theta$. Assuming, for now, that $B$ has been constructed, then $A =  BA_L$. The diagonal elements of $A_L$ can be readily determined from this relationship, since, via model (\ref{model}) satisfying (\ref{col_sum_one}), the columns of $A$ sum up to 1:  \begin{equation}\label{AL}1 = \|A_{\cdot k}\|_1 = A_{i_{k}k}\|B_{\cdot  k}\|_1,\end{equation}   for each $k$. Therefore, although $B$ is only unique up to the choice of $L$ and of the scaling matrix $A_L$, the matrix $A$ with unit column sums thus constructed is unique. 
	
	It remains to construct $B$ from $\Theta$. Let $J = \{1, \ldots, p\} \setminus I$. We let $B_J$ denote the $|J| \times K$ sub-matrix of $B$ with row indices in $J$ and $B_{I}$ denote the $|I| \times K$ sub-matrix of $B$ with row indices in $I$. Recall that $C:= n^{-1}WW^T$. Model (\ref{model}) implies the following decomposition of the submatrix of $\Theta$ with row and column indices in $L \cup J$: 
	\[
	\begin{bmatrix}
	\Theta_{LL} & \Theta_{LJ}\\
	\Theta_{JL} & \Theta_{JJ}
	\end{bmatrix} =\begin{bmatrix}
	A_LCA_L & A_LCA_{J}^T\\
	A_{J}CA_L & A_{J}CA_{J}^T
	\end{bmatrix}.
	\] 
	In particular, we have
	\begin{equation}\label{display3}
	\Theta_{LJ} = A_L C A_J^T= A_L C (A_LA_L^{-1})A_J^T = \Theta_{LL} (A_L^{-1}A_J^T) = \Theta_{LL} B^T_J.
	\end{equation}
	Note that $A_{i_kk} >0 $, for each $k\in [K]$,  from Assumption \ref{ass_sep} which, together with Assumption \ref{ass_pd},  implies that $\Theta_{LL}$ is invertible. We then have
	\begin{equation}\label{bj}  
	B_J = \Theta_{JL}\Theta_{LL}^{-1}. \end{equation}
	
	On the other hand, for any $i\in I_k$,  for each $k \in [K]$,
	we have $B_{{ik}} = 
	{ A_{ik} }/ { A_{i_{k}k}}$, 
	by the definition of $B$.
	Also, model (\ref{model}) and Assumption \ref{ass_sep} imply that for any $i \in I_k$, 
	\begin{equation}\label{display2}
	{1\over n}\sum_{t=1}^n \M_{it}= A_{ik}\left({1\over n}\sum_{t=1}^n W_{kt}\right).
	\end{equation}
	Therefore, the matrix $B_I$ has entries
	\begin{equation}\label{eq_iden_BI}
	B_{ik}  = {\|\M_{i\cdot}\|_1\over\|\M_{i_k\cdot}\|_1  } ,\quad \text{for any $i\in I_k$ and $k\in [K]$.}
	\end{equation}
	This, together with $B_J$ given above completes the construction of $B$, and  uniquely  determines $A$.
\end{proof}

Our approach for recovering both $I$ and $A$ is constructive and can be easily adapted to estimation. For this reason, we summarize our approach in Algorithm \ref{alg} and illustrate the algorithm with a simple   example. 

\begin{algorithm}[ht]
	\caption{Recover the word-topic matrix $A$ from $\M$}\label{alg}
	\begin{algorithmic}[1]
		\Require true word-document frequency matrix $\M\in\RR^{p\times n}$
		\Procedure{Top}{$\M$}
		\State compute $\Theta = n^{-1}\M\M^T$ and $R$ from (\ref{def_R})
		\State recover $\I$ via \textsc{FindAnchorWords}$(R)$
		\State construct $L = \{i_1, \ldots, i_K\}$ by choosing any $i_k \in I_k$, for $k\in [K]$
		\State compute $B_I$ from (\ref{bj}) and $B_J$ from (\ref{eq_iden_BI})
		\State recover $A$ by normalizing $B$ to unit column sums
		\State\Return $\I$ and $A$
		\EndProcedure
		\Statex
		\Procedure{FindAnchorWords}{$R$}
		\State initialize $\I = \emptyset$ and $
		\mathcal{P} = [p]$
		\While{$\mathcal{P}\ne \emptyset$} 
		\State take any $i\in \mathcal{P} $, compute $S_i$ and $T_i$ from (\ref{comp})
		\If {$\exists j\in S_i$ s.t. $T_i \ne T_j$}
		\State $\mathcal{P} = \mathcal{P} \setminus \{i\}$
		\Else 
		\State {$\mathcal{P} = \mathcal{P} \setminus S_i$ and $S_i \in \I$}
		\EndIf
		\EndWhile
		\State\Return $\I$
		\EndProcedure
	\end{algorithmic}
\end{algorithm}

\begin{example}\label{eg_2}
	Let $K = 3, p = 6, n = 3$ and consider the following $A$ and $W$:
	\begin{equation*}
	A = \begin{bmatrix}
	0.3 & 0 & 0 \\
	0.2 & 0 & 0\\
	0 & 0.5 & 0\\
	0 & 0 & 0.4\\
	0.2 & 0.5 & 0.3\\
	0.3 & 0 & 0.3
	\end{bmatrix},
	\quad W = \begin{bmatrix}
	0.6 & 0.2 & 0.2 \\ 
	0.3 & 0.7 & 0.0 \\ 
	0.1 & 0.1 & 0.8 \\ 
	\end{bmatrix}, \quad 
	\M = AW = \begin{bmatrix}
	0.18 & 0.06 & 0.06 \\ 
	0.12 & 0.04 & 0.04 \\ 
	0.15 & 0.35 & 0.00 \\ 
	0.04 & 0.04 & 0.32 \\ 
	0.30 & 0.42 & 0.28 \\ 
	0.21 & 0.09 & 0.30 \\ 
	\end{bmatrix}
	\end{equation*}
	Applying \textsc{FindAnchorWords} in  Algorithm \ref{alg} to $R$ gives $\I = \bigl\{\{{\color{red} 1,2}\},\{{\color{blue}3}\},\{{\color{green}4}\}\bigr\}$
	from 
	\[
	R = \begin{bmatrix}
	{\color{red}\bm{1.32}} & 	{\color{red}\bm{1.32}} & 0.96 & 0.72 & 0.96 & 1.02 \\ 
	{\color{red}\bm{1.32}} & 	{\color{red}\bm{1.32}} & 0.96 & 0.72 & 0.96 & 1.02 \\ 
	0.96 & 0.96 & 	{\color{blue}\bm{1.74}} & 0.30 & 1.15 & 0.63 \\ 
	0.72 & 0.72 & 0.30 & 	{\color{green}\bm{1.98}} & 0.89 & 1.35 \\ 
	0.96 & 0.96 & 1.15 & 0.89 & 1.03 & 0.92 \\ 
	1.02 & 1.02 & 0.63 & 1.35 & 0.92 & 1.19 \\ 
	\end{bmatrix}
	\Longrightarrow
	\begin{array}{lll}
	T_1 =1.32, & S_1 = \{{\color{red}1, 2}\} , & {\color{red}1}-\checkmark \\ T_2 = 1.32,&  S_2 = \{{\color{red}1, 2}\}, & {\color{red} 2}-\checkmark\\ T_3 = 1.74,& S_3 = \{{\color{blue}3}\}, & {\color{blue} 3}-\checkmark \\
	T_4 = 1.98,& S_4 = \{{\color{green}4}\}, & {\color{green} 4}-\checkmark\\
	T_5 = 1.15, &  S_5 = \{{\color{blue}3}\}, & 5-\text{\sffamily x}\\ T_6 = 1.35,& S_6 = \{ {\color{green}4}\}, & 6-\text{\sffamily x}
	\end{array}
	\]
	Based on the recovered $\I$,   the matrix $A$ can be recovered from Proposition \ref{prop_A}, which is executed via steps 4 - 6 in Algorithm \ref{alg} . Specifically, by taking $L = \{{\red 1}, {\blue 3}, {\green 4}\}$ as the representative set of anchor words, it follows from (\ref{bj}) and (\ref{eq_iden_BI}) that  
	\[
	B_{I} = \begin{bmatrix}
	{\red 1} & 0 & 0 \\
	{\red 2/3} & 0 & 0\\
	0 & {\blue 1} & 0\\
	0 & 0 & {\green 1}\\
	\end{bmatrix}
	,\quad 
	B_J = \begin{bmatrix}
	0.03 & 0.06 & 0.04 \\ 
	0.02 & 0.02 & 0.04 \\ 
	\end{bmatrix}\begin{bmatrix}
	0.01 & 0.02 & 0.01 \\ 
	0.02 & 0.05 & 0.01 \\ 
	0.01 & 0.01 & 0.04 \\ 
	\end{bmatrix}^{-1}\!\!\! = \begin{bmatrix}
	2/3 & 1 & 3/4\\
	1 & 0 & 3/4
	\end{bmatrix}.
	\]
	Finally, $A$ is recovered by normalizing $B = [B_I^T, B_J^T]^T$ to have unit column sums,
	\[
	A = \begin{bmatrix}
	{\red 1} & 0 & 0 \\
	{\red 2/3} & 0 & 0\\
	0 & {\blue 1} & 0\\
	0 & 0 & {\green 1}\\
	2/3 & 1 & 3/4\\
	1 & 0 & 3/4
	\end{bmatrix}\begin{bmatrix}
	0.3 & 0 & 0\\
	0 & 0.5  & 0 \\
	0 & 0 & 0.4
	\end{bmatrix} = \begin{bmatrix}
	0.3 & 0 & 0 \\
	0.2 & 0 & 0\\
	0 & 0.5 & 0\\
	0 & 0 & 0.4\\
	0.2 & 0.5 & 0.3\\
	0.3 & 0 & 0.3
	\end{bmatrix}.
	\]
	
\end{example}

\begin{remark}\label{simplex}  {\it Contrast with existing results.}
	It is easy to see  that the rows in $R$ (or, alternatively,  $\wt{\Pi}$)  corresponding to  non-anchor words $j\in J$ are convex combinations of the rows in $R$ (or  $\wt{\Pi}$) corresponding to anchor words $i\in I$. Therefore, finding $K$ representative anchor words,  amounts to finding the $K$ vertices of  a simplex. The latter can be accomplished by finding the unique solution of an appropriate linear program, that uses $K$ as input, as shown by \cite{rechetNMF}. This result only utilizes Assumption \ref{ass_sep} and a relaxation of Assumption \ref{ass_pd}, in which it is assumed that no rows of $\wt W$ are convex combinations of the rest. To the best of our knowledge, Theorem 3.1 in \cite{rechetNMF} is the only result to guarantee that, after representative anchor words  are found, a partition of $I$ in $K$ groups can also be found, for the specified $K$. 
	
	When $K$ is not known, this strategy can no longer be employed, since finding the representative anchor words requires  knowledge of $K$. However, we showed that  this problem can still be solved under our mild additional Assumption \ref{ass_w}. This assumption allows us to provide the if and only if characterization of $I$ proved in part (i) of Proposition \ref{prop_anchor}. Moreover, part (ii) of this proposition shows that $K$ is in one-to-one correspondence with the number of groups in $\mathcal{I}$, and we exploit this observation for the estimation of $K$. 
\end{remark}

\section{Estimation of the anchor word set and of the number of topics}\label{sec_est_I}
Algorithm \ref{alg} above recovers  the index set $I$, its partition $\mathcal{I}$ and the number of topics $K$ from the matrix 
$$
R =n\wt \M\wt \M^T=\left(n D_\M^{-1}\right)\Theta\left(nD_\M^{-1}\right)
$$ with $\Theta = n^{-1}\M\M^T$. Algorithm \ref{alg1} below is a sample version of Algorithm \ref{alg}. It has $O(p^2)$ computational complexity  and is optimization free.

\begin{algorithm}[ht]
	\caption{Estimate the partition of the anchor words $\I$ by $\wh \I$}\label{alg1}
	\begin{algorithmic}[1]
		\Require matrix $\wh R\in\RR^{p\times p}$, $C_1$ and $Q\in\RR^{p\times p}$ such that $Q[j,\ell] := C_1\delta_{j\ell}$ 
		\Procedure{FindAnchorWords}{$\wh R$, $Q$}
		\State initialize $\wh \I = \emptyset$
		\For{$i\in [p]$} 
		\State $ \wh a_i = \argmax_{1\le j\le p}\wh R_{ij}$
		\State set $\wh I^{(i)} =  \{\ell\in [p]: \wh R_{i\wh a_i}-\wh R_{il} \le Q[i,\wh a_i]+ Q[i,\ell]\}$ and $\textsc{Anchor}(i) = \textsc{True}$
		\For {$j \in \wh I^{(i)}$}
		\State $\wh a_j = \argmax_{1\le k\le p}\wh R_{jk}$
		\If {$\Bigl|\wh R_{ij}-\wh R_{j\wh a_j}\Bigr| > Q[i,j] + Q[j, \wh a_j]$}   
		\State $\textsc{Anchor}(i) =\textsc{False}$
		\State \textbf{break}
		\EndIf	
		\EndFor
		\If {$\textsc{Anchor}(i) $}
		\State $\wh \I = \textsc{Merge}(\wh I^{(i)}$, $\wh \I$)
		\EndIf
		\EndFor
		\State\Return $\wh \I = \{ \wh I_1, \wh I_2, \ldots, \wh I_{\wh K}\}$ 
		\EndProcedure
		\Statex
		
		\Procedure{Merge}{$\wh I^{(i)}$, $\wh\I$}
		\For {$G \in \wh \I$}
		\If {$G \cap \wh I^{(i)}\ne \emptyset$} 
		\State replace $G$ in $\wh \I$ by $G\cap \wh I^{(i)}$
		\State\Return $\wh \I$
		\EndIf
		\EndFor
		\State {$\wh I^{(i)} \in \wh \I$}
		\State\Return $\wh \I$
		\EndProcedure
	\end{algorithmic}
\end{algorithm}

The matrix   $\Pi$ is replaced by the observed frequency data matrix $X\in \RR^{p\times n}$ with independent  columns $X_1, \ldots X_n$. Since they that are assumed to follow the multinomial model  (\ref{model_multinomial}), an unbiased estimator of $\Theta $ is given by 
\begin{equation}\label{est_Theta}
\wh\Theta := {1\over n}\sum_{i =1}^n\left[
{N_i \over N_i - 1}X_iX_i^T - {1\over N_i-1} \textrm{diag}(X_i)\right],
\end{equation}
with $N_i$ representing the total number of words in document $i$. We then estimate $R$ by 
\begin{equation}\label{def_R_hat}
\wh R := \left(nD_X^{-1}\right) 	\wh\Theta  \left(nD_X^{-1}\right).
\end{equation}

\noindent		The quality of our estimator depends on how well we can control  the 
noise level $\wh R - R$. In the computer science related literature, albeit for different algorithms,  
\citep{arora2012learning,rechetNMF}, only global $\| \wh R - R\|_\r$
control is considered, which ultimately impacts negatively the rate of convergence of $A$. In general latent models with pure variables, the latter being the analogues of 
anchor words, \cite{LOVE} developed  a similar algorithm to ours, under  a less stringent $\| \wh R - R\|_{\infty}$ control, which is still not precise enough for sharp estimation in  topic models. To see why, we note that  Algorithm \ref{alg1}  involves comparisons between two different entries in a row of $\wh R$.  In these comparisons, we must allow for small {\it entry-wise} error margins. These margin levels are precise bounds  $C_1\delta_{j\ell}$ such that $|\wh R_{j\ell}-R_{j\ell}| \le C_1\delta_{j\ell}$ for all $j,\ell \in [p]$, with high probability, for  some universal constant $C_1>1$.
The explicit deterministic  bounds are stated in Proposition \ref{prop_sigma} of Appendix \ref{sec_stat_sigma}, while  
practical data-driven choices are based on Corollary \ref{cor_theta_R} of Appendix \ref{sec_stat_sigma} and given in Section \ref{sec_sim}.

Since the estimation of $\I$ is based on $\wh R$ which is a perturbation of $R$,  one cannot distinguish an anchor word from a non-anchor word that is very close to it,  without further signal strength conditions on $\wt A$. 
Nevertheless,  Theorem \ref{thm_anchor} shows that even without such conditions we can still estimate $K$ consistently. Moreover, 
we guarantee the recovery of $I$ and $\mathcal{I}$ with minimal mistakes. Specifically,  we denote the set of {\em quasi}-anchor words by  
\begin{eqnarray}\label{J1}
J_1 &:=&\left \{ j\in J: \text{there exists }k\in [K] \text{ such that }\wt A_{jk} \ge 1-4\delta/\nu   \right\}
\end{eqnarray}
where 
\begin{equation}\label{def_nu2}
\nu:=\min_{1\le i<j\le K} \left( \wt C_{ii}\wedge \wt C_{jj} - \wt C_{ij} \right)
\end{equation}  
and 
\begin{equation}\label{def_delta}
\delta := \max_{1\le j,\ell \le p}\delta_{j\ell}.
\end{equation}
In the proof of Proposition \ref{prop_anchor}, we argued that the set of anchor words, defined in Assumption \ref{ass_sep},  coincide with those of the scaled matrix $\widetilde{A}$ given in (\ref{normal}). The words corresponding to indices in $J_1$ are almost anchor words,  since in a row of $\wt A$ corresponding to such index the  largest entry is close to $1$, while the other entries are close to $0$,  if $\delta/\nu $ is small. 

For the remaining of the paper we make the blanket assumption that  all  documents have equal length, that is, $N_1 = \cdots = N_n = N$.  We make this assumption for ease of presentation only, as all our results continue to hold when the documents have unequal lengths. 

\begin{thm}\label{thm_anchor}
	Under model (\ref{model}) and Assumption \ref{ass_sep}, assume 
	\begin{equation}\label{ass_nu}
	\nu > 2 \max \left\{2\delta,~  \sqrt{2\|\wt C\|_{\infty}\delta } \right\}
	\end{equation}
	with $\nu$ defined in (\ref{def_nu2}), and 
	\begin{equation}\label{ass_signal}
	\min_{1\le j\le p}{1\over n}\sum_{i =1}^n\M_{ji} \ge {2\log M \over 3N},\qquad \min_{1\le j\le p}\max_{1\le i\le n}\M_{ji} \ge {(3\log M)^2 \over N}.
	\end{equation}   
	Then, with probability greater than $1-8M^{-1}$, we have 
	\[
	\wh K = K,\quad I \subseteq \wh I \subseteq I\cup J_1, \quad I_{\pi(k)} \subseteq \wh I_{k} \subseteq I_{\pi(k)}\cup J_1^{\pi(k)}, \quad \text{ for all }1\le k\le K,
	\]
	where $J_1^k := \{j\in J_1: \wt A_{jk} \ge 1-4\delta/\nu \}$ and $\pi: [K]\rightarrow [K]$ is some label permutation. 
\end{thm}

If we further impose the signal strength assumption $J_1 = \emptyset$, the following corollary guarantees  exact recovery of all anchor words. 
\begin{cor}\label{cor_anchor}
	Under model (\ref{model}) and Assumption \ref{ass_sep}, assume $\nu > 4\delta$, (\ref{ass_signal}) and $J_1 = \emptyset$. With probability greater than $1-8M^{-1}$,  we have
	$\wh K = K$, $\wh I = I$ and $\wh I_{k} = I_{\pi(k)}$, for all $1\le k \le K$ and some permutation $\pi$.
\end{cor}

\begin{remark}\mbox{}\label{remark1}
	\begin{enumerate}
		\item[(1)]  
		Condition (\ref{ass_signal}) is assumed for the analysis only and 
		the implementation of our procedure only requires $N \ge 2$.  
		Furthermore, we emphasize that (\ref{ass_signal}) is assumed to simplify  our presentation.  In particular, we used it  to obtain the precise expressions of $\wh\delta_{j\ell}$ and $\wh\eta_{j\ell} $ given in (\ref{delta3}) -- (\ref{def_eta3}) of Section \ref{sec_sim}. In fact, (\ref{ass_signal}) can be relaxed to 
		\begin{equation}\label{ass_signal_weak}
		\min_{1\le j\le p} {1\over n}\sum_{i = 1}^n \M_{ji} \ge {c\log M \over nN}
		\end{equation}
		for some sufficiently large constant $c>0$, under which  more complicated expressions of $\delta_{j\ell}'$ and $\eta_{j\ell}'$ can be derived, see Corollary \ref{cor_delta} of Appendix \ref{sec_stat_sigma}.  Theorem \ref{thm_anchor} continues to hold, provided that (\ref{ass_nu}) holds for $\delta' = \max_{j,\ell}\delta_{j\ell}'$ in lieu of $\delta$, that is,
		\begin{eqnarray}
		\label{ass_nu_prime}\nu> 2\max \left\{ 2\delta', \sqrt{2 \| \wt C\|_\infty \delta'} \right\}.
		\end{eqnarray}
		Note that  condition (\ref{ass_signal_weak}) implies the restriction
		\begin{equation}\label{constr_0}
		nN \ge c\cdot p\log M, 
		\end{equation}
		by using 
		\begin{equation}\label{constr_1}
		\min_{1\le j\le p}{1\over n}\sum_{i =1}^n\M_{ji} \le {1\over p}\sum_{j=1}^p{1\over n}\sum_{i =1}^n\M_{ji} = {1\over p}.
		\end{equation}
		Intuitively, both (\ref{ass_signal}) and (\ref{ass_signal_weak}) preclude  the   average frequency of each word, over all documents,  from being very small. Otherwise, if a word  rarely occurs, one cannot reasonably expect to detect/sample it:  $\|X_{j\cdot}\|_1$ will be close to $0$, and the   estimation of $R$ in (\ref{def_R_hat}) becomes problematic. For this reason, removing rare words or grouping several rare words together to form a new word are commonly used strategies in data pre-processing \citep{arora2013practical, arora2012learning,  Bansal, BleiLDA},  which we also employ   in the data analyses presented in Section \ref{sec_sim}. 

		\item[(2)] 
		To interpret the requirement $J_1 = \emptyset$, by recalling that $\wt A = D_\M^{-1}AD_W$,
		\[
		\wt A_{jk} = {n^{-1}\|W_{k\cdot}\|_1A_{jk} \over n^{-1}\|\M_{j\cdot}\|_1} \]
		can be viewed as  \[ 
		\PP(\text{Topic } k\ |\ \text{Word } j) = { \PP(\text{Topic } k)  \times \PP(\text{Word } j\ |\ \text{Topic } k ) \over \PP(\text{Word }  j)} .
		\] 			
		If  $J_1 \neq  \emptyset$,  
		then $\PP(\text{Topic } k\ |\ \text{Word } j) \approx 1$,  for a {\it quasi-anchor} word $j$. Then, quasi-anchor words also 
		determine a topic, and it is hopeless to try to distinguish them {\it exactly} from the anchor words of the same topic. However, Theorem \ref{thm_anchor} shows that in this case our algorithm places quasi-anchor words and anchor words for the same topic in the same estimated group, as soon as (\ref{ass_nu}) of Theorem \ref{thm_anchor} holds.  When we have only anchor words, and no quasi-anchor words, $J_1 = \emptyset$,  there is no possibility for  confusion. Then, we can have less  separation between the rows of $W$, $\nu > 4\delta$, and exact anchor word recovery, as shown in Corollary \ref{cor_anchor}. 
	\end{enumerate}
\end{remark}

\begin{remark}[Assumption \ref{ass_w} and condition $\nu > 4\delta$]\label{remark_nu}\mbox{}
	\begin{enumerate}
		\item[(1)] The exact recovery of anchor words in the noiseless case (Proposition \ref{prop_anchor}) relies on Assumption \ref{ass_w}, which requires  the angle between two different rows of $W$ not be too small in the following sense:
		\begin{equation}\label{eq_ass_W}
		\cos\left( \angle( W_{i\cdot},   W_{j\cdot} ) \right)< \frac{\zeta_i}{\zeta_j} \wedge  \frac{\zeta_j}{\zeta_i},\quad \text{for all $1\le i \ne j\le K$},
		\end{equation}
		with $\zeta_i:= \| W_{i\cdot} \|_2 / \|W_{i\cdot} \|_1$. Therefore, the more balanced the rows of $W$ are, the less restrictive this assumption  becomes. The most ideal case is $\min_i\zeta_i / \max_i \zeta_i \to 1$ under which (\ref{eq_ass_W}) holds whenever two rows of $W$ are not parallel, whereas the least favorable case is  $\min_i\zeta_i / \max_i \zeta_i \to 0$, for which we need the rows of $W$ close to orthogonal (the topics are uncorrelated). 
		
		Although in this work the matrix $W$ has non-random entries, it is interesting to study when (\ref{eq_ass_W}) holds, with high probability, under appropriate distributional assumptions on $W$.
		A popular and widely used distribution of the columns of $W$ is the Dirichlet distribution \citep{BleiLDA}. Lemma \ref{lem_dir} in the supplement shows that, when the columns of $W$ are i.i.d. samples from the Dirichlet distribution, (\ref{eq_ass_W}) holds with high probability, under mild conditions on the hyper-parameter of the Dirichlet distribution.
		
		\item[(2)] We prove in Lemma \ref{lem_fact} that Assumption \ref{ass_w} is equivalent with $\nu>0$, where we recall that $\nu$ has been defined in (\ref{def_nu2}). For finding the anchor words from the noisy data, we need  that $\nu > 4\delta$, a strengthened version of Assumption \ref{ass_w}.  Furthermore, Lemmas \ref{lem_fact}  and \ref{lem_delta} in the supplement guarantee that there exists a sequence $\eps_n$ such that  $\nu > 4\delta$ is implied by 
		\begin{equation}\label{eq_ass_W'}
		\cos\left( \angle( W_{i\cdot},   W_{j\cdot} ) \right)< \left(\frac{\zeta_i}{\zeta_j} \wedge  \frac{\zeta_j}{\zeta_i}\right)\left(1 - \eps_n 
		\right),\quad \text{for all $1\le i \ne j\le K$}.
		\end{equation}
		Thus, we need $\eps_n$ more separation between any two different rows of $W$ than what we require in (\ref{eq_ass_W}). Under the following balance condition of words across documents
		\begin{equation}\label{cond_word_balance}
		\max_{1\le i\le n} \M_{ji} \bigg / \left({1\over n}\sum_{i = 1}^n\M_{ji}\right) = o(\sqrt{n}),\qquad \text{for }1\le j\le p,
		\end{equation}
		Lemma \ref{lem_delta} guarantees that $\eps_n \to0$ as $n\to\infty$. 
		The same interplay between the angle of rows of $W$ and their balance condition as described in part (1) above holds. 
		We view (\ref{cond_word_balance}) as a  reasonable, mild,  balance condition, as it effectively asks 
		the maximum frequency of each particular word, across documents not  be larger  than the average frequency of that word over the $n$ documents, multiplied by $\sqrt{n}$. 
		
		If the columns of $W$ follow the Dirichlet distribution, under mild conditions on the hyper-parameter, we directly prove, in Lemma \ref{lem_dir} in the supplement, that 
		$\nu >4\delta$ holds with probability greater than $1- O(M^{-1})$ with $M:= n\vee p\vee N$. 
	\end{enumerate}
\end{remark}

\section{Estimation of the word-topic membership matrix.}\label{sec_A} 
We derive minimax lower bounds for the estimation of $A$  in topic models, with respect to the $L_1$ and $L_{1, \infty}$ losses in Section  \ref{sec_lowerbound}.
We follow with a description of our estimator $\wh A$ of $A$, in Section \ref{sec_est_A}.  In Section \ref{sec_upperbound}, we establish upper bounds on $L_1(A, \wh A)$ and $L_{1, \infty}(A, \wh A)$, for the estimator $\wh A$ constructed in Section \ref{sec_est_A}, and provide conditions under which the bounds are minimax optimal.

\subsection{Minimax lower bounds}\label{sec_lowerbound}
In this section, we establish the lower bound of model (\ref{model}) based on $L_1(\wh A, A)$ and $L_\c(\wh A, A)$ for any estimator $\wh A$ of $A$ over the parameter space 
\begin{equation}\label{classA}
\A(K,|I|, |J|) := \left\{ A\in\RR_+^{p\times K}:\ A^T\bm{1}_p = \bm{1}_K,\quad \text{$A$ has $|I|$ anchor words}\right\}
\end{equation}
where $\bm{1}_d$ denotes the $d$-dimensional vector with all entries equal to 1. 
Let 
\begin{equation}\label{def_W0}
W=  W^0 + {1\over nN}\bm{1}_K\bm{1}_K^T - {K\over nN}\bm{I}_K,\quad 
W^0 = \{\underbrace{e_1, \ldots, e_1}_{n_1}, \underbrace{e_2, \ldots, e_2}_{n_2}, \ldots, \underbrace{e_K, \ldots, e_K}_{n_K}\}
\end{equation}
with $\sum_{k =1}^{K}n_k = n$ and $|n_i-n_j| \le 1$ for $1\le i, j \le K$. We use $e_k$ and $\bm{I}_d$ to denote, respectively, the canonical basis vectors in $\RR^K$ and the identity matrix in $\RR^{d\times d}$. It is easy to verify that $W$ defined above satisfies Assumptions \ref{ass_pd} and \ref{ass_w}.
Denote by $\PP_{A}$ the joint distribution of $(X_1,\ldots,X_n)$, under model (\ref{model}) for this choice of $W$. Let $|I_{\max}| = \max_k |I_k|$. 

\begin{thm}\label{thm_lb}
	Under model (\ref{model}), assume (\ref{model_multinomial}) and let 
	$
	|I|+ K|J| \le c(nN),
	$
	for some universal constant $c>1$. Then, there exist $c_0>0$ and $c_1\in (0, 1]$ such that
	\begin{align}\label{ondergrens}
	\inf_{\wh A} \sup_{A}\PP_{A}\left\{ L_{1}(\wh A, A) \ge c_0K\sqrt{|I| + K|J| \over nN} \right\} \ge c_1.
	\end{align}
	Moreover, if $K(|I_{\max}| +|J|)\le c(nN)$ holds, we further have
	$$
	\inf_{\wh A} \sup_{A}\PP_{A}\left\{ L_\c(\wh A, A) \ge c_0\sqrt{K(|I_{\max}|+|J|) \over nN} \right\} \ge c_1.
	$$
	The $\inf_{\wh A}$ is taken over all estimators $\wh A$ of $A$; the supremum  is taken over all $A\in \A(K, |I|, |J|)$. 
\end{thm}

\begin{remark}
	The product $nN$ is the total number of sampled words, while $|I| + K|J|$ is the number of unknown parameters in  $A\in \A(K, |I|, |J|)$. 
	Since we do not make any further structural assumptions on the parameter space, we studied minimax-optimality of estimation in topic models with anchor words in the regime  
	\[   nN > c(|I| + K|J|), \] in which one can expect to be able to develop procedures for the  consistent  estimation of the matrix $A$.	
	
	%
	%
	%

	In order to facilitate the interpretation of the lower bound of the $L_1$-loss, we can rewrite the second statement in (\ref{ondergrens}) as
	\[
	\inf_{\wh A} \sup_{A\in \A(K, |I|, |J|)}\PP_{A}\left\{ { L_1(\wh A , A)  \over \|A\|_1} \ge c_0\sqrt{|I| + K|J| \over nN}  \right\} \ge c_1,
	\]
	using the fact   $\|A\|_1 = K$. Thus, the right-hand-side becomes the square root of the ratio between number of parameters to estimate and overall sample size.

\end{remark}

\begin{remark} \label{minimax_Ke}	When $K$ is known and independent of $n$ or $p$, \cite{Tracy} derived the minimax rate (\ref{tr}) of $L_1(A, \wh A)$ in their Theorem 2.2: 
	\begin{equation}\label{tr}
	\inf_{\wh A}\sup_{A\in \A(p,K)}\PP\left\{L_1(A, \wh A) \ge c_1\sqrt{p\over nN}\right\} \ge c_2
	\end{equation}
	for some constants $c_1, c_2>0$. The parameter space considered in (\cite{Tracy}) for the derivation of the lower bound in (\ref{tr}) is 
	$$
	\A(p, K) = \{A \in \RR_+^{p\times K}: A^T\bm{1}_p = \bm{1}_K, \|A_{j\cdot}\|_1 \ge c_3/p, \forall j\in [p]\}
	$$
	for some constant $c_3>0$,  and the lower bound is independent of $K$.
	%
	In contrast, the lower bound in Theorem \ref{thm_lb} holds over $\A(K, |I|, |J|) \subseteq \A(p, K)$, and the dependency on $K$ in (\ref{ondergrens}) is explicit. The upper bounds derived for  $L_1(A, \wh A)$ in both this work and \cite{Tracy} correspond to $A \in \A(K, |I|, |J|)$, making the latter the appropriate space for discussing attainability of lower bounds.  
	
	Nevertheless, we notice that, when $K$ is treated as a fixed constant, and recalling that $|I| + |J| = p$, the lower bounds over both spaces have the same {\it order} of magnitude, $\sqrt{p/nN}$. From this perspective, when $K$ is fixed, the result in  \cite{Tracy} can be viewed as a minimax result over the smaller parameter space. 
	
	A non-trivial modification of  the  proof in \cite{Tracy}  allowed us 
	to recover the dependency on $K$ that was absent in their 
	original lower bound (\ref{tr}): the corresponding rate is  $\sqrt{pK/nN}$, and it is relative to estimation over 
	the larger parameter space $\A(p, K)$. For comparison purposes, we note that this space corresponds to $\A(K, |I|, |J|)$, with $I = \emptyset$ and $|J| = p$. In this case, our lower bound (\ref{ondergrens}) becomes $K\sqrt{pK/nN}$, larger by a factor of $K$ than the bound that can be derived by modifying arguments in \cite{Tracy}. Therefore, Theorem \ref{thm_lb} improves upon existing lower bounds for estimation in topic models without anchor words and with a growing number of topics, and offers the first minimax lower bound for estimation in topic models with anchor words and a growing $K$.


\end{remark}

\subsection{An estimation procedure for $A$}\label{sec_est_A}
Our estimation procedure follows the constructive proof of  Proposition \ref{prop_A}.
Given the set of estimated anchor words $\wh \I = \{\wh I_1, \ldots, \wh I_{\wh K} \}$, we begin by selecting a set of representative indices  of words per topic, by  choosing  $\wh i_k\in \wh I_k$ at random, to form $\wh L:= \{\wh i_{1}, \ldots, \wh i_{\wh K}\}$.
As we explained in the proof of Proposition  \ref{prop_A}, 
we first estimate a normalized version of $A$, the matrix $B= AA_L^{-1}$. We estimate separately $B_{I}$ and $B_{J}$. 
In light of (\ref{eq_iden_BI}), we estimate the $|I| \times K$ matrix $B_I$ by
\begin{equation}\label{est_BI}
\wh B_{ik} = \left\{\begin{array}{ll}
{\|X_{i\cdot}\|_1\ \big /\ \|X_{\wh i_k\cdot}\|_1}, &\text{ if } i \in \wh I_k \text{ and } 1\le k\le \wh K\\
0, &\text{ otherwise }.
\end{array} \right.
\end{equation}
Recall from (\ref{bj}) that  $B_J = \Theta_{JL}\Theta_{LL}^{-1}$ and that Assumption \ref{ass_pd} ensures that $\Theta_{LL} := A_LCA_L$ is invertible, with $\Theta$ defined in (\ref{teta}). Since we have already obtained 
$\wh I$,  we can estimate $J$ by $\wh J = \{1, \ldots, p\} \setminus \wh I$. We then use the estimator $\wh \Theta$ given in (\ref{est_Theta}), to estimate $\Theta_{JL}$ by $ \wh \Theta_{\wh J\wh L}$. 
It remains to estimate the $K \times K$ matrix  $\Omega:= \Theta_{LL}^{-1}$. For this, 
we solve the linear program
\begin{eqnarray}\label{obj_omega_prev}
(\hat t,\  \wh\Omega)&=\arg\min\limits_{t\in\RR^{+},\ \Omega \in \RR^{\wh K \times \wh K}}t&
\end{eqnarray}
subject to
\begin{equation}\label{constr_omega_prev}
\bigl\|\Omega \wh \Theta_{\wh L\wh L}-I\bigr\|_\r\le  \lambda t, \quad \|\Omega\|_{\infty,1}\le t,
\end{equation}
with $\lambda = C_0\max_{i\in \wh L}\sum_{j\in \wh L}\eta_{ij}$,  where $\eta_{ij}$ is defined such that $|\wh \Theta_{ij} - \Theta_{ij}|\le C_0\eta_{ij}$ for all $i,j\in [p]$, with high probability, and $C_0$ is a universal constant. The precise expression of $\eta_{ij}$ is given in Proposition \ref{prop_sigma} of Appendix \ref{sec_stat_sigma}, see also Remark \ref{etaij} below. To accelerate the computation, we can decouple the above optimization problem, and solve instead  $\wh K$ linear programs separately. We estimate $\Omega$ by $\wh \Omega = (\wh \omega_1, \ldots, \wh \omega_{\wh K})$ where, for any $k = 1,\ldots, \wh K$,
\begin{eqnarray}\label{obj_omega}
\wh\omega_k&:=\arg\min\limits_{\omega\in \RR^{\wh K }}\|\omega\|_1&
\end{eqnarray}
subject to
\begin{equation}\label{constr_omega}
\bigl\|\wh \Theta_{\wh L\wh L}\omega-e_k\bigr\|_1\le  \lambda \|\omega\|_1
\end{equation}
with $e_1, \ldots, e_{\wh K}$ denoting the canonical basis  in $\RR^{\wh K}$.
After constructing $\wh \Omega$ as above, we estimate $B_J$ by
\begin{equation}\label{est_B}
\wh B_{\wh J}=\left(\wh \Theta_{\wh{J}\wh{L}}\wh \Omega\right)_+,
\end{equation}
where the operation $(\cdot )_+=\max(0,\cdot)$ is applied entry-wise. Recalling that $A_L$ can be determined from $B$ via  (\ref{AL}), the combination of (\ref{est_B}) with (\ref{est_BI}) yields $\wh B$ and hence the desired estimator of $A$: 
\begin{equation}\label{est_A}
\wh A = \wh B\cdot \diag\left(\|\wh B_{\cdot 1}\|_1^{-1}, \ldots, \|\wh B_{\cdot \wh K}\|_1^{-1}\right).
\end{equation}

\begin{remark}
	The decoupled linear programs given by  (\ref{obj_omega}) and (\ref{constr_omega}) are computationally attractive and can be done in parallel. This improvement over (\ref{obj_omega_prev}) becomes significant when $K$ is large. 
\end{remark}
\begin{remark}\label{L}
	Since we can select all anchor words with high probability,  as shown in Theorem \ref{thm_anchor},  in practice we can repeat randomly selecting different sets of representatives $\wh L$ from $\wh I$ several times, and we can estimate $A$ via (\ref{est_BI}) -- (\ref{est_A}) for each $\wh L$. The entry-wise average of these estimates inherits, via Jensen's inequality,  the same theoretical guarantees shown in Section \ref{sec_upperbound}, while benefiting from an improved numerical performance.
\end{remark}

\begin{remark}\label{etaij}
	To preserve the flow of the  presentation we refer  to Proposition \ref{prop_sigma} of Appendix \ref{sec_stat_sigma} for the  precise expressions of  $\eta_{ij}$  used in constructing the tuning parameter $\lambda$.  The estimates of $\eta_{ij}$, recommended for practical implementation,  are shown in (\ref{def_eta3}) based on Corollary \ref{cor_theta_R} in Appendix \ref{sec_stat_sigma}.
	We also note that in precision matrix estimation, $\lambda$ is proportional, in our notation, to the norm  $\|\wh \Theta_{LL} - \Theta_{LL}\|_{\infty}$, see, for instance, \cite{LOVE} and the references therein for a similar construction, but devoted to general sub-Gaussian distributions. In this work, the data is  multinomial, and we  exploited this fact to propose  a more refined tuning parameter, based on entry-wise control. 
\end{remark} 

We summarize our procedure, called {\sc Top}, in the following algorithm. 
\begin{algorithm}[H]
	\caption{Estimate the word-topic matrix $A$}\label{alg2}
	\begin{algorithmic}[1]
		\Require frequency data matrix $X\in\RR^{p\times n}$ with document lengths $N_1, \ldots, N_n$; two positive constants $C_0, C_1$ and positive integer $T$
		\Procedure{Top}{$X, N_1,\ldots, N_n; C_0, C_1$}
		\State compute $\wh \Theta$ from (\ref{est_Theta}) and $\wh R$ from (\ref{def_R_hat})
		\State compute $\wh \eta_{ij}$ and $Q[i, j]:=C_1 \wh\delta_{ij}$ from (\ref{delta3}) and (\ref{def_eta3}), for $i, j \in [p]$
		\State estimate $\I$ via \textsc{FindAnchorWords}$(\wh R, Q)$
		\For {$i = 1,\ldots, T$}
		\State randomly select $\wh L$ and solve $\wh \Omega$ from (\ref{obj_omega}) by using $\lambda = C_0\max_{i\in \wh L}\sum_{j\in \wh L}\wh \eta_{ij}$ in (\ref{constr_omega})
		\State estimate $B$ from (\ref{est_BI}) and (\ref{est_B})
		\State compute $\wh A^i$ from (\ref{est_A})
		\EndFor
		\State \Return $\wh \I = \{ \wh I_1, \wh I_2, \ldots, \wh I_{\wh K}\}$ and $\wh A = T^{-1}\sum_{i =1}^T\wh A^i$
		\EndProcedure
	\end{algorithmic}
\end{algorithm}

\subsection{Upper bounds of the estimation rate of $\wh A$}\label{sec_upperbound}
In this section we derive upper bounds for estimators $\wh A$ constructed in Section \ref{sec_est_A}, under the matrix $\|\cdot \|_1$ and $\|\cdot \|_{\c}$ norms. 
$\wh A$ is obtained by choosing the tuning parameter $\lambda = C_0\max_{i\in \wh L}\sum_{j\in \wh L}\eta_{ij}$ in the optimization (\ref{obj_omega}). To simplify notation and properly adjust the scales, we define
\begin{equation}\label{def_alpha_gamma}
\alpha_j := p\max_{1\le k\le K} A_{jk}, \qquad {\g_k} := {K\over n}\sum_{i=1}^nW_{ki},\qquad \text{ for each }j\in[p], \ k\in[K],
\end{equation}
such that $\sum_{k=1}^K\g_k = K$ and $p\le \sum_{j=1}^p \alpha_j \le pK$ from (\ref{col_sum_one}). We further set 
\begin{equation}\label{def_oaua}
\oa_I =\max_{i\in I}\alpha_i, \quad  \ua_I =\min_{i\in I}\alpha_i,\quad \rho_j = \alpha_j / \oa_I,\quad \og = \max_{1\le k\le K} \g_k,\quad
\ug = \min_{1\le k\le K} \g_k. 
\end{equation}
For future reference, we note that \[ \og \geq 1\ge \ug.\] 	
\begin{thm}\label{thm_rate_Ahat}
	Under model (\ref{model}), Assumptions \ref{ass_sep} and \ref{ass_pd}, assume $\nu > 4\delta$, $J_1 = \emptyset$ and (\ref{ass_signal}).
	Then, with probability $1-8M^{-1}$, we have
	\begin{align*}
	&\min_{P\in \H_K}\left\|\wh A_{\cdot k} - (AP)_{\cdot k}\right\|_1~ \le  ~\text{Rem}(I, k)+ \text{Rem}(J, k),\qquad \text{for all $1\le k\le K$,}
	\end{align*}
	where $\H_K$ is the space of $K\times K$ permutation matrices,
	\begin{align*}
	\text{Rem}(I, k) &~\lesssim~ \sqrt{K\log M \over npN}\cdot \sum_{i\in I_k}{\alpha_i \over \sqrt{\ua_I \ug}},\\
	\text{Rem}(J, k) & ~\lesssim~  \sqrt{K\log M \over nN}\cdot {\og^{1/2}\|C^{-1}\|_\r  \over K}\cdot {\oa_I \over \ua_I} \left(\sqrt{|J| + \sum_{j\in J}\rho_j} +{\oa_I \over \ua_I}\sqrt{K\sum_{j\in J}\rho_j}\right).
	\end{align*}
	\noindent Moreover, summing over $1\le k\le K$,  yields 
	\begin{align*}
	L_1(A, \wh A)~&\lesssim ~\sum_{k = 1}^K \text{Rem}(I, k)+ \sum_{k = 1}^K \text{Rem}(J, k).
	\end{align*}
\end{thm}	

In Theorem \ref{thm_rate_Ahat} we explicitly state bounds on Rem($I, k$) and Rem($J, k$), respectively, which allows us to separate out the error made in the estimation of the rows of $A$ that correspond to anchor words from that corresponding to non-anchor words.  This facilitates the statement and explanation of the quantities that play a key role in this rate, and of the conditions under which our estimator achieves near minimax optimal rate, up to a logarithmic factor of $M$. We summarize it in the following corollary and the remarks following it.  Recall that $C = n^{-1}WW^T$.

\begin{cor}[Attaining the optimal rate]\label{cor_opt_rate}
	In addition to the conditions in Theorem \ref{thm_rate_Ahat},
	suppose
	\begin{enumerate}
		\item[(i)] $\oa_I \asymp \ua_I$, \ $\sum_{j\in J}\rho_j \lesssim |J|
		$ 
		\item[(ii)] $\og \asymp \ug$ ,\ \ $\sum_{k'\ne k}\sqrt{C_{kk'}} = o(\sqrt{C_{kk}})$ for any $1\le k\le K$
	\end{enumerate}
	hold. Then with probability $1-8M^{-1}$, we have
	\begin{eqnarray}\label{habit}
	L_\c(A, \wh A) ~\lesssim~ \sqrt{K(|I_{\max}|+|J|)\log M \over nN}, \quad L_1(A, \wh A) ~ \lesssim~  K\sqrt{|I| + K|J|\log M \over nN}. 
	\end{eqnarray}			
\end{cor}

\begin{remark}\label{rem-noise-control}
	The optimal estimation rate depends  on the bounds for  $\wh \Theta_{j\ell}-\Theta_{j\ell}$ and $\wh R_{j\ell}-R_{j\ell}$ derived via a careful analysis  in Proposition \ref{prop_sigma} in Appendix  \ref{sec_stat_sigma}. We rule out quasi-anchor words ($J_1 = \emptyset$, see Remark \ref{remark1} as well) since otherwise, the presentation, analysis and proofs   will become much more cumbersome.
\end{remark}

\begin{remark}[Relation between document length $N$ and dictionary size $p$]\label{rem-cond-p,N}
	Our procedure can be implemented for any $N\ge2$. However, Theorem \ref{thm_rate_Ahat} and Corollary \ref{cor_opt_rate}   indirectly impose some restrictions on $N$ and $p$. Indeed, the restriction
	\begin{equation}\label{constr}
	N \ge (2p\log M / 3) \vee (9p\log^2 M / K) 
	\end{equation}
	is subsumed by  (\ref{ass_signal}), via 
	(\ref{constr_1})
	and 
	\[
	\min_{1\le j\le p}\max_{1\le i\le n}\M_{ji} = \min_{1\le j\le p}\sum_{k=1}^K A_{jk}\max_{1\le i\le n}W_{ki} \le{1\over p}\sum_{j=1}^p\sum_{k=1}^K A_{jk} = {K\over p}.
	\]
	Inequality (\ref{constr}) describes the regime for  which we establish the upper bound results in this section, and are able to show minimax optimality, as the lower bound restriction 
	$cnN\ge |I|+|J|K$ for some $c>1$ in Theorem \ref{thm_lb} implies $N\ge p /(cn)$.

	We can extend the range (\ref{constr}) of $N$ at the cost of a stronger condition  than (\ref{ass_nu}) on $\nu$.
	Assume (\ref{ass_nu_prime}) holds with  $\delta' = \max_{j,\ell}\delta_{j\ell}'$ and with $\delta_{j,\ell}'$   defined in Corollary \ref{cor_delta} of Appendix \ref{sec_stat_sigma}.
	In that case,   as in Remark \ref{remark1},
	condition (\ref{ass_signal}) can be relaxed to (\ref{ass_signal_weak}).  Provided
	\begin{equation}\label{ass_signal_weak_mu}
	\min_{j\in I}{1\over n}\sum_{i =1}^n\M_{ji} \ge {c\log M \over N},\qquad \min_{j\in I}\max_{1\le i\le n}\M_{ji} \ge {c'(\log M)^2 \over N}
	\end{equation}
	for some constant $c, c'>0$,  we prove in Appendix \ref{sec_proof_rem_noise_control} of the supplement that Theorem \ref{thm_rate_Ahat} and Corollary \ref{cor_opt_rate} still hold. 
	As discussed in  Remark \ref{remark1}, condition (\ref{ass_signal_weak})  implies 
	(\ref{constr_0}), \[ N \ge c\cdot (p\log M)/n, \] which is a
	much weaker restriction  on $N$ and $p$ than (\ref{constr}). Condition (\ref{ass_signal_weak_mu}) in turn is  weaker than (\ref{ass_signal}) as it only   restricts  the smallest (averaged over documents) frequency of anchor  words. As a result, (\ref{ass_signal_weak_mu}) does not necessarily imply  the constraint   (\ref{constr}). For instance, if 
	$\min_{j\in I}n^{-1}\sum_{i=1}^n \M_{ji} \gtrsim 1/{ |I|},$ then (\ref{ass_signal_weak_mu}) is implied by 
	$
	N \gtrsim |I|(\log M)^2.
	$
	The problem of developing a procedure that can be shown to be minimax-rate optimal in the absence of  condition (\ref{ass_signal_weak_mu}) is left open for future research.

\end{remark}

\begin{remark}[Interpretation  of the conditions of Corollary \ref{cor_opt_rate}]
	\mbox{}
	\vspace{1.2mm}
	
	\noindent(1)  {\it Conditions regarding anchor words}.          
	Condition $\oa_I \asymp \ua_I$ implies that all anchor words, across  topics,  have the same order of frequency. The second condition $\sum_{j\in J}\rho_j \lesssim |J|$ is equivalent with $|J|^{-1}\sum_{j\in J}\|A_{j\cdot}\|_\i \lesssim \max_{i\in I}\|A_{i\cdot}\|_\i$. Thus it holds when the averaged frequency of non-anchor words is no greater, in order, than the largest frequency among all anchor words.
	\vspace{1.2mm}\\
	\noindent (2) {\it Conditions regarding the topic matrix $W$.}     Condition (ii) implies that  the topics are balanced, through $\og \asymp \ug$, and prevents too strong a linear dependency between the rows in $W$, via  $\sum_{k'\ne k}\sqrt{C_{kk'}} = o(\sqrt{C_{kk}})$. As a result, we can show $\|C^{-1}\|_\r = O(K)$ (see Lemma \ref{lem_C_inv} in the supplement) and the rate of Rem($J, k$) in Theorem \ref{thm_rate_Ahat} can be improved by a factor of $1/\sqrt{K}$. 
	The  most favorable situation under which condition (ii) holds corresponds to the extreme case when each document  contains a prevalent topic $k$, in that the corresponding $W_{ki} \approx 1$, and the topics are approximately balanced across documents, so that approximately $n/K$ documents cover the same prevalent topic.  The minimax lower bound is also derived based on this ideal structure of $C$. At the other extreme, all topics are equally likely to be covered in each document, so that $W_{ki} \approx 1/K$, for all $i$ and $k$.  In the latter case, $\og \asymp \ug \approx 1$, but $\|C^{-1}\|_\r$ may be larger, in order, than $K$ and the rates in  Theorem \ref{thm_rate_Ahat} are slower than the optimal rates by at least a factor of $\sqrt{K}$.  When $K$ is fixed or comparatively small, this loss is ignorable. Nevertheless, our condition (ii) rules out this extreme case, as in general we do not expect any of the given documents to cover, in the same proportion, all of the $K$ topics we consider. 
\end{remark}  

\begin{remark}[Extensions]
	Both our procedure and the outline of our analysis can be naturally extended to the  more general Nonnegative Matrix Factorization (NMF) setting, and to different data generating distributions, as long as Assumptions \ref{ass_sep}, \ref{ass_pd} and \ref{ass_w} hold, by adapting the  control of the stochastic error terms $\eps$. 
\end{remark}

\subsubsection{	Comparison with the rates of other existing estimators} \label{sec_compare_rate}
As mentioned in the Introduction, 
the rate analysis of estimators in topic models received very little attention, with the two exceptions discussed below, both assuming that $K$ is known in advance. 

An  upper bound on $L_1(\wh A, A)$ has been established in \cite{arora2012learning, arora2013practical}, for the estimators $\wh A$ considered in these works, and different than ours. Since the estimator of \cite{arora2013practical} inherits the rate of \cite{arora2012learning}, we only discuss the latter rate, given below: 
\[
L_1(A, \wh A) ~\lesssim ~{a^2K^3\over \Gamma\delta_p^3}\cdot \sqrt{\log p \over nN}.
\]
Here $a$ can be viewed as $\og/\ug$, $\Gamma$ can be treated as the $\ell_1$-condition number of $C = n^{-1} WW^T$ and $\delta_p$ is the smallest non-zero entry among all the anchor words,  and corresponds to   $\ua_I / p$,  in our notation.  To understand the order of magnitude of this bound, we evaluate it in the most favorable scenario, that of $W=W^0$ in (\ref{def_W0}). Then $\Gamma \le \sqrt{K}\sigma_{\min}(C) \lesssim 1/\sqrt{K}$, where $\sigma_{\min}(C)$ is the smallest eigenvalue of $C$, and  $\ug\asymp \og$.  Since $\sum_{j\in J}\rho_j \lesssim |J|$ implies $\ua_I \gtrsim p^{-1}\sum_{j=1}^p\alpha_j $ and $p\le \sum_{j=1}^p\alpha_j \le pK$, suppose  also $\ua_I \ge K$. Then, the above rate becomes
\[
L_1(A, \wh A) ~\lesssim ~ p^3\cdot \sqrt{K\log p \over nN},
\]
which is slower than what we obtained in (\ref{habit}) by at least a factor of $ (p^5\log p)^{1/2}/K$.

The upper bound on $L_{1}(\wh A, A)$ in  \cite{Tracy} is derived for $K$ fixed, under a number of non-trivial assumptions on $\Pi$, $A$ and $W$ given in their work.  Their rate analysis does not assume all anchor words have the same order of frequency but requires that  the number of anchor words in each topic grows as $p^2\log^2(n) / (nN)$ at the estimation level. With an abundance of anchor words, the estimation problem becomes easier, as there will be fewer parameters to estimate. If this  assumption does not hold, the error upper bound established in Theorem 2.1 of \cite{Tracy}, for fixed $K$,  may become sub-optimal by factors in $p$. 
In contrast, although in our work we allow for the existence of more anchor words per topic, we only {\it require} a minimum of one anchor word per topic. 

To further understand how the number of anchor words  per topic  affects the estimation rate, we consider the extreme example, used for illustration purposes only, of $I = \{1, \ldots, p\} := [p]$, when all words are anchor words. Our Theorem \ref{thm_lb} immediately shows that in this case the minimax lower bound for $L_1(\wh A ,A)$  becomes
\[
\inf_{\wh A}\sup_{A\in \A(K, p, 0)} \PP_A\left\{ L_1( \wh A , A)\ge c_0K\sqrt{p \over nN} \right\} \ge c_1
\]
for two universal constant $c_0, c_1>0$, where the infimum is taken over all estimators $\wh A$. 
Theorem \ref{thm_rate_Ahat} shows that our estimator does indeed attain this rate when when $\ug \asymp 1$ and $\min_{i\in I}\|A_{i\cdot}\|_1 \gtrsim K/p$. This rate becomes  faster (by a factor $\sqrt{K}$), as expected,  since there is only one non-zero entry of each row of $A$ to estimate. These considerations show that when we return to the realistic case in which an unknown  subset of the words are anchor words, the bounds $L_1(A, \wh A)$, for our estimator $\wh A$,  only increase at most by an optimal factor of  $\sqrt{K}$, and not by factors depending on $p$. 

\section{Experimental results}\label{sec_sim}

\paragraph{Notation:}
Recall that $n$ denotes the number of documents, $N$ denotes the number of words drawn from each document, $p$ denotes the dictionary size, $K$ denotes the number of topics, and $|I_k|$ denotes the cardinality of anchor words for topic $k$. We write $\xi:= \min_{i\in I} K^{-1} \sum_{k=1}^KA_{ik}$ for the minimal  average frequencies of anchor words $i$. The quantity $\xi$ plays the same role in our work as $\delta_p$ defined in the \emph{separability assumption} of \cite{arora2013practical}.
Larger values are more favorable for estimation. 

\paragraph{Data generating mechanism:}
For each document $i\in [n]$, we randomly generate the topic vector $W_i\in\RR^{K}$ according to the following principle. We first randomly choose the cardinality $s_i$    of $W_i$ from the integer set  $\{1,\ldots\lfloor K/3\rfloor\}$. Then we randomly choose its support  of  cardinality  $s_i$ from $[K]$. 
Each entry of the chosen support is then generated from Uniform$(0, 1)$. Finally, we normalize $W_i$ such that it sums to $1$. In this way, each document 
contains a (small) subset of topics instead of all possible topics. 

Regarding the word-topic matrix $A$, we first generate its anchor words by putting $A_{ik}:=K\xi$ for any $i\in I_k$ and $k\in [K]$. Then, each entry of non-anchor words is sampled from a Uniform$(0,1)$ distribution. Finally, we normalize each sub-column  $A_{Jk}\subset A_{\cdot k}$ to have sum  $1- \sum_{i\in I}A_{ik}$.

Given the matrix $A$ and $W_i$, we  generate the $p$-dimensional column $NX_i$ by independently drawing $N$ samples from a Multinomial$_p(N, AW_i)$ distribution. 

We consider the setting $N = 1500$, $n=1500$, $p = 1000$, $K = 30$, $|I_k| = p/100$ and $\xi = 1/p$ as our benchmark setting. 

\paragraph{Specification of the tuning parameters in our algorithm.}
In practice,  based on Corollary \ref{cor_theta_R} in Appendix \ref{sec_stat_sigma}, we recommend the choices  
\begin{equation}\label{delta3}
\wh \delta_{j\ell} =  {n^2 \over\| X_{j\cdot}\|_1 \|  X_{\ell \cdot}\|_1 }
\!\left\{\wh \eta_{j\ell}  +2 \wh\Theta_{j\ell}   \sqrt{\log M \over n }
\!\left[\! \frac{n}{\| X_{j\cdot}\|_1} \!\left( \frac{1}{n} \sum_{i=1}^n \frac{X_{ji} }{ N_i} \right)^{\rs{1\over 2}}\!\!\!+  \!\frac{n}{\| X_{\ell\cdot}\|_1}  \!\left( \frac{1}{n} \sum_{i=1}^n \frac{X_{\ell i} }{ N_i} \right)^{\rs{1\over 2}}
\right]\right\}
\end{equation}
and
\begin{align}\label{def_eta3}
\wh \eta_{j\ell} =  &~3\sqrt{6}\left( \left\| X_{j\cdot}\right\|_\infty^{1\over 2} +\left\| X_{\ell\cdot}\right\|_\infty^{1\over 2}\right) \sqrt{ \log M\over n}\left(\frac{1}{n} \sum_{i=1}^n \frac{ X_{ji} X_{\ell i} }{N_i}\right)^{1\over 2}+
\\\nonumber &+  {2 \log M \over n}\left( \| X_{j\cdot}\|_\infty + \| X_{\ell\cdot}\|_\infty \right) \frac1n \sum_{i=1}^n {1\over N_i}
+ 31 \sqrt{(\log M)^4 \over n}\left({1\over n}\sum_{i =1}^n{X_{ji} + X_{\ell i} \over N_i^3} \right)^{\rs \frac12} 
\end{align}
and 
set 
$C_0 = 0.01$ and $C_1 = 1.1$ in Algorithm \ref{alg2}.
We found that these  choices for $C_0$ and $C_1$ not only give good overall performance, but are robust as well. To verify this claim,
we generated $50$ datasets  under a benchmark setting of $N = 1500$, $n=1500$, $p = 1000$, $K = 30$, $|I_k| = p/100$ and $\xi = 1/p$. 
We first applied our Algorithm \ref{alg2} with $T=1$
to each dataset by setting $C_1=1.1$ and varying $C_0$ within the grid  $\{0.001, 0.003, 0.005, \ldots, 0.097, 0.099\}$. The  estimation error  $L_1(\wh A,A)/ K$, averaged over the 50 datasets, is shown in Figure \ref{fig_robust_C_0} and clearly demonstrates that our algorithm is   robust to the choice of $C_0$ in terms of  {overall estimation error}. In addition, we applied Algorithm \ref{alg2} by keeping $C_0 = 0.01$ and varying $C_1$ from $\{0.1, 0.2,\ldots,11.9, 12\}$. Since $C_1$ mainly controls the selection of anchor words in Algorithm \ref{alg1}, we averaged  the estimated topics number $\wh K$, \emph{sensitivity} $|\wh I \cap I|/ \ |I|$ and \emph{specificity} $|\wh I^c \cap I^c|/ \ |I^c|$ of the selected anchor words  over the $50$ datasets. Figure \ref{fig_robust_C_1} shows that Algorithm \ref{alg1} recovers all anchor words by choosing any $C_1$ from the whole range of $[1, 10]$ and consistently estimates the number of topics for all $0.2\le C_1 \le 10$, which strongly supports the robustness of Algorithm \ref{alg1}   relative to the choice of the tuning parameter $C_1$.\\

{	\begin{figure}[ht]
		\centering
		\begin{tabular}{cc}
			\includegraphics[width =0.34\textwidth]{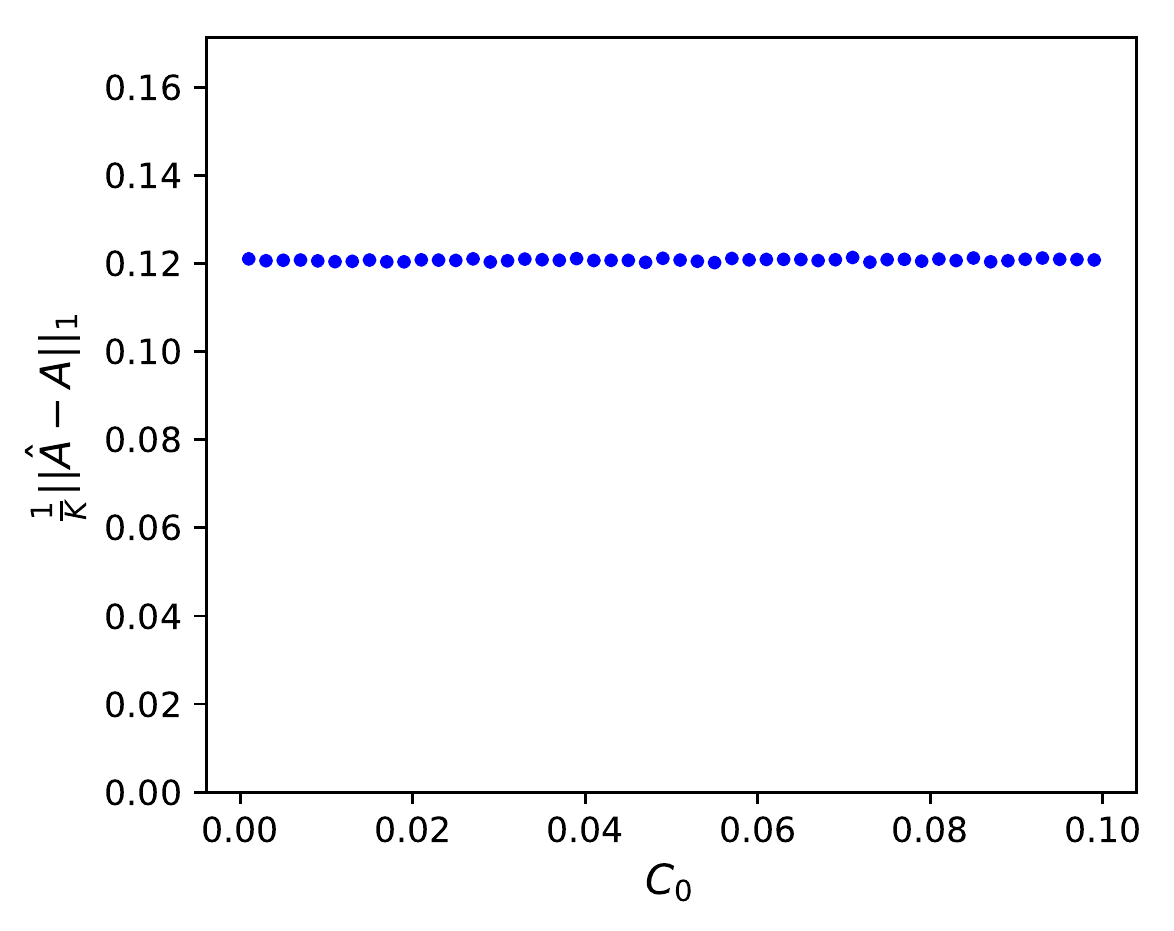} &\hspace{4pt}
			\includegraphics[width =0.35\textwidth]{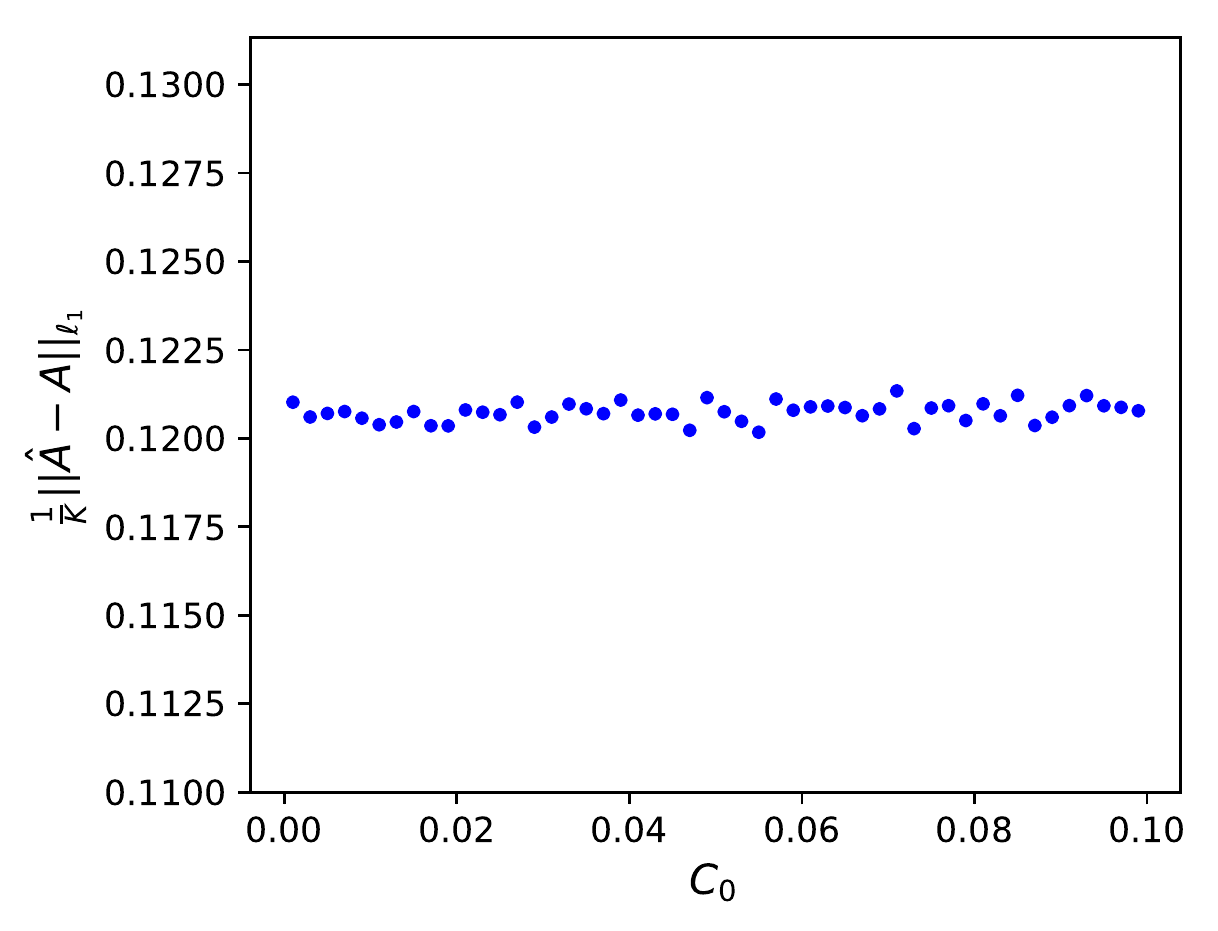}
		\end{tabular}
		\vspace{-2mm}
		\caption{Plots of \emph{overall estimation error} vs $C_0$. The right plot is zoomed in.}
		\label{fig_robust_C_0}
	\end{figure}
	\vspace{-5mm}
	\begin{figure}[ht]
		\centering
		\begin{tabular}{cc}\hspace{4pt}
			\includegraphics[width =0.33\textwidth]{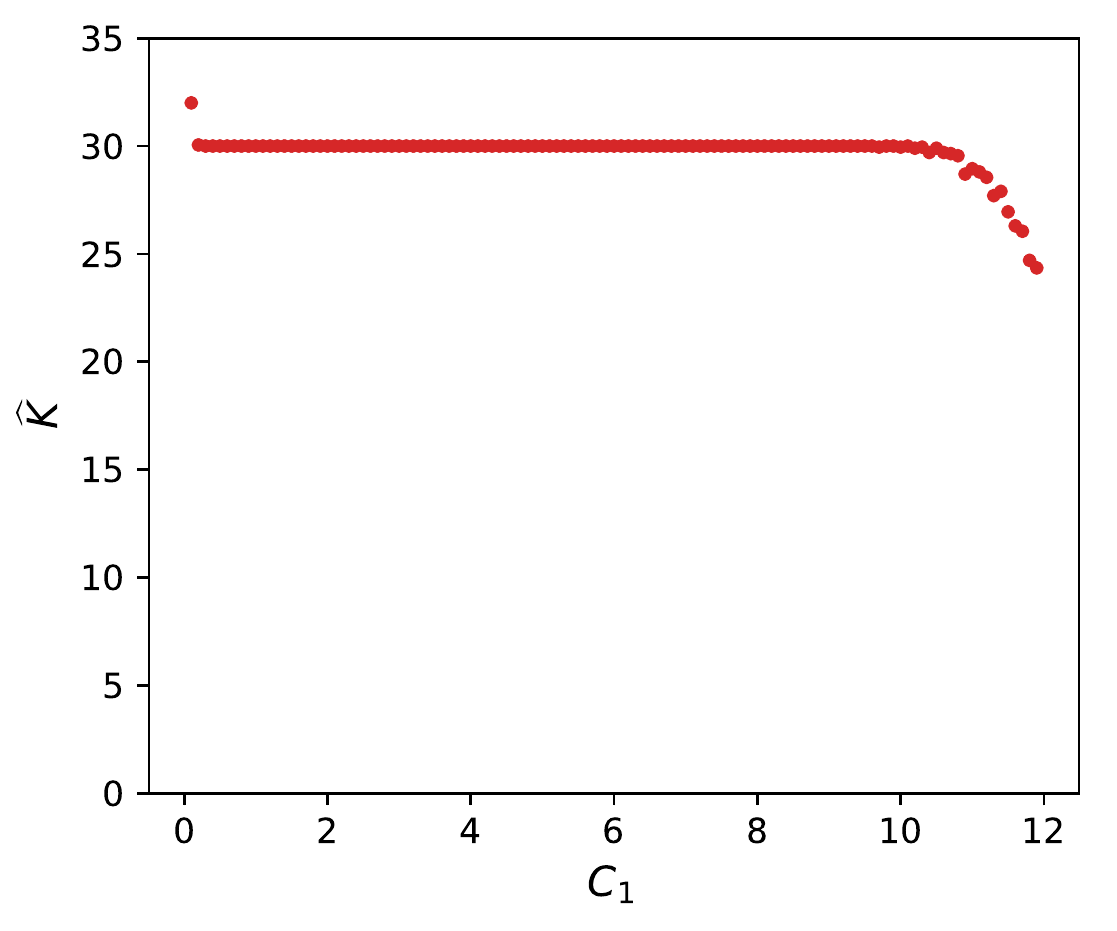} & \hspace{15pt}
			\includegraphics[width =0.32\textwidth]{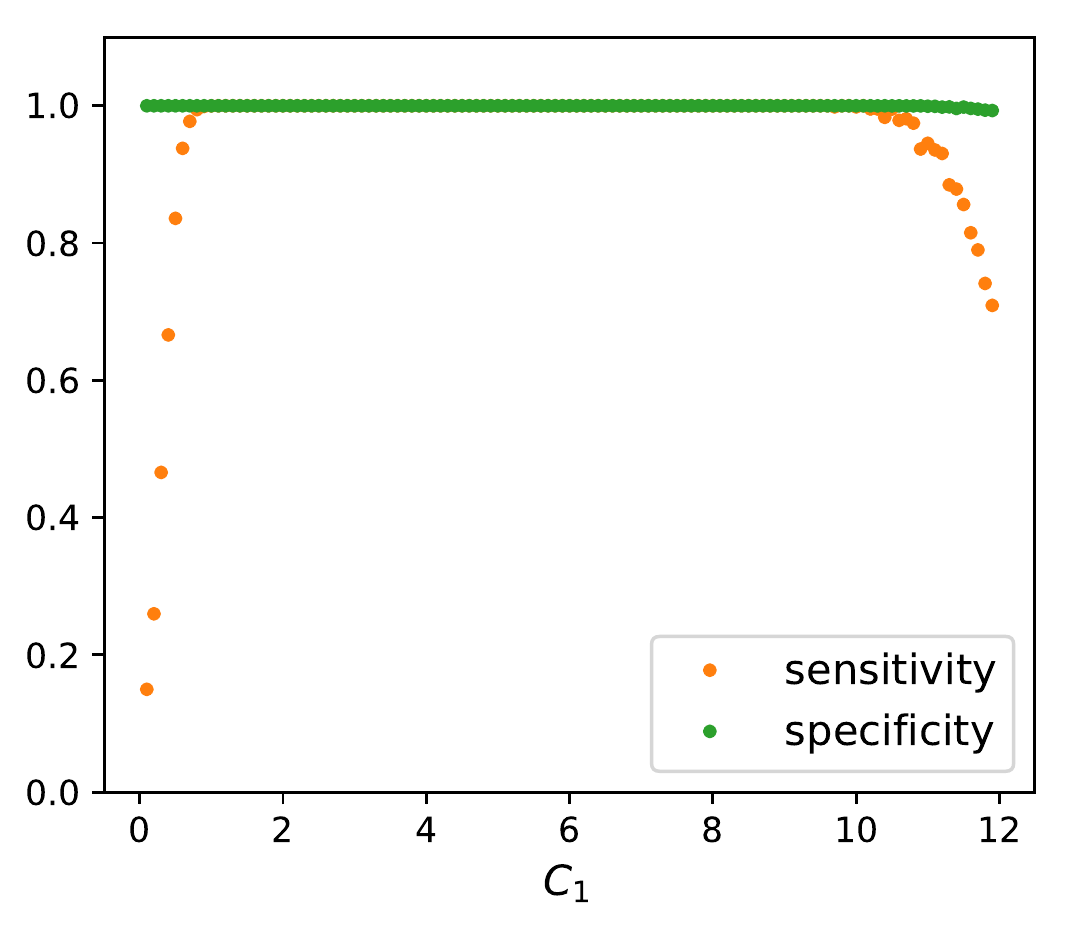}
		\end{tabular}
		\vspace{-2mm}
		\caption{Plots of $\wh K$, \emph{sensitivity} and \emph{specificity} vs $C_1$ when the true $K_0 = 30$.}
		\label{fig_robust_C_1}
	\end{figure}
}

\noindent
Throughout, we consider two versions of our algorithm: {\sc Top1} and {\sc Top10} described in Algorithm \ref{alg2} with $T=1$ and  $T = 10$, respectively. 
We compare {\sc Top} with best performing  algorithm available, that of   \cite{arora2013practical}.  We denote this algorithm by \textsc{Recover-L2} and \textsc{Recover-KL} depending on which 
loss function is used for estimating non-anchor rows in their Algorithm 3. 
In Appendix \ref{app_sim} we conducted  a small simulation study  to compare  these  two methods, and ours, with the recent procedure  of \cite{Tracy}, using the implementation the authors kindly made available to us.  Their method is tailored to topic models with a  known, small, number of topics. Our study revealed that, in the ``small $K$" regime, their procedure is comparable or outperformed by existing methods. 
Latent Dirichlet Allocation (LDA) \citep{BleiLDA} is a popular Bayesian approach to topic models,  but is computationally demanding.


The procedures from 
\cite{arora2013practical} 
have better performance than LDA  in terms of overall loss and computational cost, as evidenced by their simulations.
For this reason, 
we only focus on the comparison of our method  with \textsc{Recover-L2} and \textsc{Recover-KL} for the synthetic data. The comparison with \textsc{LDA} is considered in the semi-synthetic data.

We report the findings of our simulation studies in this section by showing that our algorithms estimate both the number of topics  and anchor words consistently, and have superior performance in terms of estimation error as well as computational time in various settings over the existing algorithms.

{\it We re-emphasize that in  all the comparisons presented below, the  existing methods have as input the true $K$ used to simulate the data, while we also estimate $K$. In Appendix \ref{sec_sim_K}, we show that these   algorithms  are  very sensitive to the choice of $K$. This demonstrates that correct estimation of $K$ is indeed highly critical for the  estimation of the entire  matrix $A$.  }

\subsection*{Topics and anchor words recovery}
{\sc Top10} and {\sc Top1} use the same procedure (Algorithm \ref{alg1}) to  select the anchor words, likewise for \textsc{Recover-L2} and \textsc{Recover-KL}. 
We present in Table \ref{table_anchor} the observed \emph{sensitivity} $ {|\wh I\cap I| / |I|}$ and \emph{specificity} ${| \wh I^c \cap I^c |/ |I^c|}$ of selected anchor words  in the benchmark setting with $|I_k|$ varying. It is clear that {\sc Top} recovers all anchor words and estimates the topics number $K$ consistently. All algorithms are performing perfectly for not selecting non-anchor words. We emphasize that the correct $K$ is given for procedure \textsc{Recover}.

\begin{table}[H]
	\centering
	\caption{Table of anchor recovery and topic recovery for varying $|I_k|$.}
	\label{table_anchor}
	\resizebox{\textwidth}{!}{
		\begin{tabular}{|c|ccccc|ccccc|}
			\hline
			Measures & \multicolumn{5}{c|}{{\sc Top}} & \multicolumn{5}{c|}{\textsc{Recover}}\\
			\hline
			$|I_k|$ & 2 & 4 & 6 & 8 & 10 & 2 & 4 & 6 & 8 & 10\\
			\hline
			\emph{sensitivity} & $100\%$&$100\%$ & $100\%$& $100\%$& $100\%$&$50\%$ & $25\%$& $16.7\%$& $12.5\%$ & $10\%$\\
			\hline
			\emph{specificity} & $100\%$&$100\%$ & $100\%$& $100\%$& $100\%$& $100\%$&$100\%$ & $100\%$& $100\%$& $100\%$\\
			\hline
			Number of topics & \multicolumn{5}{c|}{$100\%$}& \multicolumn{5}{c|}{N/A}\\
			\hline
		\end{tabular}
	}
\end{table}

\subsection*{Estimation error}

In the benchmark setting, we varied $N$ and $n$ over  $\{500, 1000, 1500, 2000, 2500\}$, $p$ over $\{500, 800, 1000, 1200, 1500\}$, $K$ over $\{20, 25, 30, 35, 40\}$ and $|I_k|$ over $\{2, 4, 6, 8, 10\}$, one at a time. For each case, the averaged \emph{overall estimation error} $\|\wh A - AP\|_1/K$ and \emph{topic-wise estimation error} $\|\wh A-AP\|_{1,\infty}$ over $50$ generated datasets for each dimensional setting were recorded. We used a simple linear program to find the best permutation matrix $P$ which aligns $\wh A$ with $A$. Since the two measures had   similar patterns for all settings, we only present \emph{overall estimation error} in Figure \ref{fig_ell1_error}, which can be summarized as follows:

\begin{itemize}
	\item[-]  The estimation error of all four algorithms decreases as $n$ or $N$ increases, while it increases as $p$ or $K$ increases. This confirms our  theoretical findings and   indicates that  $A$ is harder to estimate when not only $p$, but $K$ as well, is allowed to grow. 
	\item[-] In all settings, {\sc Top10} has the smallest \emph{estimation error}. Meanwhile, {\sc Top1} has better performance than \textsc{Recover-L2} and \textsc{Recover-KL} except for $N = 500$ and $|I_k| = 2$. The difference between {\sc Top10} and {\sc Top1} decreases as  the length $N$ of each sampled document increases. This is to be expected since the larger the $N$, the better each column of $X$ approximates the corresponding column of $\M$, which lessens the benefit of selecting different representative sets $\wh L$ of anchor words. 
	\item[-] \textsc{Recover-KL} is more sensitive to the specification of  $K$ and $|I_k|$ than the other approaches. Its performance increasingly worsens  compared to the other procedures for increasing values of  $K$. On the other hand, when the sizes $|I_k|$ are small, it performs almost as well as {\sc Top10}. However, its performance does not improve as much as the performances of the other algorithms in the presence of more anchor words.  
\end{itemize}

\begin{figure}[ht]
	\centering
	\begin{tabular}{ccc}
		\includegraphics[width=0.33\textwidth
		]{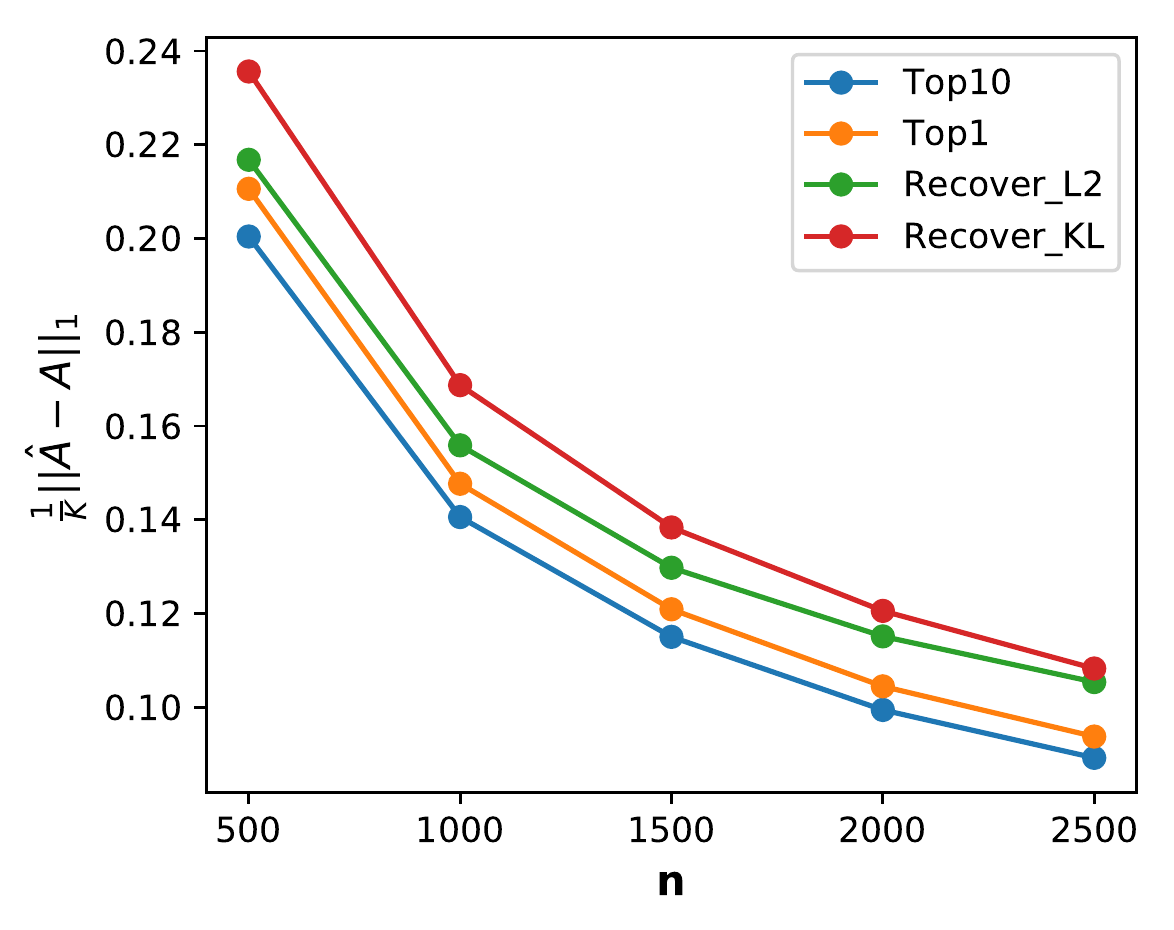}& \hskip-15pt
		\includegraphics[width=0.33\textwidth
		]{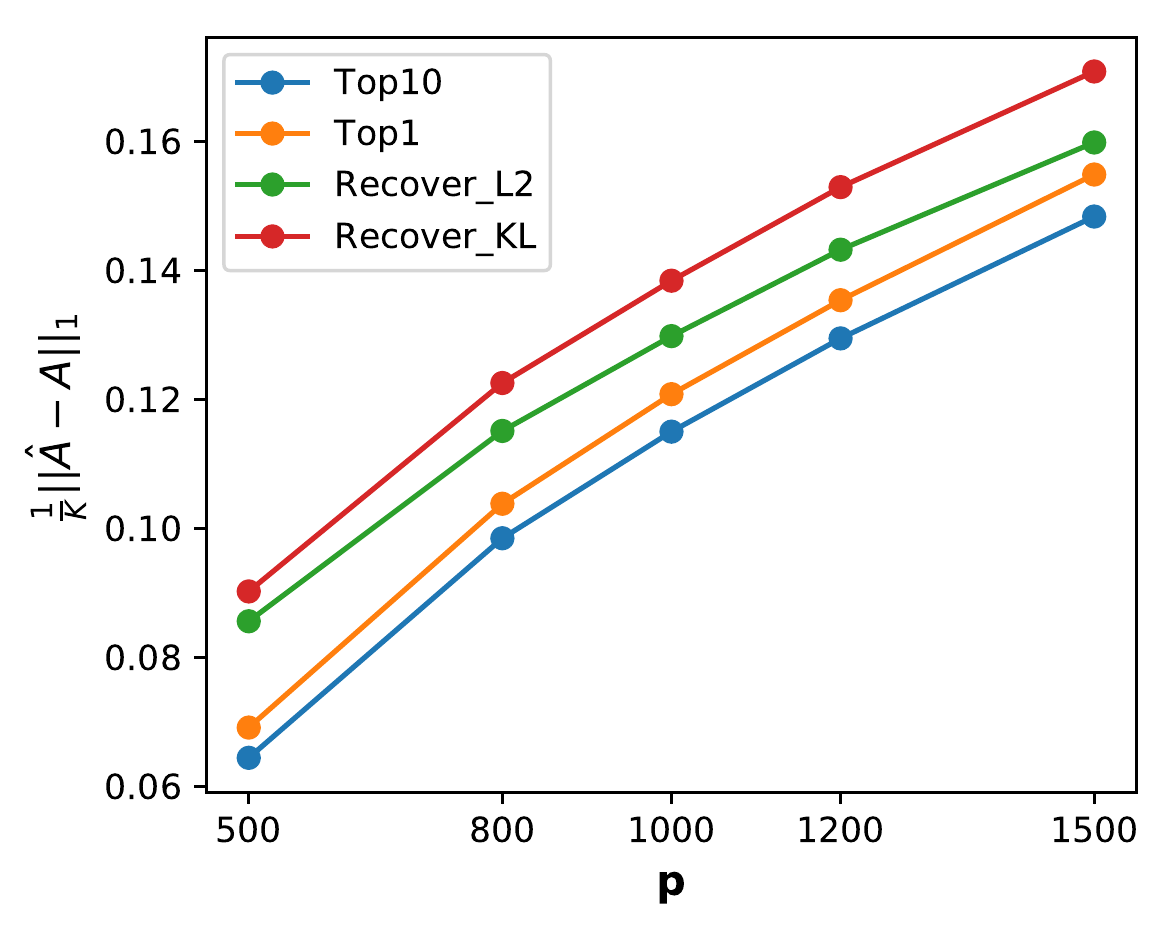}& \hskip-15pt
		\includegraphics[width=0.33\textwidth
		]{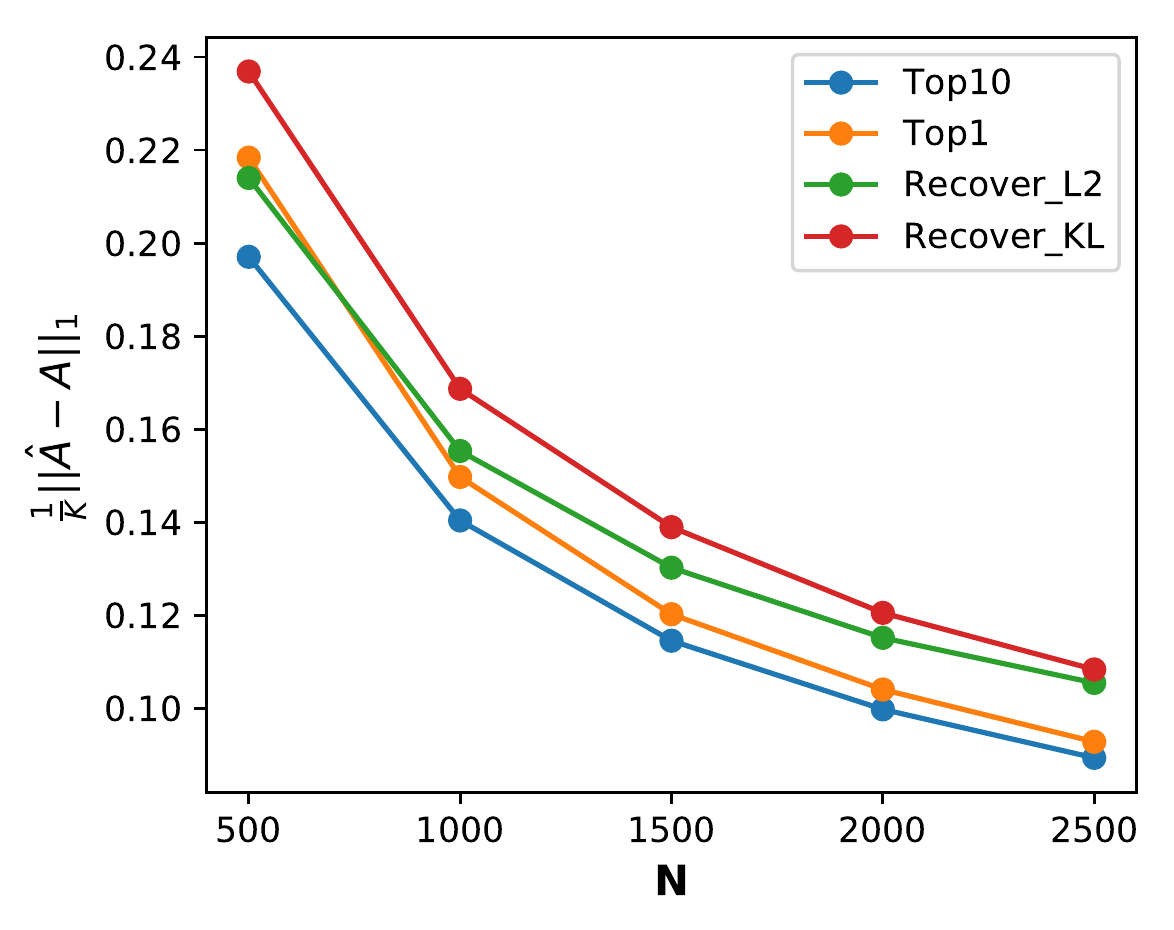}
	\end{tabular}
	\begin{tabular}{cc}
		\includegraphics[width=0.32\textwidth
		]{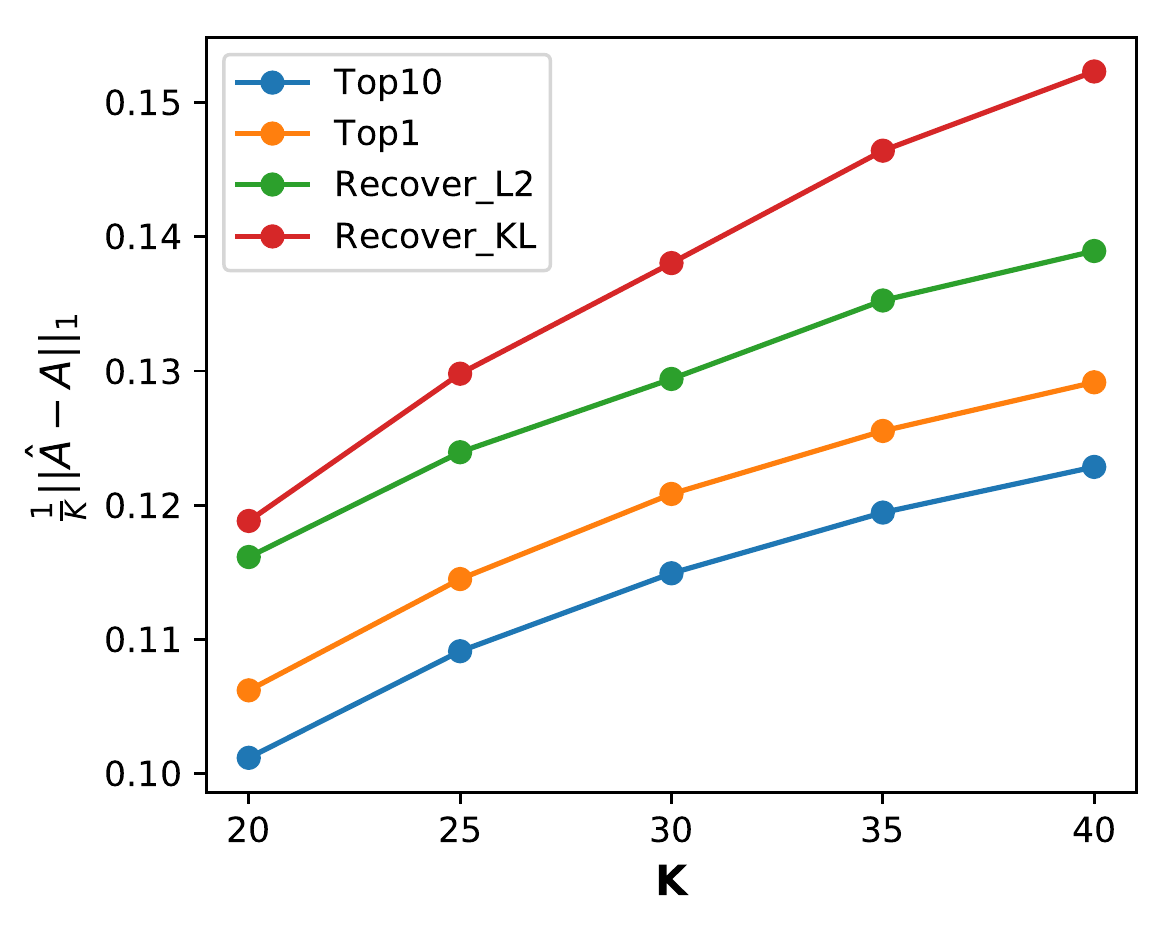}&
		\includegraphics[width=0.32\textwidth
		]{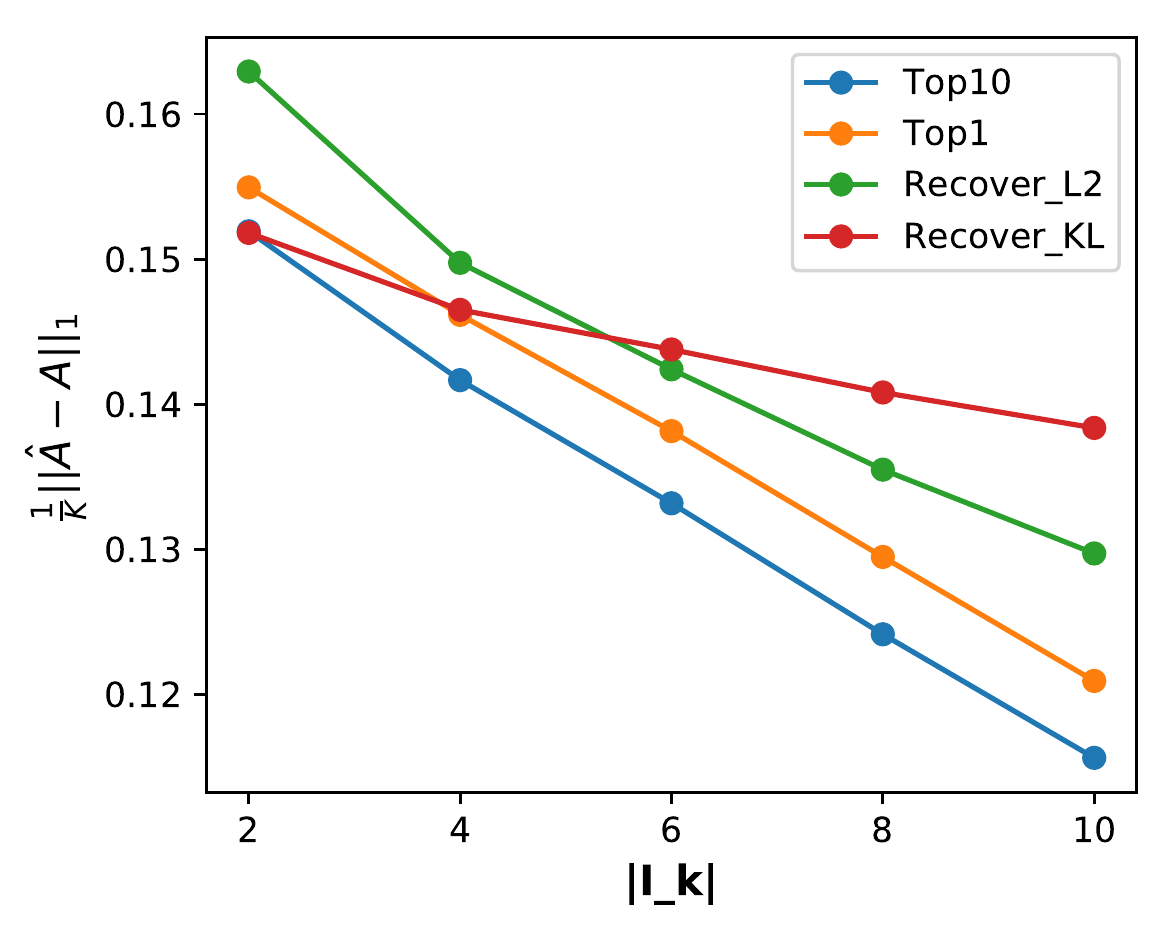} 
	\end{tabular}
	\caption{Plots of averaged \emph{overall estimation error} for varying parameter one at a time.}
	\label{fig_ell1_error}
	\vspace{-5mm}
\end{figure}

\subsection*{Running time}
The running time of all four algorithms is shown in Figure \ref{fig_time_elapse}. As expected, {\sc Top1}   dominates in terms of  computational efficiency. Its computational cost only slightly increases in $p$ or $K$.  Meanwhile, the running times of {\sc Top10} is better than \textsc{Recover-L2} in most of the settings and becomes comparable to it when $K$ is large or $p$ is small. \textsc{Recover-KL} is overall much more computationally demanding than the others. 
We see that  {\sc Top1} and {\sc Top10} are nearly independent of $n$, the number of documents, and $N$,  the document length, as these parameters   only appear in the computations of the matrix $\wh R$ and the tuning parameters $\wh\delta_{ij}$ and $\wh \eta_{ij}$. 
More importantly, as the dictionary size $p$ increases, the two  {\sc Recover} algorithms become much more computationally expensive than {\sc Top}. This difference stems from the fact that our procedure of estimating $A$   is almost independent of $p$ computationally. {\sc Top} solves $K$ linear programs in $K$ dimensional space, while  {\sc Recover} must solve $p$ convex optimization problems over in $K$ dimensional spaces.

We emphasize again that our {\sc Top} procedure  accurately estimates $K$ in the reported times, whereas we provide the two {\sc Recover} versions with the true values of $K$.  In practice, one needs to resort to various cross-validation schemes to select a value of $K$ for the {\sc Recover} algorithms, see 
\cite{arora2013practical}. This would dramatically increase the actual running time for these procedures. 

\begin{figure}[ht]
	\centering
	\begin{tabular}{ccc}
		\includegraphics[width=0.32\textwidth
		]{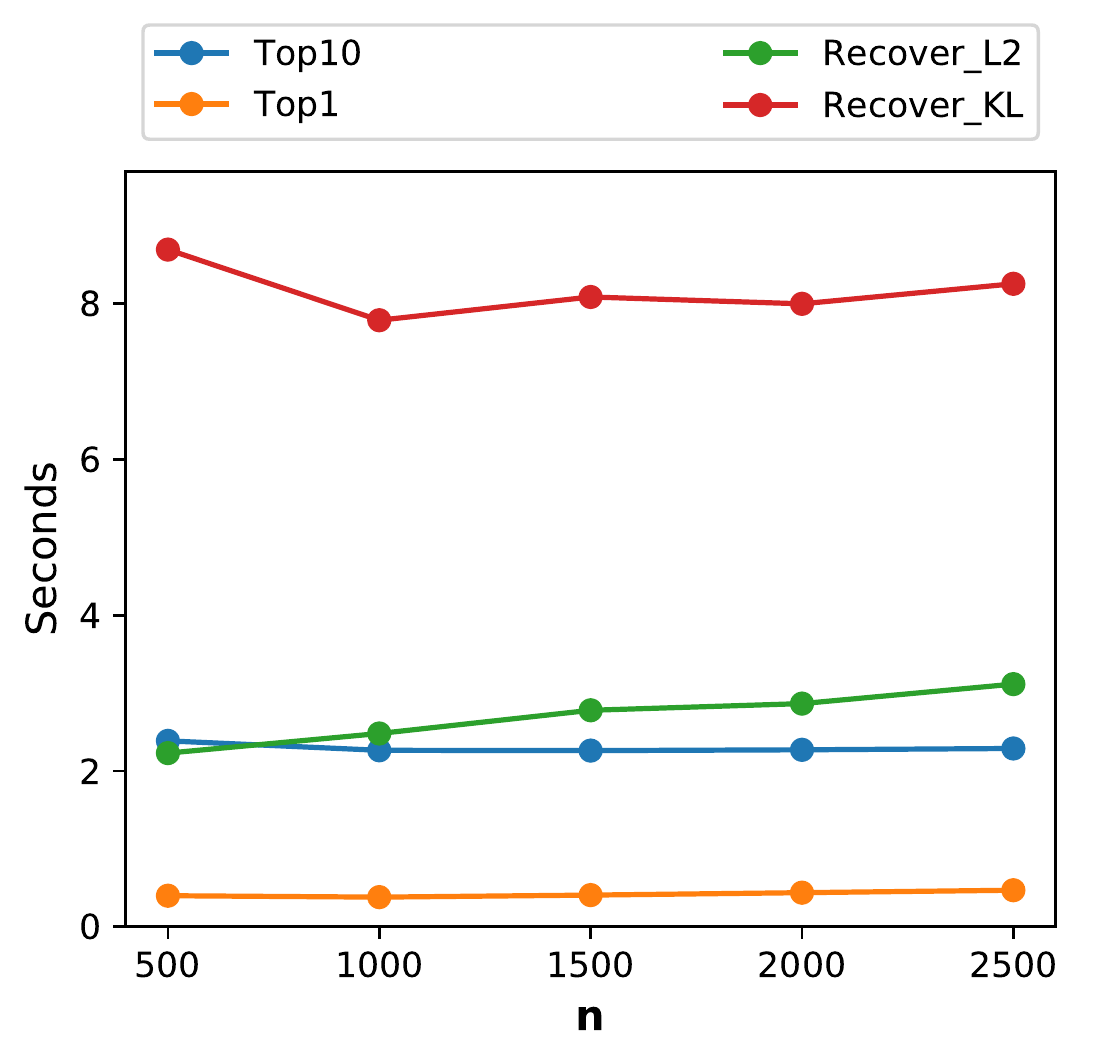} & 
		\includegraphics[width=0.32\textwidth
		]{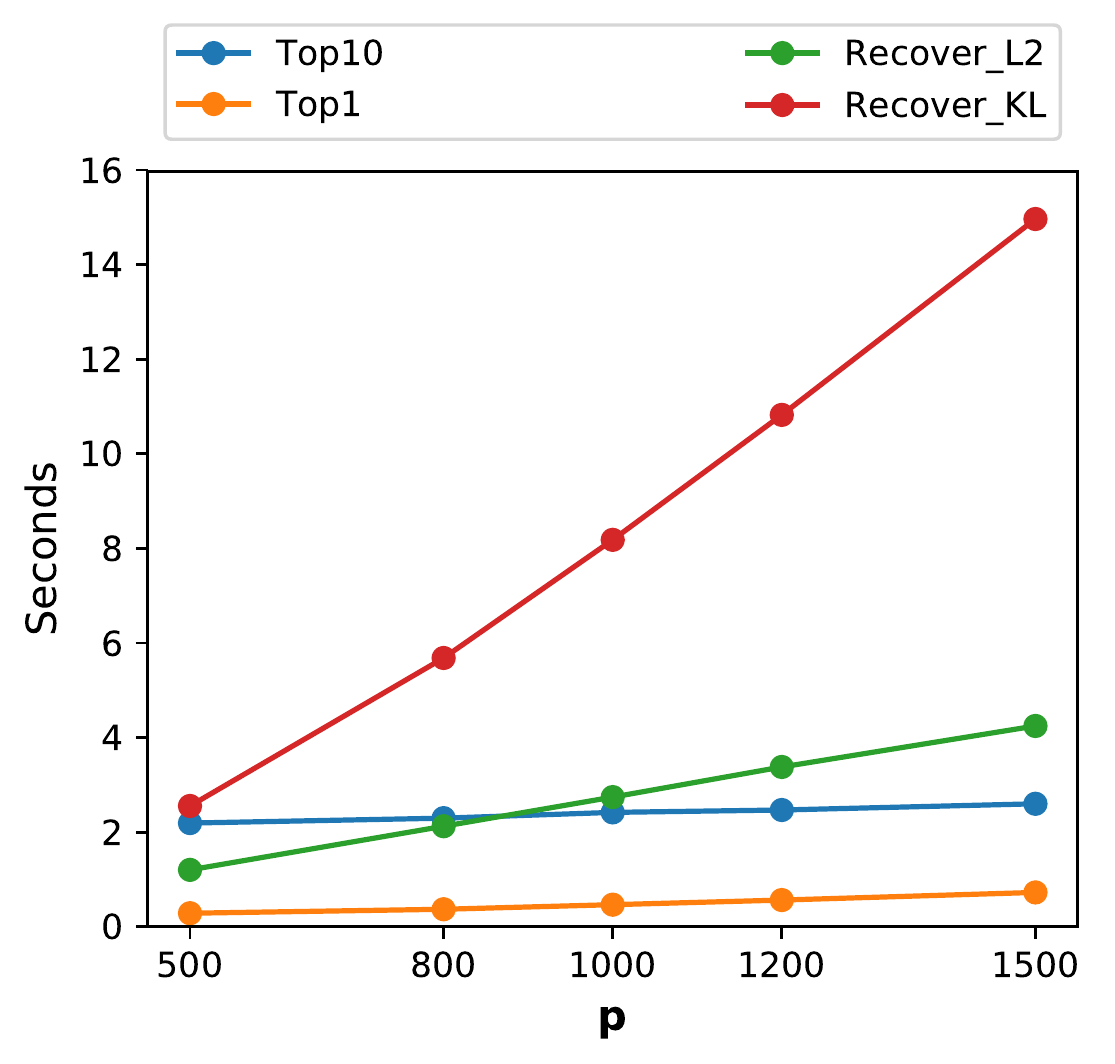}&
		\includegraphics[width=0.32\textwidth
		]{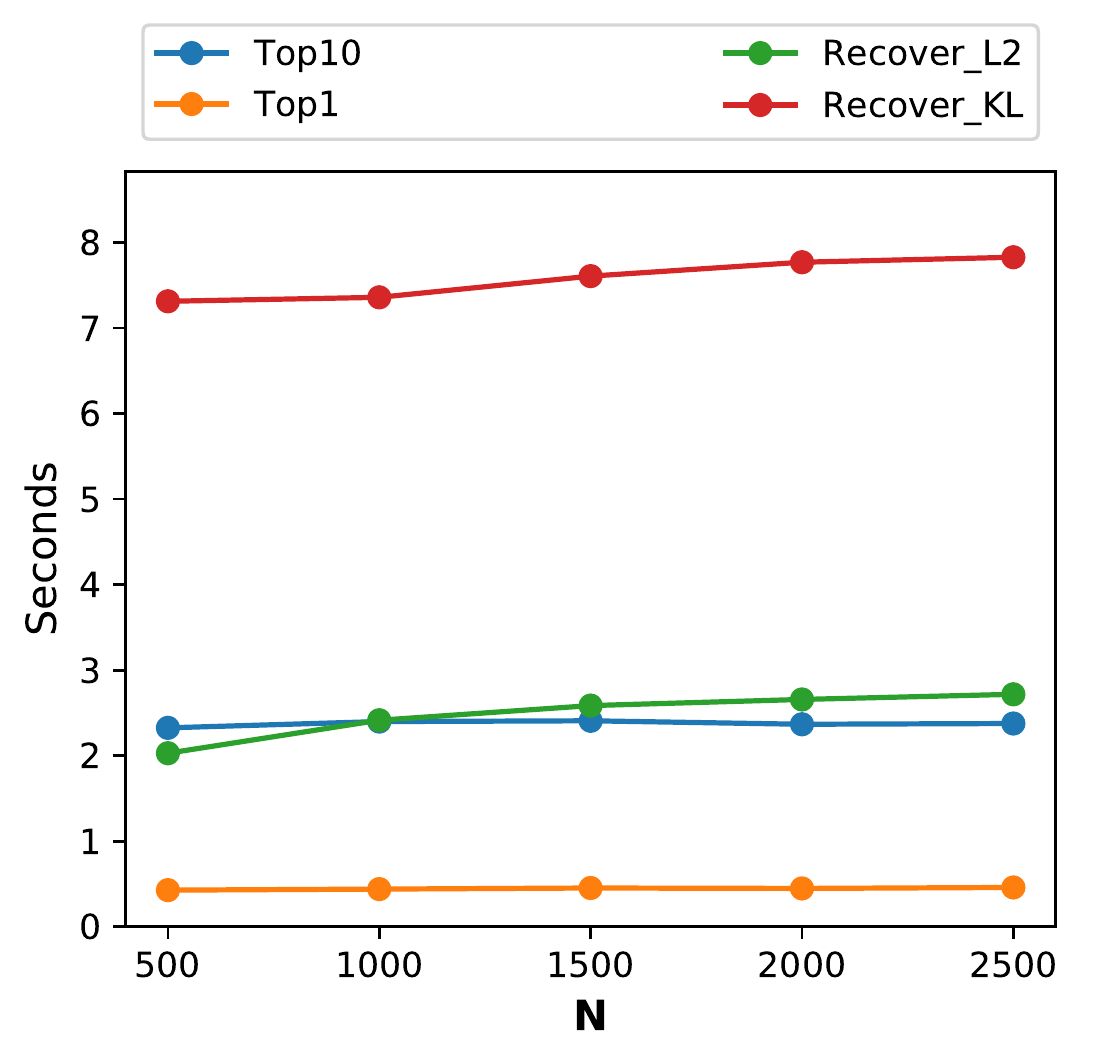}
	\end{tabular}
	\begin{tabular}{cc}
		\includegraphics[width=0.32\textwidth
		]{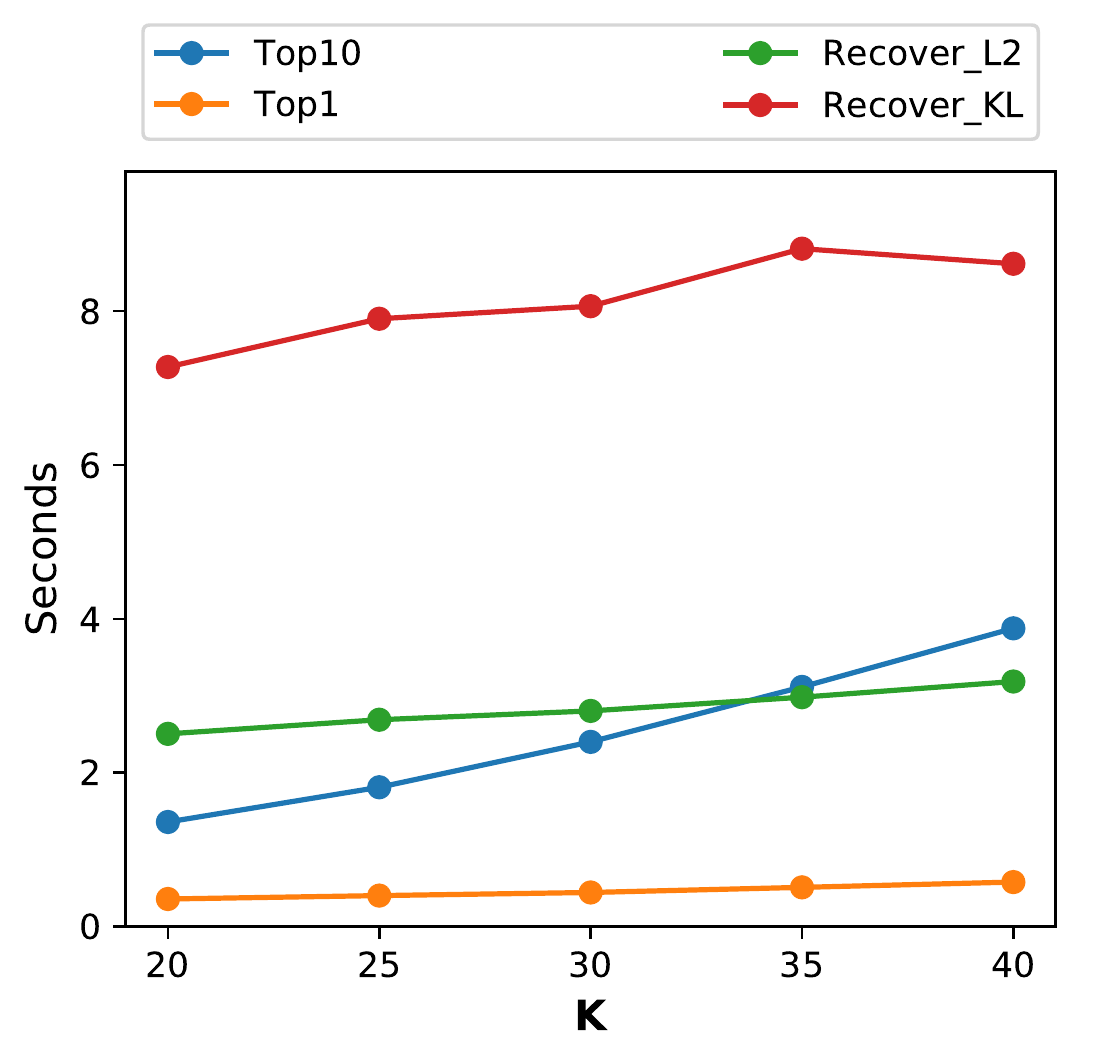}&
		\includegraphics[width=0.32\textwidth
		]{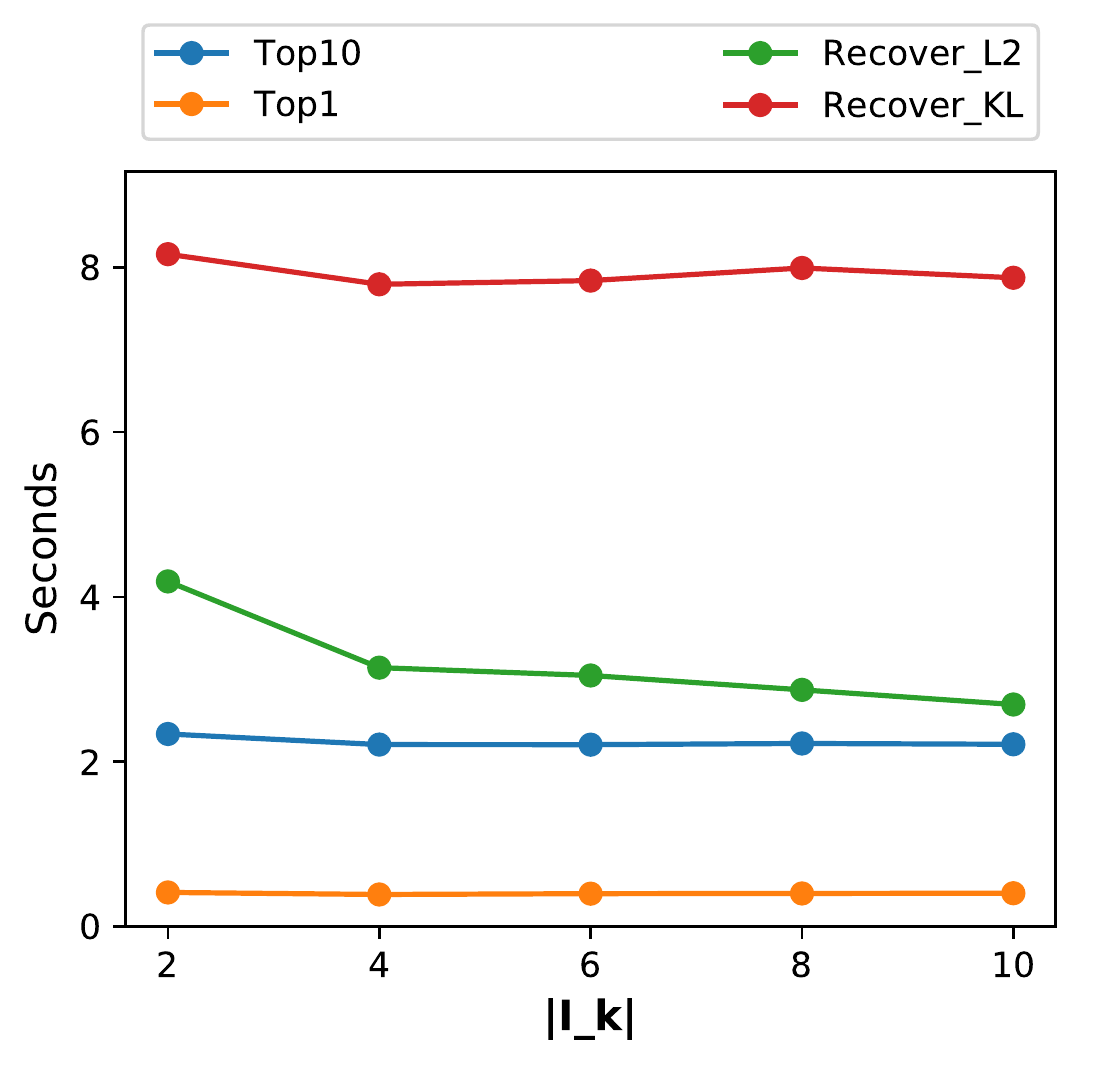} 
	\end{tabular}
	\caption{Plots of running time for varying parameter one at a time.}
	\label{fig_time_elapse}
\end{figure}

\subsection*{Semi-synthetic data from NIPs corpus}
In this section, we compare our algorithm with existing competitors on semi-synthetic data, generated as follows.

We begin with one real-world dataset\footnote{More comparison based on the New York Times dataset is relegated to the supplement.}, a corpus of NIPs articles \citep{ucldata} to benchmark our algorithm and compare {\sc Top1} with \textsc{LDA} \citep{BleiLDA}, \textsc{Recover-L2} and \textsc{Recover-KL}. We use the code of LDA from \cite{lda} implemented via the fast collapsed Gibbs sampling with the default 1000 iterations. To preprocess the data, following \cite{arora2013practical}, we removed common stopping words and rare words occurring in less than 150 documents, and cut off the documents with less than $150$ words. The resultant dataset has $n = 1480$ documents with dictionary size $ p = 1253$ and mean document length $858$.

To generate semi-synthetic data,  we first apply {\sc Top} to this real data set, in order  to obtain the estimated word-topic matrix $A$, which we then use as the ground truth in our simulation experiments, performed as follows.\footnote{ \cite{arora2013practical} uses the posterior estimate of $A$ from LDA with $K=100$. Since we do not have prior information of $K$, we instead use our {\sc Top} to estimate it. Moreover, the posterior from LDA does not satisfy the anchor word assumptions and to evaluate the effect of anchor words, one has to manually add additional anchor words \citep{arora2013practical}. In contrast, the estimated $A$ from {\sc Top} automatically gives anchor words.}   For each document  $i\in [n]$, we sample $W_i$ from a specific distribution (see below) and we sample $X_i$ from $\text{Multinomial}_p(N_i, AW_i)$. The estimated $A$ from {\sc Top} (with $C_1 = 4.5$ chosen via cross-validation and $C_0 = 0.01$) contains 178 anchor words and $120$ topics. We consider three distributions of $W$, chosen  as in  \citep{arora2013practical}:
\begin{enumerate}
	\item[(a)] symmetric Dirichlet 
	distribution with parameter $0.03$;
	\item[(b)] logistic-normal distribution with block diagonal covariance matrix and $\rho = 0.02$;
	\item[(c)] logistic-normal distribution with block diagonal covariance matrix and $\rho = 0.2$. 
\end{enumerate}
Cases (b) and (c)  are designed to investigate how the correlation among topics affects the estimation error. 
To construct the block diagonal covariance structure,  we divide the 120 topics into 10 groups. For each group, the off-diagonal elements of the covariance matrix of topics is set to $\rho$ while the diagonal entries are set to $1$. The parameter $\rho = \{0.02, 0.2\}$ reflects the magnitude of correlation among topics. 

The number of documents $n$ is varied as  $\{2000, 3000, 4000, 5000, 6000\}$ and the document length is set to $N_i = 850$ for $1\le i\le n$. In each setting, we repeat generating 20 datasets and report the averaged \emph{overall estimation error} $\|\wh A - AP\|_1/K$ and \emph{topic-wise estimation error} $\|\wh A-AP\|_{1,\infty}$ of different algorithms  in Figure \ref{fig_nips}. The running time of each   algorithm is reported in Table \ref{table_time_nips}. 

Overall, {\sc LDA} is outperformed by the other three methods, though its performance might be improved by increasing the number of iterations. {\sc Top}, {\sc Recover-KL} and {\sc Recover-L2} are comparable when columns of $W$ are sampled from a symmetric Dirichlet with parameter 0.03, whereas {\sc Top} has better performance when the correlation among topics increases. Moreover, {\sc Top} has the best control of {\it topic-wise estimation error} as expected, while the comparison between {\sc Recover-KL} and {\sc Recover-L2} depends on the error metric. From the  running-time perspective, {\sc Top} runs significantly faster than the other three methods.

Finally, we emphasize that we provide {\sc LDA} and the two {\sc Recover} algorithms with the true $K$, whereas {\sc Top} estimates it.

\begin{figure}[ht]
	\centering
	\includegraphics[width = \textwidth]{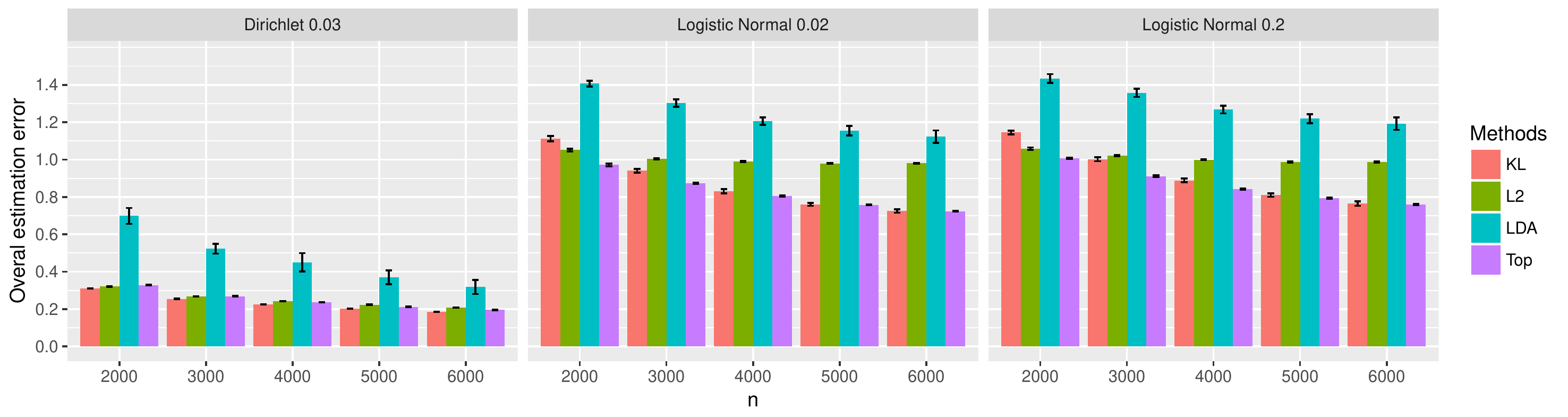}
	\includegraphics[width = \textwidth]{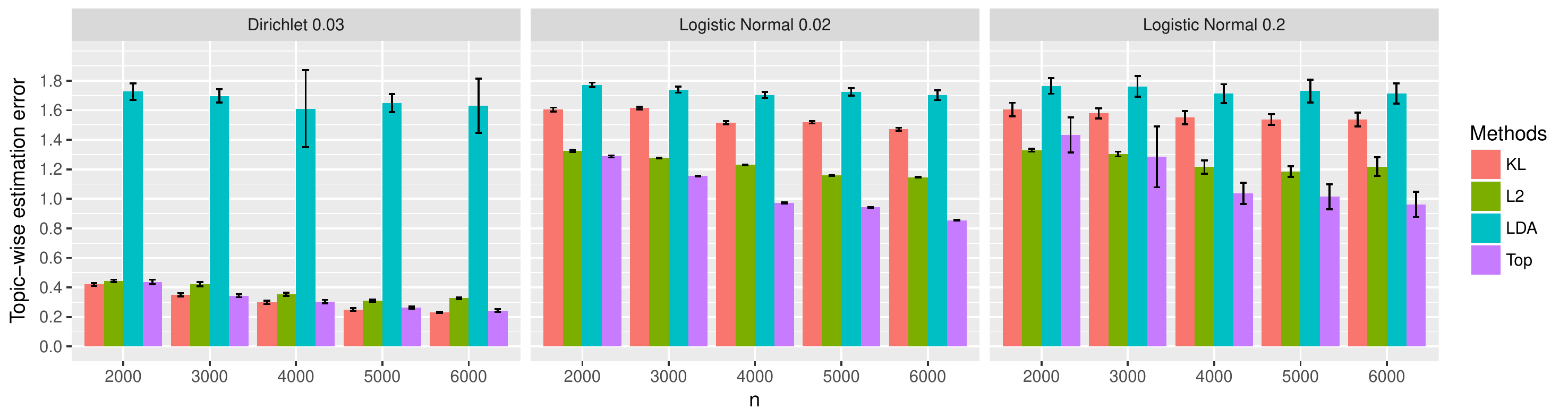}
	\caption{Plots of averaged \emph{overall estimation error} and \emph{topic-wise estimation error} of {\sc Top}, {\sc Recover-L2} (L2), {\sc Recover-KL} (KL) and LDA. TOP estimates $K$, the other methods use the true $K$ as input. The bars denote one standard deviation.}
	\label{fig_nips}
\end{figure}

\begin{table}[ht]
	\centering
	\caption{Running time (seconds) of different algorithms}
	\label{table_time_nips}
	\begin{tabular}{ccccc}
		\hline
		& {\sc Top} & {\sc Recover-L2} & {\sc Recover-KL} & {\sc LDA} \\ 
		\hline
		$n = 2000$ & 21.4 & 428.2 & 2404.5 & 3052.3 \\ 
		$n = 3000$ & 22.3 & 348.2 & 1561.8 & 4649.5 \\ 
		$n = 4000$ & 25.3 & 353.5 & 1764.8 & 6051.1 \\ 
		$n = 5000$ & 28.5 & 349.0 & 1800.4 & 7113.0 \\ 
		$n = 6000$ & 29.5 & 346.6 & 1848.1 & 7318.4 \\ 
		\hline
	\end{tabular}
\end{table}

\section*{Acknowledgements} 
Bunea and Wegkamp are supported in part by NSF grant DMS-1712709. 
We thank  the Editor,  Associate Editor and three referees for their constructive remarks.
	\setlength{\bibsep}{0.85pt}{
    	\bibliographystyle{ims}
		\bibliography{ref}
	}

    \newpage
	\appendix
	\section*{Supplement}\label{supp}
From the topic model specifications, the matrices $\M$, $A$ and $W$ are all scaled as 
\begin{equation}\label{orig_sum_to_one}
\sum_{j =1}^p\M_{ji} = 1,\quad\sum_{j =1}^pA_{jk} = 1,\quad 
\sum_{k=1}^KW_{ki} = 1	
\end{equation}
for any $1\le j\le p$, $1\le i\le n$ and $1\le k\le K$. 
In order to adjust  their scales properly, we denote 
\begin{equation}\label{def_uag}
m_{j} = p\max_{1\le i\le n}\M_{ji},~~ \u_j = {p\over n} \sum_{i=1}^n\M_{ji},~~ \alpha_j = p\max_{1\le k\le K} A_{jk}, ~~ \g_k = {K\over n}\sum_{i=1}^nW_{ki},
\end{equation}
so that 
\begin{equation}\label{sum_to_1}
\sum_{k =1}^K\g_k =K, \qquad \sum_{j =1}^p\u_j = p.
\end{equation}
We further denote
$m_{\min} = \min_{1\le j\le p}m_j$ and  $\u_{\min} = \min_{1\le j\le p} \u_j.$

\section{Proofs of Section \ref{sec_iden}}
\begin{proof}[Proof of Proposition \ref{prop_orth_AW}]
	Recall that $Y_i := N_iX_i \sim \text{Multinomial}_p(N_i; \M_i)$ for any $i\in [n]$. The joint log-likelihood of $(Y_1, \ldots, Y_n)$ is 
	\begin{align*}
	\ell(Y_1,\ldots, Y_n) &= \sum_{i=1}^n\log(N_i!) - \sum_{i =1}^n\sum_{j =1}^p\log(Y_{ji})+\sum_{i =1}^n\sum_{j =1}^pY_{ji}\log \M_{ji}\\
	&=\sum_{i=1}^n\log(N_i!) - \sum_{i =1}^n\sum_{j =1}^p\log(Y_{ji})+\sum_{i =1}^n\sum_{j =1}^pY_{ji}\log \left( \sum_{k = 1}^KA_{jk}W_{ki}\right).
	\end{align*}
	Fix any $j\in [p]$, $k\in [K]$ and $i\in [n]$. It follows that
	\[	
	{\partial \ell(Y_1,\ldots, Y_n) \over \partial A_{jk}} = 
	\left\{
	\begin{array}{ll}
	\sum_{i =1}^nY_{ji}W_{ki}\Big/\left(\sum_{t = 1}^K A_{jt}W_{ti}\right), & \text{if }A_{jk}\ne 0, W_{ki}\ne 0\\
	0, & \text{otherwise}
	\end{array}\right.
	\]
	from which we further deduce
	\begin{align*}
	&{\partial^2 \ell(Y_1,\ldots, Y_n) \over \partial A_{jk}\partial W_{ki}}\\
	&= 
	\left\{
	\begin{array}{ll}
	\sum_{i =1}^nY_{ji}\left(\sum_{t=1}^KA_{jt}W_{ti} - A_{jk}W_{ki}\right)\Big/\left(\sum_{t = 1}^K A_{jt}W_{ti}\right)^2, & \text{if }A_{jk}\ne 0, W_{ki}\ne 0\\
	0, & \text{otherwise}
	\end{array}\right..
	\end{align*}
	Since $\EE[Y_{ji}] = N_i\M_{ji}$, taking expectation yields 
	\begin{align}\label{eq_llk_1}\nonumber
	&\EE\left[-{\partial^2 \ell(Y_1,\ldots, Y_n) \over \partial A_{jk}\partial W_{ki}}\right]\\
	&=
	\left\{
	\begin{array}{ll}
	\sum_{i =1}^nN_i\left(\sum_{t\ne k}A_{jt}W_{ti}\right)\Big/\left(\sum_{t = 1}^K A_{jt}W_{ti}\right), & \text{if }A_{jk}\ne 0, W_{ki}\ne 0\\
	0, & \text{otherwise}
	\end{array}\right..
	\end{align}
	Similarly, for this $j$, $k$ and $i$ but with any $k'\ne k$, we have 
	\begin{align*}
	&{\partial^2 \ell(Y_1,\ldots, Y_n) \over \partial A_{jk}\partial W_{k'i}}\\
	&= 
	\left\{
	\begin{array}{ll}
	-\sum_{i =1}^nY_{ji} A_{jk'}W_{ki}\Big/ \left(\sum_{t = 1}^K A_{jt}W_{ti}\right)^2, & \text{if }A_{jk}\ne 0, A_{jk'}\ne 0, W_{k'i}\ne 0, W_{ki}\ne 0\\
	0, & \text{otherwise}
	\end{array}\right.
	\end{align*}
	and 
	\begin{align}\label{eq_llk_2}\nonumber
	&\EE\left[-{\partial^2 \ell(Y_1,\ldots, Y_n) \over \partial A_{jk}\partial W_{k'i}}\right]\\&=
	\left\{
	\begin{array}{ll}
	\sum_{i =1}^n N_iA_{jk'}W_{ki}\Big / \sum_{t = 1}^K A_{jt}W_{ti}, & \text{if }A_{jk}\ne 0, A_{jk'}\ne 0, W_{k'i}\ne 0, W_{ki}\ne 0\\
	0, & \text{otherwise}
	\end{array}\right..
	\end{align}
	From (\ref{eq_llk_1}) and (\ref{eq_llk_2}), it is easy to see that condition (\ref{cond_orth}) implies
	\begin{equation}\label{eq_crit_orth}
	\EE\left[-{\partial^2 \ell(Y_1,\ldots, Y_n) \over \partial A_{jk}\partial W_{k'i}}\right] = 0
	\end{equation}
	for any $j\in[p]$, $k, k'\in [K]$ and $i\in [n]$. This proves the sufficiency. To show the necessity, we use contradiction. If (\ref{eq_crit_orth}) holds for any $j\in[p]$, $k, k'\in [K]$ and $i\in [n]$, suppose there exist at least one $j\in[p]$ and $i\in [n]$ such that 
	$\textrm{supp}(A_{j\cdot}) \cap \textrm{supp}(W_{\cdot i}) = \{k_1, k_2\}$ and $k_1 \ne k_2$. Then, (\ref{eq_llk_2}) implies 
	\[
	\EE\left[-{\partial^2 \ell(Y_1,\ldots, Y_n) \over \partial A_{jk_1}\partial W_{k_2i}}\right] \ge {N_iA_{jk_1}W_{k_2i}\over A_{jk_1}W_{k_1i} + A_{jk_2}W_{k_2i}} \ne 0.
	\]
	This contradicts (\ref{cond_orth}) and concludes the proof.
\end{proof}
\smallskip
\begin{proof}[Proof of Proposition \ref{prop_anchor}]
	Since the columns of $A$ sum up to 1, and Assumption \ref{ass_sep} holds, then the matrix  $\wt A$ satisfies:
	\begin{equation}\label{spec_A_tilde}
	\wt A_{jk}\ge 0,\quad \bigl\|\wt A_{j\cdot}\bigr\|_1 = 1, \qquad \text{ for each }j = 1,\ldots, p, \text{ and } K = 1,\ldots, K.
	\end{equation}
	Additionally, $\wt A$ has the same sparsity pattern as $A$, and thus $\wt A$ satisfies Assumption \ref{ass_sep}, with the same $I$ and $\I$. We further notice that Assumption \ref{ass_w} is equivalent to 
	\[|\langle \wt W_{i\cdot}, \wt W_{j\cdot}\rangle| ~<~ \|\wt W_{i\cdot}\|^2\wedge \|\wt W_{j\cdot}\|^2,\qquad \text{for all $1\le i< j\le K$,}
	\] 
	which is further equivalent with 	$\nu>0$, defined in (\ref{def_nu2}).
	
	To finish the proof we invoke  Theorem 1 in \cite{LOVE}, slightly adapted to our situation. Specifically, we consider any matrix $R$ defined in (\ref{def_R}) that factorizes as $R = \wt A \wt C \wt A^T$ where $\wt{A}$ satisfies Assumption \ref{ass_sep}
	and $\wt C$ satisfies (\ref{def_nu2}). Note that the quantities $M_i$ and $S_i$ defined in page 9 of \cite{LOVE} are replaced by, respectively, $T_i$ and $S_i$ in (\ref{comp}). 
	We proceed to prove $(a)$ and $(b)$ of Proposition \ref{prop_anchor}.

	\textbf{Proof of $(a)$.}  We first show the sufficiency part.	Consider any  $i\in [p]$  with  $T_i = T_j$ for all $j\in S_i$. 
	Part (a) of Lemma \ref{lem1}, stated after the proof, states that  there exists a $j\in I_a\cap S_i$ for some $a\in[K]$. For  this $j\in I_a$, 
	we have $T_j=\wt C_{aa}$ from part (b) of Lemma \ref{lem1}.
	Invoking our premise  $T_j= T_i$ as $j\in S_i$, we conclude  that $T_i=\wt C_{aa}$, that is, $\max_{k\in p} R_{ik}=\wt C_{aa}$. By   Lemma \ref{lem1a}, stated after the proof, the maximum is achieved for any $i\in I_a$. However, if $i\not\in I_a$, we have that $R_{ik}< \wt C_{aa}$ for all $k\in [p]$. Hence $i\in I_a$ and this concludes the proof of the sufficiency part. 
	
	It remains to  prove
	the necessity part.
	Let $i\in I_a$ for some $a\in[K]$ and $j\in S_i$.  Lemma \ref{lem1} implies that $j\in I_a$ and $T_i=C_{aa}$.
	Since $j\in S_i$, we have $R_{ij}=T_i=\wt C_{aa}$, while  $j\in I_a$ yields $R_{jk} \le \wt C_{aa}$ for all $k\in [p]$,  and $R_jk = \wt C_{aa}$ for $k\in I_a$, as a result of Lemma \ref{lem1a}. Hence, $T_j =\max_{k\in [p]} R_{jk}= \wt C_{aa}= T_i$ for any $j\in S_i$, which proves our claim. 
	
	\textbf{Proof of $(b)$.} The proof of $(b)$ follows by the same arguments for proving part ${\bf (b)}$ of Theorem 1 in \cite{LOVE}. The proof is then complete. \end{proof}

The following two lemmas are  used in the above proof. They are adapted from Lemmas 1 and 2 of the supplement of \cite{LOVE}.	

\begin{lemma}\label{lem1a} For any $a\in[K]$ and $i\in I_{a}$, we have
	\begin{itemize}
		\item[(a)]  $R_{ij} = \wt C_{aa}$ for all $j\in I_a$,
		\item[(b)]  $R_{ij} < \wt C_{aa}$ for all $j\not\in I_a$.
	\end{itemize}
\end{lemma}
\begin{proof}
	The proof follows the same argument for proving Lemma 1 in \cite{LOVE}.
\end{proof}

\begin{lemma}\label{lem1}
	Let $T_i$ and $S_i$ be defined in (\ref{comp}). We have
	\begin{itemize}
		\item[(a)]  $S_i \cap I \ne \emptyset$, for any $i\in[p]$,
		\item[(b)]   $S_i\cup \{i\} = I_a$ and
		$ T_i = \wt C_{aa}$, for any $i\in I_a$ and $a\in[K]$.
	\end{itemize}
\end{lemma}
\begin{proof}
	Lemma \ref{lem1a} implies
	that, for any $i\in I_a$, 
	$T_i= \wt C_{aa}$ and $S_i= I_a$, which proves part (b). Part (a) follows from the result of part (b) and the same arguments of the proof of Lemma 2 in \cite{LOVE} by replacing $M_i$ by $T_i$, $|\Sigma_{ij}|$ by $R_{ij}$, $A$ by $\wt A$ and $C$ by $\wt C$.
\end{proof}

\section{Error bounds of stochastic errors}\label{app_error}
We use this section to present tight bounds on the error terms which are critical to our later estimation rate. We recall that $\eps_{ji} := X_{ji} - \M_{ji}$, for $1\le i\le n$ and $1\le j\le p$ and assume $N_1 =\ldots = N_n = N$ for ease of presentation since similar results for different $N$ can be derived by using the same arguments. The following results, Lemmata \ref{lem_t1} - \ref{lem_t4} control several terms related with $\eps_{ji}$ under the multinomial assumption (\ref{model_multinomial}). We start by stating the well-known Bernstein inequality and Hoeffding inequality for bounded random variables which are used in the sequel.

\begin{lemma}[Bernstein's inequality for bounded random variable]\label{lem_bernstein}
	For independent random variables $Y_1, \ldots, Y_n$ with bounded ranges $[-B, B]$ and zero means, 
	\[
	\PP\left\{ {1\over n}\left|\sum_{i =1}^nY_i \right| > x \right\} \le 2\exp\left(-{n^2x^2 /2 \over v + nBx/3} \right),\qquad \text{for any }x\ge 0,
	\]
	where $v \ge var(Y_1 + \ldots + Y_n)$.
\end{lemma}

\begin{lemma}[Hoeffding's inequality] \label{hoeff}
	Let $Y_1,\ldots,Y_n$ be independent random variables with $\EE[Y_i]=0$ and bounded by $[a_i, b_i]$: 
	For any $t\ge0$, we have
	\[ \PP\left\{ \left| \sum_{i=1}^n Y_i \right| > t \right\} \le 2\exp\left( -\frac{2t^2}{\sum_{i =1}^n(b_i-a_i)^2}  \right).\]
\end{lemma}
\begin{lemma}\label{lem_t1}
	Assume $\u_{\min}/p\ge 2\log M/(3N)$. With probability $1-2M^{-1}$,
	\[
	{1\over n}\left|\sum_{i =1}^n \eps_{ji} \right| \le 2\sqrt{\u_j\log M\over npN}\left(1+\sqrt{6}n^{-{1\over 2}}\right),	\quad \text{uniformly in $1\le j\le p$.}
	\]
\end{lemma}
\begin{proof}
	Fix $1\le j\le p$. From model (\ref{model}), we know $NX_{ji}\sim$ Binomial$(N; \M_{ji})$. We  express the binomial random variable as a sum of i.i.d. Bernoulli random variables:
	$$
	\eps_{ji} = X_{ji} - \M_{ji} = {1\over N}\sum_{k=1}^{N}\left(B_{ik}^j - \M_{ji}\right) :={1\over N}\sum_{\ell=1}^{N}Z^{j}_{ik}$$ with $B_{i\ell}^j \sim$ Bernoulli$(\M_{ji})$, such that
	$
	N\sum_{i =1}^n\eps_{ji} = \sum_{i =1}^n\sum_{\ell = 1}^{N}Z^{j}_{i\ell}.$
	Note that $|Z_{ik}^j| \le 1$, $\EE[ Z_{ik}^j ]= 0$ and $\EE[(Z^j_{ik})^2] = \M_{ji}(1-\M_{ji})\le \M_{ji}$, for all $i\in [n]$ and $k \in [N]$.  An application of Bernstein's inequality, see Lemma \ref{lem_bernstein}, to $Z_{i\ell}^j$ with $v = N\sum_{i =1}^n\M_{ji}=Nn \mu_j/p$ and $B = 1$ gives
	\[
	\PP\left\{ {1\over n}\left|\sum_{i =1}^n \eps_{ji} \right| > t \right\} \le 2\exp\left(-{ n^2N^2t^2/2 \over  N\sum_{i =1}^n\M_{ji}  + nNt/3}\right),\quad \text{for any $t>0$.}
	\]
	This implies, for all $t>0$
	\[
	\PP\left\{{1\over n}\left|\sum_{i =1}^n \eps_{ji} \right| > {\sqrt{\u_jt\over npN}} + {t\over nN}\right\} \le 2e^{-t/2}.
	\]
	Choosing $t = 4\log M$ and recalling that $M=n\vee N\vee p\ge p$, we find by the union bound
	\begin{equation}\label{eq_lem_t1}
	\sum_{j=1}^{p}\PP\left\{ {1\over n}\left|\sum_{i =1}^n \eps_{ji} \right| > 2{\sqrt{\u_j\log M\over npN}} + {4\log M \over nN} \right\} \le 2p M^{-2} \le 2M^{-1} .
	\end{equation}
	Using $\u_{\min}/p \ge 2\log M/(3N)$ 
	concludes the proof.
\end{proof}

\begin{remark}\label{rem_t1}
	By inspection of the proof of Lemma \ref{lem_t1}, if instead 
	of the bound $\EE[(Z_{ik}^j)^2] \le \Pi_{ji}$, we had used the overall bound $\EE[(Z_{ik}^j)^2] \le 1$, and application of Bernstein's inequality would yield, 
	\[
	{1\over n}\left|\sum_{i =1}^n \eps_{ji} \right| \le c{\sqrt{\log M\over nN}},	\quad \text{uniformly in $1\le j\le p$.}
	\] 
	Summing  over $1\le j\le p$ in the above display would further give  
	\[
	{1\over n}\sum_{j =1}^p\left|\sum_{i =1}^n \eps_{ji} \right| \le c\cdot p{\sqrt{\log M\over nN}},
	\]
	and the right hand side would be slower by an important  $\sqrt{p}$ factor that what we would obtain by summing the bound in Lemma \ref{lem_t1}  over $1\le j\le p$, since 
	\[
	2\sum_{j =1}^p\sqrt{\u_j\log M\over npN} \le 2\sqrt{\log M \over nN}\sqrt{\sum_{j =1}^p\u_j} = 2\sqrt{p\log M \over nN}.
	\]
	In this last display, we used the Cauchy-Schwarz inequality in the first inequality,  and the constraint (\ref{sum_to_1}) in the last equality. The  bound of Lemma \ref{lem_t1} is an important intermediate step for deriving the final bounds of Theorem \ref{thm_rate_Ahat}, and the simple 
	calculations presented above that the constraints of unit column sums induced by model (\ref{model}) permit a more refined control of the stochastic errors than those previously considered. A larger impact of this refinement on the final rate of convergence is illustrated in Remark \ref{rem_t3}, following the proof of Lemma \ref{lem_t4}, in which we control sums of quadratic, dependent terms  $\eps_{ji}\eps_{\ell i}$.
\end{remark}

\begin{lemma}\label{lem_t2}
	With probability $1-2M^{-1}$,
	\[
	{1\over n}\left|\sum_{i =1}^n \M_{\ell i}\eps_{ji} \right| \le {\sqrt{6m_\ell\Theta_{j\ell}\log M\over npN}}+{2m_\ell \log M \over  npN},	\quad \text{uniformly in  $1\le j, \ell \le p$.}
	\]
\end{lemma}
\begin{proof}
	Similar as in the proof of Lemma \ref{lem_t1}, we write 
	$$ \M_{\ell i}\eps_{ji} = {1\over N}\sum_{k=1}^N \M_{\ell i}Z_{ik}^{j}$$ such that $| \M_{\ell i}Z_{ik}^j| \le \M_{\ell i}\le m_{\ell}/p$ by (\ref{def_uag}), $\EE [\M_{\ell i}Z_{ik}^j] = 0$ and $\EE[\M_{\ell i}^2Z_{ik}^2] = \M_{\ell i}^2\M_{ji}(1-\M_{ji})\le m_{\ell}\M_{\ell i}\M_{ji}/p$, for all $i\in [n]$ and $k \in [N]$. Fix $1\le j, \ell\le p$ and recall that $\Theta_{j\ell} = n^{-1} \sum_{i=1}^n\M_{ji}\M_{\ell i}$. Applying Bernstein's inequality to $\M_{\ell i}Z_{ik}^j$ with $v = nNm_\ell \Theta_{j\ell}/p$ and $B = m_\ell / p$ gives
	\[
	\PP\left\{ {1\over n}\left|\sum_{i=1}^n \M_{\ell i}\eps_{ji}\right| \ge t \right\} \le 2\exp\left(-{n^2N^2t^2/2 \over nNm_\ell \Theta_{j\ell}/p+ nNm_{\ell}t/(3p)}\right)
	\]
	which further implies 
	\[
	\PP\left\{ {1\over n}\left|\sum_{i =1}^n\M_{\ell i}\eps_{ji}\right| > \sqrt{m_{\ell}\Theta_{j\ell}t \over  npN} +  {m_\ell t \over  3npN} \right\} \le 2e^{-t/2}.
	\]
	Taking the union bound over $1\le j,\ell \le p$, and choosing $t = 6\log M$, concludes the proof.
\end{proof}

\begin{lemma}\label{lem_t4}
	If $\u_{\min}/p \ge 2\log M / (3N)$, then with probability $1-4M^{-1}$,
	\[
	{1\over n}\left|\sum_{i =1}^n\left(\eps_{ji}\eps_{\ell i} - \EE[\eps_{ji}\eps_{\ell i}]\right) \right| \le 12\sqrt{6}\sqrt{\Theta_{j\ell}+{(\u_j + \u_{\ell})\log M \over pN}}\sqrt{(\log M)^3 \over nN^2} + 4M^{-3},
	\]
	holds, uniformly in  $1\le j,\ell\le p$.
\end{lemma}
\begin{proof}
	For any $1\le j\le p$, recall that $\eps_{j i} = X_{ji} - \M_{ji}$ and 
	\[
	\eps_{j i} = {1\over N}\sum_{k =1}^NZ_{ik}^j
	\]
	where $Z_{ik}^j:= B_{ik}^j - \M_{ji}$ and $B_{ik}^j \sim$ Bernoulli$(\M_{ji})$. By using the arguments of Lemma \ref{lem_t1}, application of Bernstein's inequality to $Z_{ik}^j$ with $v = N\M_{ji}$ and $B = 1$ gives
	\[
	\PP\left\{  \left|\eps_{ji} \right| > t \right\} \le 2\exp\left(-{ Nt^2/2\over  \M_{ji} + t/3}\right),\quad \text{ for any $t>0$,}
	\]
	which yields
	\[
	\left| \eps_{ji} \right| \le {\sqrt{6\M_{ji}\log M\over N}} + {2\log M\over N} := T_{ji}
	\]
	with probability greater than $1-2M^{-3}$. We define 
	$
	Y_{ji} = \eps_{j i}\1_{\S_{ji}}
	$ with
	$
	\S_{ji} := \left\{|\eps_{j i}| \le T_{ji} \right\}
	$
	for each $1\le j\le p$ and  $1\le i\le n$, and
	$
	\S := \cap_{j=1}^p\cap_{i=1}^n\S_{ji}.
	$
	It follows that 
	$
	\PP(\S_{ji}) \ge 1-2M^{-3}
	$ for all $1\le i\le n$ and $1\le j\le p$, so that
	$\PP(\S ) \ge 1- 2M^{-1}$ as  $M := n\vee p\vee N$. 
	On the event $\S$, we have
	\[
	{1\over n}\left|\sum_{i =1}^n\left(\eps_{ji}\eps_{\ell i} - \EE[\eps_{ji}\eps_{\ell i}]\right) \right|\le 
	\underbrace{{1\over n}\left|\sum_{i =1}^n\left(Y_{ji}Y_{\ell i}- \EE[Y_{ji}Y_{\ell i}]\right) \right|}_{T_1}+ \underbrace{{1\over n}\left|\sum_{i =1}^n\left(\EE[\eps_{ji}\eps_{\ell i}]- \EE[Y_{ji}Y_{\ell i}]\right) \right|}_{T_2} 
	\]
	We first study $T_2$. By writing 
	\begin{align*}
	\EE[\eps_{ji}\eps_{\ell i}] &= \EE[Y_{ji}Y_{\ell i}]+\EE\left[Y_{ji}\eps_{\ell i}\1_{\S_{\ell i}^c}\right]+\EE\left[\eps_{ji}\1_{\S_{ji}^c}\eps_{\ell i}\right],
	\end{align*}
	we have
	\begin{align}\nonumber
	T_2 ={1\over n}\left|\sum_{i =1}^n\left(\EE[\eps_{ji}\eps_{\ell i}]- \EE[Y_{ji}Y_{\ell i}]\right) \right| & \le {1\over n}\left|\sum_{i =1}^n\left(\EE\left[Y_{ji}\eps_{\ell i}\1_{\S_{\ell i}^c}\right]+\EE\left[\eps_{ji}\1_{\S_{ji}^c}\eps_{\ell i}\right]\right) \right|\\\label{eq_T2}
	&\le {1\over n}\left|\sum_{i =1}^n\left(\PP(\S_{ji}^c)+\PP(\S_{\ell i}^c)\right) \right|\\ &\le 4M^{-3}\nonumber
	\end{align}
	by using $|Y_{ji}| \le |\eps_{ji}| \le 1$ in the second inequality. 
	
	Next we bound $T_1$. Since $|Y_{ji}| \le T_{ji}$, we know $-2T_{ji}T_{\ell i}\le Y_{ji}Y_{\ell i} - \EE[Y_{ji}Y_{\ell i}]\le 2T_{ji}T_{\ell i}$ for all $1\le i\le n$. Applying the Hoeffding inequality Lemma \ref{hoeff} with $a_i = -2T_{ji}T_{\ell i}$ and $b_i = 2T_{ji}T_{\ell i}$ gives 
	\[
	\PP\left\{ \left|\sum_{i =1}^n\left(Y_{ji}Y_{\ell i}- \EE[Y_{ji}Y_{\ell i}]\right) \right| \ge t \right\} \le 2\exp\left(-{t^2 \over 8\sum_{i =1}^nT_{ji}^2T_{\ell i}^2}\right).
	\] 
	Taking $t = \sqrt{24\sum_{i =1}^nT_{ji}^2T_{\ell i}^2\log M}$ yields
	\begin{equation}\label{eq_T1}
	T_1 = {1\over n}\left|\sum_{i =1}^n\left(Y_{ji}Y_{\ell i}- \EE[Y_{ji}Y_{\ell i}]\right) \right| \le  2\sqrt{6}\left({1\over n}\sum_{i =1}^nT_{ji}^2T_{\ell i}^2\cdot {\log M \over n}\right)^{1/2}
	\end{equation}
	with probability greater than $1- 2M^{-3}$. Finally, note that 
	\begin{align}\label{eq_TjiTelli}\nonumber
	{1\over n}\sum_{i =1}^nT_{ji}^2T_{\ell i}^2 & = {1\over n}\sum_{i =1}^n\left(\M_{ji}\M_{\ell i}{36(\log M)^2 \over N^2} + \left(2\log M \over N\right)^4 + (\M_{ji}+\M_{\ell i}){24 (\log M )^3 \over N^3}\right)\\
	&= 36\Theta_{j\ell}\left(\log M \over N\right)^2 + 16\left({\log M \over N}\right)^4 + 24 (\u_j+\u_{\ell }){(\log M)^3\over pN^3}
	\end{align}
	by using (\ref{def_uag}) and $\Theta_{j\ell} =  n^{-1} \M\M^T$ in the second equality. Finally, combining (\ref{eq_T2}) - (\ref{eq_TjiTelli}) and using $\u_{\min}/p \ge 2\log M / (3N)$ conclude the proof.
\end{proof}

\begin{remark}\label{rem_t3}
	We illustrate the improvement of our result over a simple application of Hanson-Wright inequality. Write 
	\[
	4\eps_{ji}\eps_{\ell i} = \left(\eps_{ji}+\eps_{\ell i} \right)^2 - \left(\eps_{ji}-\eps_{\ell i} \right)^2 
	\]
	for each $i\in [n]$. Since $\eps_{ji} = N^{-1}\sum_{k = 1}^NZ_{ik}^j$ and $\|\eps_{ji}\pm \eps_{\ell i}\|_{\phi_2} \le {2/ \sqrt{N}}$, a direct application of the Hanson-Wright inequality to the two terms in the right hand side will give
	\[
	{1\over n}\left|\sum_{i =1}^n\left(\eps_{ji}\eps_{\ell i} - \EE[\eps_{ji}\eps_{\ell i}]\right) \right| \le c\sqrt{\log M \over nN}
	\]
	with high probability. Summing over $1\le j\le p$ and $1\le \ell \le p$ further yields
	\begin{equation}\label{eq_hw}
	\sum_{j,\ell =1}^p{1\over n}\left|\sum_{i =1}^n\left(\eps_{ji}\eps_{\ell i} - \EE[\eps_{ji}\eps_{\ell i}]\right) \right| \le c\cdot p^2\sqrt{\log M \over nN}.
	\end{equation}
	By contrast, summing the first term in Lemma \ref{lem_t4} yields
	\begin{align*}
	\sum_{j,\ell =1}^p{1\over n}\left|\sum_{i =1}^n\left(\eps_{ji}\eps_{\ell i} - \EE[\eps_{ji}\eps_{\ell i}]\right) \right| &\le c\cdot \sqrt{p^2(\log M)^3 \over nN^2}\sqrt{\sum_{1\le j,\ell \le p} \Theta_{j\ell}} + c\cdot \sqrt{p^3(\log M)^4 \over nN^3}\\ 
	&= c\cdot \sqrt{p^2(\log M)^3 \over nN^2}+c\cdot \sqrt{p^3(\log M)^4 \over nN^3}
	\end{align*}
	by using Cauchy-Schwarz in the first inequality and (\ref{orig_sum_to_one}) in the last equality which is $(p\sqrt{N})\wedge (N\sqrt{p})$ faster than the result in (\ref{eq_hw}) after ignoring the logarithmic term.
\end{remark}

\section{Proofs of Section \ref{sec_est_I}}\label{app_sec_A}
Throughout this section, we define the event $\E :=  \E_1 \cap \E_2 \cap \E_3$ by 
\begin{align}\nonumber
\E_1 &= \bigcap_{j=1}^p \left\{{1\over n}\left|\sum_{i =1}^n \eps_{ji} \right| \le 2(1+\sqrt{6/n})\sqrt{\u_j\log M \over npN} \right\},\\\nonumber
\E_2 &= \bigcap_{j,\ell=1}^p\left\{{1\over n}\left|\sum_{i =1}^n \M_{\ell i}\eps_{ji} \right| \le {\sqrt{6m_\ell\Theta_{j\ell}\log M\over npN}} + {2m_{\ell} \log M \over npN}\right\},\\\nonumber
\E_3 &=  \bigcap_{j,\ell=1}^p \left\{
{1\over n}\left|\sum_{i =1}^n\left(\eps_{ji}\eps_{\ell i} - \EE[\eps_{ji}\eps_{\ell i}]\right) \right| \le 12\sqrt{6}\sqrt{\Theta_{j\ell}+{(\u_j + \u_{\ell})\log M \over pN}}\sqrt{(\log M)^3 \over nN^2}+4M^{-3}
\right\}
\end{align}
Recall (\ref{def_uag}) and (\ref{sum_to_1}), if 
\[
\min_{1\le j\le p}{1\over n}\sum_{i=1}^n\M_{ji} \ge {2\log M \over 3N}
\]
holds, we have 
\begin{equation}\label{cond_rate_pN}
{1\over p}\ge {\u_{\min} \over p} = \min_{1\le j\le p}{1\over n}\sum_{i =1}^n \M_{ji} \ge {2\log M \over 3N},
\end{equation}
Therefore, invoking Lemmas \ref{lem_t1} -- \ref{lem_t4} yields $\PP(\E) \ge 1-8M^{-1}$. 

\subsection{Preliminaries}\label{sec_pre}
From model specifications (\ref{orig_sum_to_one}) and (\ref{def_uag}), we first give some useful expressions which are repeatedly invoked later. 
\begin{itemize}
	\item[(a)] For any $j\in [p]$, by using (\ref{def_uag}),
	\begin{equation}\label{eq_mu}
	\u_j = {p\over n}\sum_{i =1}^n\M_{ji} = {p\over n}\sum_{i =1}^n\sum_{k =1}^KA_{jk}W_{ki} = {p\over K }\sum_{k =1}^KA_{jk}\g_k ~~\Rightarrow~~{p\over  K} \sum_{k = 1}^KA_{jk}\cdot \ug\le \u_j\le \alpha_j.
	\end{equation}
	In particular, for any $j\in I_k$ with any $k\in [K]$, 
	\begin{equation}\label{eq_mu_I}
	\u_j = {p\over n}\sum_{i =1}^n\sum_{k =1}^KA_{jk}W_{ki} = {p\over K }A_{jk}\g_k\overset{(\ref{def_uag})}{=}{\alpha_j\g_k \over K}.
	\end{equation}
	\item[(b)] For any $j\in [p]$,
	\begin{equation}\label{eq_m}
	m_j \overset{(\ref{def_uag})}{=} {p}\max_{1\le i\le n}\M_{ji} = {p}\max_{1\le i\le n}\sum_{k =1}^KA_{jk}W_{ki}\le  p\max_{1\le k\le K}A_{jk} \overset{(\ref{def_uag})}{=} \alpha_j ~~ \Rightarrow ~~ \u_j \le m_j \le \alpha_j,
	\end{equation}
	by using $0\le W_{ki}\le 1$ and $\sum_k W_{ki} = 1$ for any $k\in[K]$ and $i\in [n]$. 
\end{itemize}

\subsection{Control of $\wh \Theta - \Theta$ and $\wh R-R$}\label{sec_stat_sigma}

\begin{prop}\label{prop_sigma}
	Under model (\ref{model}), assume (\ref{ass_signal}).
	Let $\wh \Theta$ and $\wh R$ be defined in (\ref{est_Theta}) and (\ref{def_R_hat}), respectively. Then $\wh \Theta$ is an unbiased estimator of $\Theta$. Moreover, with probability greater than $1- 8M^{-1}$,
	\[
	|\wh \Theta_{j\ell} - \Theta_{j\ell}| \le \eta_{j\ell},\qquad |\wh R_{j\ell} - R_{j\ell}| \le \delta_{j\ell}
	\]
	for all $1\le  j,\ell \le p$, where
	\begin{align}\label{def_eta}\nonumber
	\eta_{j\ell} &:=  3\sqrt{6}\left(\sqrt{m_j\over p}+\sqrt{m_\ell\over p} \right)\sqrt{\Theta_{j\ell}\log M\over nN}+{2(m_j+m_\ell)\over p}{ \log M \over nN}\\
	&\qquad + 31(1+\kappa_1)\sqrt{{\u_j + \u_{\ell}\over p}{(\log M)^4 \over nN^3}}+\kappa_2
	\end{align}
	and 
	\begin{equation}\label{delta}
	\delta_{j\ell} := (1+\kappa_1\kappa_3){p^2 \over \u_j \u_{\ell}}\eta_{j\ell}  +\kappa_3{p^2\Theta_{j\ell}\over \u_j\u_{\ell}}\left(\sqrt{p \over \mu_j} + \sqrt{p\over \mu_\ell}\right) \sqrt{\log M \over nN},
	\end{equation}
	with $\kappa_1 = \sqrt{6/n}$, $\kappa_2 = 4/M^{3}$ and 
	\[
	\kappa_3 = {2(1+\kappa_1) \over (1-\kappa_1-\kappa_1^2)^2}.
	\]
\end{prop}

\begin{remark}
	For ease of presentation, we assumed that the document lengths are equal, that is,  $N_i=N$ for all $i\in [n]$. Inspection of the proofs of Lemmas \ref{lem_t1} - \ref{lem_t4} and Proposition \ref{prop_sigma}, we may allow for unequal  document lengths $N_i$ by adjusting the quantities
	$\eta_{j\ell}$ and $\delta_{j\ell}$ with
	\begin{align}\label{est_eta}\nonumber
	\eta_{j\ell} &:=  3\sqrt{6}\left( \sqrt{m_j}+\sqrt{m_\ell} \right)\sqrt{\log M\over np}\left({1\over n}\sum_{i =1}^n{\M_{ji}\M_{\ell i} \over N_i}\right)^{1/2}+{2(m_j+m_\ell) \log M \over np}\left({1\over n}\sum_{i =1}^n{1\over N_i}\right)\\
	&\quad + 31(1+\kappa_1)\sqrt{(\log M)^4 \over n}\left({1\over n}\sum_{i =1}^n{\M_{ji} + \M_{\ell i} \over N_i^3} \right)^{1/2} + \kappa_2\\\label{est_delta}
	\delta_{j\ell} &:=  (1+\kappa_1\kappa_3){p^2\eta_{j\ell} \over \u_j \u_{\ell}} +\kappa_3{p^3\Theta_{j\ell} \over \u_j\u_{\ell }}\sqrt{\log M \over n}\left[\left({1\over n}\sum_{i =1}^n{\M_{ji} \over \u_j^2N_i}\right)^{1/2} + \left({1\over n}\sum_{i =1}^n{\M_{\ell i} \over \u_\ell^2N_i}\right)^{1/2}\right] .
	\end{align}
\end{remark}
\medskip

The quantities $m_j$ and $\u_j$ appearing in the above rates are related with $\Pi$ and can be directly estimated from $X$. Let 	
\begin{equation}\label{est_m_u}
{\wh m_j \over p} = \max_{1\le i\le n}X_{ji}, \qquad {\wh \u_j \over p} = {1\over n}\sum_{i=1}^nX_{ji}.
\end{equation}
The following corollary gives the data dependent bounds of $\wh \Theta-\Theta$ and $\wh R-R$.
\begin{cor}\label{cor_theta_R}
	Under the same conditions as Proposition \ref{prop_sigma}, with probability greater than $1-8M^{-1}$, we have 
	\[
	|\wh \Theta_{j\ell} - \Theta_{j\ell}| \le \wh \eta_{j\ell},\qquad |\wh R_{j\ell}-R_{j\ell}|\le \wh\delta_{j\ell},\qquad \text{for all $1\le  j,\ell \le p$.}
	\]
	The quantities $\wh \eta_{j\ell}$ have the same form as (\ref{est_eta}) and (\ref{est_delta}) except for replacing $\Theta_{j\ell}$, $m_j/p$ and $\u_j/p$ by $\wh \Theta_{j\ell}+\kappa_5$, $\wh m_j/p + \kappa_4$ and $\wh \u_j/p + \kappa_5$, respectively, with $\kappa_4 = O(\sqrt{\log M / N})$ and $\kappa_5= O(\sqrt{\log M / (nN)})$. Similarly, $\wh \delta_{j\ell}$ can be estimated in the same way by replacing $\eta_{j\ell}$, $(\u_j/p)^{-1}$ and $\Theta_{j\ell}$ by $\wh \eta_{j\ell}$, $(\wh\u_j/p - \kappa_5)^{-1}$ and $\wh \Theta_{j\ell}+\kappa_5$, respectively.\\
\end{cor}

\begin{proof}[Proof of Proposition \ref{prop_sigma}]
	Throughout the proof, we work on the event $\E$. Write $X = (X_1, \ldots, X_n)\in \RR^{p\times n}$ and similarly, for $\eps = (\eps_1, \ldots, \eps_n)$ and $W = (W_1,\ldots, W_n)$. We first show that  $\EE[\wh\Theta]= \Theta$.
	Recall that $X_i = AW_i+\eps_i$ satisfying 
	\[
	\EE X_i = AW_i, \quad \text{Cov}(X_i) = {1\over N_i} \diag (AW_i) - {1\over N_i} AW_iW_i^TA^T. 
	\]
	This gives
	\[
	\EE [\wh\Theta] = {1\over n}\sum_{i =1}^n\left[
	{N_i \over N_i - 1}\EE[X_iX_i^T] - {1\over N_i-1} \textrm{diag}(\EE X_i)
	\right] = {1\over n}\sum_{i =1}^n AW_iW_i^TA^T=\Theta.
	\]
	Next   we bound  the  entry-wise error rate of $\wh\Theta-\Theta$. Observe that
	\begin{align*}
	\wh\Theta &= {1\over n}\sum_{i =1}^n\left[
	{N_i \over N_i - 1}(AW_i+\eps_i)(AW_i+\eps_i)^T - {1\over N_i-1} \textrm{diag}(AW_i+\eps_i)
	\right]\\
	&= {1\over n}\sum_{i =1}^n\left[
	{N_i \over N_i - 1}(AW_iW_i^TA^T + AW_i\eps_i^T+\eps_i(AW_i)^T + \eps_i\eps_i^T)- {1\over N_i-1} \textrm{diag}(AW_i+\eps_i)
	\right]\\
	&= {1\over n}\sum_{i =1}^n\Bigg[
	AW_iW_i^TA^T +{N_i \over N_i - 1}(AW_i\eps_i^T+\eps_i(AW_i)^T ) - {\diag(\eps_i)\over N_i-1}\\
	&\qquad + {N_i \over N_i - 1}\left(\eps_i\eps_i^T - \EE[\eps_i\eps_i^T]\right)
	\Bigg].
	\end{align*}
	The third equality comes from the fact that
	\[
	\EE[\eps_i\eps_i^T] = {1\over N_i} \diag(AW_i) - {1\over N_i} AW_iW_i^TA^T.
	\]
	Recall that $\Theta = n^{-1} \sum_{i =1}^nAW_iW_i^TA^T$. We have
	\begin{align}\label{def_t1t2t3}\nonumber
	\left|\wh \Theta_{j\ell}- \Theta_{j\ell}\right| &\le  \underbrace{\left|{1\over n}\sum_{i =1}^n (AW_i\eps_i^T+\eps_i(AW_i)^T)_{j\ell}\right|}_{T_1} + \underbrace{\left|{1\over n}\sum_{i =1}^n {1\over N_i}\left(\diag(\eps_i)\right)_{j\ell}\right|}_{T_2}\\
	&\qquad +  \underbrace{\left|{1\over n}\sum_{i =1}^n \left(\eps_{ji}\eps_{\ell i} - \EE[\eps_{ji}\eps_{\ell i}]\right)\right|}_{T_3}.
	\end{align}
	It remains to bound $T_1, T_2$ and $T_3$.  Fix $1\le j, \ell \le p$. To bound $T_1$, we have
	\begin{equation}\label{eq_t1}
	T_1\ \le \ {1\over n}\left|\sum_{i =1}^n \M_{ji}\eps_{\ell i}\right|+\ {1\over n}\left|\sum_{i =1}^n \M_{\ell i}\eps_{j i}\right|
	\ \overset{\E}{\le}\ (\sqrt{m_j}+\sqrt{m_\ell})\sqrt{6\Theta_{j\ell}\log M\over npN}+{2(m_j+m_\ell) \log M \over npN}.
	\end{equation}
	For $T_2$, we have
	\begin{equation}\label{eq_t2}
	T_2\  =\  {1\over nN}\left|\sum_{i =1}^n \eps_{ji} \right|\ \overset{\E}{\le} \ 2(1+\kappa_1)\sqrt{\u_{j}\log M \over npN^3},
	\end{equation}
	if $j = \ell$. Note that $(T_2)_{j\ell}=0$ if $j\ne \ell$.
	Finally, to bound $T_3$, we obtain
	\begin{align}\label{eq_t3}\nonumber
	T_3 &\ \le\ {1\over n}\left|\sum_{i =1}^n\left(\eps_{ji}\eps_{\ell i} - \EE[\eps_{ji}\eps_{\ell i}]\right)\right|\\\nonumber 
	&\ \overset{\E}{\le}\ 
	12\sqrt{6}\sqrt{\Theta_{j\ell}+{(\u_j + \u_{\ell})\log M \over pN}}\sqrt{(\log M)^3 \over nN^2}+\kappa_2\\
	&\ \le \ 12\sqrt{6\Theta_{j\ell}(\log M)^3 \over nN^2}+12\sqrt{6(\u_j + \u_{\ell})(\log M)^4 \over npN^3}+\kappa_2.
	\end{align}
	Since (\ref{ass_signal}) implies 
	\begin{equation}\label{cond_rate_m_pN}
	{m_{\min} \over p} = \min_{1\le j\le p}\max_{1\le i\le n}\M_{ji} \ge {(3\log M)^2 \over N}
	\end{equation}
	by recalling (\ref{def_uag}), we have 
	\[
	12\sqrt{6\Theta_{j\ell}(\log M)^3 \over nN^2} + (\sqrt{m_j}+\sqrt{m_\ell})\sqrt{6\Theta_{j\ell}\log M\over npN} \le 3\sqrt{6} (\sqrt{m_j}+\sqrt{m_\ell})\sqrt{\Theta_{j\ell}\log M\over npN}.
	\] 
	In addition,
	\[
	2(1+\kappa_1)\sqrt{\u_{j}\log M \over npN^3}\1_{\{j=\ell\}} + 12\sqrt{6 (\u_j + \u_{\ell})(\log M)^4 \over npN^3}\le 31(1+\kappa_1)\sqrt{(\u_j + \u_{\ell})(\log M)^4 \over npN^3}.
	\] 
	Combining (\ref{eq_t1}) - (\ref{eq_t3}) concludes the desired rate of $\wh \Theta - \Theta$.
	
	To prove the rate of $\wh R - R$, recall that $R = (n D_\M^{-1})\Theta(nD_\M^{-1})$. Fix $1\le j,\ell \le p$. By using the diagonal structure of $D_X$ and $D_\M$, it follows
	\begin{align*}
	&\left[\left(nD_X^{-1}\right) \wh\Theta \left(nD_X^{-1}\right) - R\right]_{j\ell}\\ & =  \underbrace{n^2\left(D_X^{-1}-D_\M^{-1}\right)_{jj}\wh\Theta_{j\ell}\left(D_X^{-1}\right)_{\ell\ell}}_{T_4} + \underbrace{n^2\left(D_\M^{-1}\right)_{jj}\wh\Theta_{j\ell}\left(D_X^{-1}-D_\M^{-1}\right)_{\ell\ell}}_{T_5} \\
	& \quad + \underbrace{n^2\left(D_\M^{-1}\right)_{jj}\left(\wh\Theta-\Theta\right)_{j\ell}\left(D_\M^{-1}\right)_{\ell\ell}}_{T_6}.
	\end{align*}
	We first quantify the term $n(D_X^{-1}-D_\M^{-1})$. From their definitions, 
	\begin{eqnarray}\label{eq_wtdd}\nonumber
	n\left|\left(D_X^{-1}-D_\M^{-1}\right)_{jj}\right| &=&  n\left|{1 \over \sum_{i =1}^n \M_{ji} + \sum_{i =1}^n\eps_{ji}} - {1\over \sum_{i =1}^n \M_{ji} }\right|\\\nonumber
	&\overset{(\ref{def_uag})}{\le}& {1\over n}\left| \sum_{i =1}^n\eps_{ji} \right|  \bigg / \left({\u_j \over p}\left|{\u_j \over p}+ {1\over n}\sum_{i=1}^n\eps_{ji}\right|\right)\\
	& \overset{\E}{\le} & {2(1+\kappa_1) \over 1-\kappa_1(1+\kappa_1)}\cdot {p \over \u_j}\sqrt{ p\log M \over \u_jnN},
	\end{eqnarray}
	where the last inequality uses
	\begin{align*}	
	\left|{\u_j \over p}+{1\over n} \sum_{i =1}^n\eps_{ji} \right|  &~\overset{\E}{\ge}~ {\u_j \over p}-2(1+\kappa_1)\sqrt{ \u_j\log M \over npN}\\
	&~ =~ {\u_j\over p}\left(1 - 2(1+\kappa_1)\sqrt{p\log M \over \u_j nN} \right)\\
	& \overset{(\ref{cond_rate_pN})}{\ge}{\u_j\over p}\left( 1 -  2(1+\kappa_1)\sqrt{3 \over 2n}\right) = {\u_j \over p}(1- \kappa_1(1+\kappa_1))
	\end{align*}
	by recalling that $\kappa_1 = \sqrt{6/n}$.
	Since
	\begin{equation}\label{eq_du}
	\left(D_\M^{-1}\right)_{jj} = {1\over \sum_{i =1}^n (AW)_{ji} } = {1\over \sum_{i =1}^n\M_{ji}} ={p \over n\u_{j}},
	\end{equation}
	combined with  (\ref{eq_wtdd}), we find
	\begin{align}\label{eq_dx}\nonumber
	n\left|\left(D_X^{-1}\right)_{jj}\right| & ~\le ~ {p \over \u_{j}} \left( 1 +{2(1+\kappa_1) \over 1-\kappa_1(1+\kappa_1)}\sqrt{p\log M \over\u_{j} nN} \right)\\\nonumber &\overset{(\ref{cond_rate_pN})}{\le} \left( 1 +{\kappa_1(1+\kappa_1) \over 1-\kappa_1(1+\kappa_1)} \right){p\over \u_j}\\
	&~=~ {1\over 1-\kappa_1(1+\kappa_1)}\cdot {p\over \u_j}
	\end{align}
	Finally, since $
	|\wh\Theta_{j\ell}| \le  \Theta_{j\ell} + \left|\wh\Theta_{j\ell}- \Theta_{j\ell} \right|
	$, combining (\ref{eq_du}) and (\ref{eq_dx}) gives
	\begin{align*}
	\left|T_4\right| +\left|T_5\right|  & \le   {2(1+\kappa_1) \over (1-\kappa_1(1+\kappa_1))^2}\cdot {p^2\over \u_j\u_{\ell}}\left(\sqrt{p\over \u_j}+\sqrt{p\over \u_{\ell}}\right)\sqrt{\log M \over nN}\left(\Theta_{j\ell} + |\wh \Theta_{j\ell} - \Theta_{j\ell}|\right),
	\\
	\left|T_6\right| & =  {p^2\over \u_j\u_{\ell}}|\wh \Theta_{j\ell} - \Theta_{j\ell}|.
	\end{align*}
	Collecting these bounds for  $T_4$, $T_5$ and $T_6$ and using (\ref{cond_rate_pN}) again yield
	\[
	|(\wh R - R)_{j\ell}| \le\left(1+\kappa_1\kappa_3\right){p^2 \over \u_j \u_{\ell}}|\wh \Theta_{j\ell} - \Theta_{j\ell}| +\kappa_3 {p^2\Theta_{j\ell} \over \u_j \u_{\ell}}\left(\sqrt{p\over \u_j}+\sqrt{p\over \u_{\ell}}\right)\sqrt{\log M \over nN}.
	\]	
	with 
	\[
	\kappa_3 = {2(1+\kappa_1) \over (1-\kappa_1-\kappa_1^2)^2}.
	\]
	This completes the proof of Proposition \ref{prop_sigma}.
\end{proof}

\smallskip

\begin{proof}[Proof of Corollary \ref{cor_theta_R}]
	It suffices to show the following on the event $\E$,
	\[
	|\wh \Theta_{j\ell} - \Theta_{j\ell}| = O\left(\sqrt{\log M \over nN}\right),~ {|\wh m_j -m_j|\over p} = O\left(\sqrt{\log M\over N}\right),~ {|\wh \u_j - \u_j| \over p}= O\left(\sqrt{\log M\over nN}\right),
	\]
	for all $j,\ell \in [p]$.
	Recall that the definitions (\ref{def_uag}). Since
	\begin{equation}\label{display1}
	{\u_j\over p} = {1\over n}\sum_{i =1}^n\M_{ji} \le \max_{1\le i\le n}\M_{ji} = {m_j\over p}   \le 1, \qquad \Theta_{j\ell} = {1\over n}\sum_{i =1}^n\M_{ji}\M_{\ell i}\le 1,
	\end{equation}
	for any $j,\ell \in [p]$, display (\ref{def_eta}) implies 
	$$
	\eta_{j\ell} \le 3\sqrt{6\log M \over nN} + {2\log M \over nN} + 31(1+\kappa_1)\sqrt{2(\log M)^4 \over nN^3} + \kappa_2  = O(\sqrt{\log M/ (nN)}).
	$$
	In addition, 
	\begin{align*}
	{|\wh \u_j - \u_j| \over p} &= {1\over n}\left|\sum_{i =1}^n(X_{ji}-\M_{ji})\right| = {1\over n}\left|\sum_{i =1}^n\eps_{ji}\right|\\ &\overset{\E}{\le}  2(1+\kappa_1)\sqrt{\u_j \log M \over npN}\\
	& \overset{(\ref{display1})}{\le} 2(1+\kappa_1)\sqrt{ \log M \over nN} = O\left(\sqrt{\log M\over nN}\right).
	\end{align*}
	Finally, we show ${|\wh m_j -m_j|/ p} = O(\sqrt{\log M / N})$. Fix any $j\in [p]$ and define 
	\[
	i^* := \argmax_{1\le i\le n}\M_{ji},\qquad i':= \argmax_{1\le i\le n} X_{ji}.
	\]
	Thus, from the definitions (\ref{def_uag}) and (\ref{est_m_u}), we have
	\begin{align*}
	{|\wh m_j -m_j| \over p}  &= |X_{ji'} - \M_{ji^*}| \\
	&\le  |X_{ji'} - \M_{ji'}| + \M_{ji^*} - \M_{ji'} \\
	&\le 2|X_{ji'} - \M_{ji'}| + |X_{ji^*} - M_{ji^*}| +X_{ji^*}- X_{ji'}  \\
	&\le 2|\eps_{ji'}|+ |\eps_{ji^*}|,
	\end{align*}
	from the definition of $i'$.
	
	From the proof of Lemma \ref{lem_t4}, we conclude, on the event $\S_{ji}$, that holds with probability at least $1-2M^{-3}$,
	\[
	{|\wh m_j -m_j| \over p} \le 2|\eps_{ji'}| + |\eps_{ji^*}| \overset{\S_{ji}}{\le} 3\sqrt{6\M_{ji^*} \log M \over N} + {6\log M \over N} \overset{(\ref{display1})}{=} O\left(\sqrt{\log M\over N}\right).
	\]
	This completes the proof.
\end{proof}
\medskip

The following corollary provides the expressions of $\delta_{j\ell}$ and $\eta_{j\ell}$ under condition (\ref{ass_signal_weak}). 

\begin{cor}\label{cor_delta}
	Under model (\ref{model}) and (\ref{ass_signal_weak}), with probability greater than $1- O(M^{-1})$,
	\[
	|\wh \Theta_{j\ell} - \Theta_{j\ell}| \le c_0\eta_{j\ell},\qquad |\wh R_{j\ell} - R_{j\ell}| \le c_1\delta_{j\ell},\quad \text{	for all $1\le  j,\ell \le p$}
	\]
	for some constant $c_0, c_1>0$, where
	\begin{align}\label{def_eta_wt}\nonumber
	\eta_{j\ell} & = \sqrt{\Theta_{j\ell}\log M\over nN}\sqrt{{m_j+m_\ell \over p} \vee {\log^2M \over N} }+{2(m_j+m_\ell)\over p}{ \log M \over nN}\\
	&\qquad + \sqrt{\log^4 M \over nN^3}\sqrt{{\u_j + \u_{\ell}\over p} \vee {\log M \over N}}
	\end{align}
	and 
	\begin{equation}\label{delta_wt}
	\delta_{j\ell} := {p^2\eta_{j\ell} \over \u_j \u_{\ell}} +{p^2\Theta_{j\ell}\over \u_j\u_{\ell}}\left(\sqrt{p \over \mu_j} + \sqrt{p\over \mu_\ell}\right) \sqrt{\log M \over nN}.
	\end{equation}
\end{cor}
\begin{proof}
	We define the event $\E' :=  \E_1' \cap \E_2 \cap \E_3'$ with
	\begin{align*}
	\E_1' &= \bigcap_{j=1}^p \left\{{1\over n}\left|\sum_{i =1}^n \eps_{ji} \right| \lesssim \sqrt{\u_j\log M \over npN} \right\},\\\nonumber
	\E_2 &= \bigcap_{j,\ell=1}^p\left\{{1\over n}\left|\sum_{i =1}^n \M_{\ell i}\eps_{ji} \right| \le {\sqrt{6m_\ell\Theta_{j\ell}\log M\over npN}} + {2m_{\ell} \log M \over npN}\right\},\\\nonumber
	\E_3' &=  \bigcap_{j,\ell=1}^p \Bigg\{
	{1\over n}\left|\sum_{i =1}^n\left(\eps_{ji}\eps_{\ell i} - \EE[\eps_{ji}\eps_{\ell i}]\right) \right| \lesssim \sqrt{\Theta_{j\ell}\log^3 M \over nN^2} + \sqrt{\log^5 M \over nN^4} \\
	&\hspace{7cm}+\sqrt{(\mu_j + \mu_\ell)\log^4M \over npN^3}
	\Bigg\}.
	\end{align*}
	From  display (\ref{eq_lem_t1}) in the proof of Lemma \ref{lem_t1} and condition (\ref{ass_signal_weak}), one has $\PP(\E_1') \ge 1-O(M^{-1})$. Further invoking Lemma \ref{lem_t2} and (\ref{eq_T2}) -- (\ref{eq_TjiTelli})  in Lemma \ref{lem_t4} yields $\PP(\E') \ge  1-O(M^{-1})$. 
	
	We proceed to work on $\E'$. The bound of $|\wh \Theta_{j\ell} - \Theta_{j\ell}|$ requires upper bounds of $T_1$, $T_2$ and $T_3$ defined in (\ref{def_t1t2t3}). From (\ref{eq_t1}) -- (\ref{eq_t3}) and by invoking $\E'$, after a bit algebra, we have 
	\begin{align*}
	|\wh \Theta_{j\ell} - \Theta_{j\ell}| &\lesssim  \sqrt{\Theta_{j\ell}\log M\over nN}\sqrt{{\log^2 M \over N} \vee {m_j + m_\ell \over p}} +{(m_j+m_\ell) \log M \over npN} \\
	&\qquad + \sqrt{\log^4M \over nN^3}\sqrt{{\log M \over N} \vee {\mu_j + \mu_{\ell} \over p}}.
	\end{align*}
	Finally, the bound of $|\wh R_{j\ell} - R_{j\ell}|$ follows from the same arguments in the proof of Proposition \ref{prop_sigma} by invoking condition (\ref{ass_signal_weak}) instead of (\ref{cond_rate_pN}).

\end{proof}

\subsection{Proofs of Theorem \ref{thm_anchor}}
We start by stating and proving the following lemma which is crucial for the proof of Theorem \ref{thm_anchor}.  Recall 
\[
J_1^a:=\{j\in [p]: \wt A_{ja} \ge 1-4\delta/\nu\},\qquad \text{for all }a\in [K].
\]
Let 
\[
\wh a_i = \argmax_{1\le j\le p}\wh R_{ij},\qquad \text{for any $i\in [p]$.}
\]
\begin{lemma}\label{lem_anchor}
	Under the conditions in Theorem \ref{thm_anchor}, for any $i\in I_a$ with some $a\in [K]$, the following inequalities hold on the event $\E$:
	\begin{eqnarray}
	\label{ineq1}
	\Bigl|\wh R_{ij} - \wh R_{ik}\Bigr| &\le& \delta_{ij}+\delta_{ik}, \qquad\text{ for all }j, k\in I_a;\\
	\label{ineq2}
	\wh R_{ij} - \wh R_{ik} &>& \delta_{ij}+\delta_{ik} , \qquad\text{ for all }j\in I_a, k\notin I_a\cup J_1^a;\\\label{ineq3}
	\wh R_{ij}- \wh R_{ik}&<& \delta_{ij}+\delta_{ik}, \qquad\text{ for all }j\in J_1^a \text{ and }k\in I_a.
	\end{eqnarray}
	For any $i\in J_1^a$, we have
	\begin{equation}\label{ineq4}
	\wh R_{i\wh a_i} - \wh R_{ij} \le \delta_{i\wh a_i} + \delta_{ij},\qquad \text{for any }j\in I_a.
	\end{equation} 
\end{lemma} 
\begin{proof}
	We work on the event $\E$ so that, in particular,
	$|\wh R_{j\ell} - R_{j\ell }| \le \delta_{j\ell}$ for all $1\le j,\ell \le p$.\\ To prove (\ref{ineq1}), fix $i\in I_a$ and $j,k\in I_a$ with some $a\in [K]$. Since  $R = \wt A \wt C\wt A^T$, we have $R_{ij} = R_{ik}= \wt C_{aa}$ and 
	\[
	|\wh R_{ij} - \wh R_{ik}| \le |\wh R_{ij}-R_{ij}| + |\wh R_{ik}-R_{ik}|\overset{\E}{\le} \delta_{ij}+\delta_{ik}.
	\]
	To prove (\ref{ineq2}), fix $i, j\in I_a$ and $k\in [p]\setminus I_a$. On the one hand, 
	\begin{equation}\label{left}
	\wh R_{ik} \overset{\E}{\le} 
	\sum_{b=1}^K\wt A_{kb}\wt C_{ab}+ \delta_{ik} \overset{(\ref{def_nu2})}{\le} \wt A_{ka}\wt C_{aa}+(1-\wt A_{ka})(\wt C_{aa}-\nu)+\delta_{ik} =\wt C_{aa}-(1-\wt A_{ka})\nu+\delta_{ik}.
	\end{equation}
	On the other hand, $i,j\in I_a$ implies $R_{ij} = \wt C_{aa}$. Thus, on the event $\E$, (\ref{left}) gives
	\[
	\wh R_{ij} - \wh R_{ik} \overset{\E}{\ge} (1-\wt A_{ka})\nu - \delta_{ij}-\delta_{ik}.
	\]
	If $\wt A_{ka} = 0$, using (\ref{def_delta}) and $\nu > 4\delta$ gives
	the desired result. If $\wt A_{ka} >0$, from the definition of $J_1^a$, we have $\wt A_{ka} < 1-4\delta/\nu$ which finishes the proof by using (\ref{def_delta}) again.\\ To prove (\ref{ineq3}), observe that, for any $j\in J_1^a$ and $k\in I_a$, 
	\begin{equation*}
	\wh R_{ij} \overset{(\ref{left})}{\le} \wt C_{aa} - (1-\wt A_{ja})\nu + \delta_{ij} < \wt C_{aa}+\delta_{ij} =  R_{ik} +\delta_{ij}\overset{\E}{\le} \wh R_{ik}+\delta_{ij}+\delta_{ik}.
	\end{equation*}
	It remains to
	show (\ref{ineq4}). For any $i\in J_1^a$ and $j\in I_a$, we have, for some $c\in[K]$,
	\begin{align*}
	\wh R_{i\wh a_i} \overset{\E}{\le}  \max_{k\in[p]}R_{ik}+\delta_{i\wh a_i}
	&\overset{(\ast)}{ \le } \sum_{b=1}^K\wt A_{ib}\wt C_{bc}+\delta_{i\wh a_i}\\
	&\overset{(\ast\ast)}{\le} \sum_{b=1}^K\wt A_{ib}\wt C_{ba}+\delta_{i\wh a_i} \\
	&= R_{ij}+\delta_{i\wh a_i}  \overset{\E}{\le} \wh R_{ij}+\delta_{ij}+\delta_{i\wh a_i}.
	\end{align*}
	Inequality $(\ast)$ holds since 
	\[
	\max_{k\in[p]} R_{ik} = \max_{k\in[p]}\sum_{b =1}^K\wt A_{kb}\left( \sum_{a=1}^K\wt A_{ia}\wt C_{ab}\right) \le 
	\max_{k\in[p]}\max_{b\in[K]}\sum_{a=1}^K\wt A_{ia}\wt C_{ab} =  \sum_{a=1}^K\wt A_{ia}\wt C_{ac}.
	\]
	Inequality $(\ast\ast)$ holds, since, 
	for any $c\ne a$,  we have
	\begin{equation*}
	\sum_{b=1}^K\wt A_{ib}\wt C_{bc} \le \wt A_{ia}\wt C_{ac}+(1-\wt A_{ia})\wt C_{cc}\overset{(\ref{def_nu2})}{\le} \wt A_{ia} (\wt C_{aa}-\nu)+(1-\wt A_{ia})\wt C_{cc},
	\end{equation*} and
	\begin{equation*}
	\sum_{b=1}^K \wt A_{ib} \wt C_{ab} \overset{(\ref{def_nu2})}{\ge}\wt A_{ia}\wt C_{aa}.
	\end{equation*}
	
	\begin{equation*}
	\sum_{b=1}^K\wt A_{ib}\wt C_{ab} - \sum_{b=1}^K\wt A_{ib}\wt C_{bc} \ge \wt A_{ia}\nu - (1-\wt A_{ia})\wt C_{cc} > \nu - 2(1-\wt A_{ia})\wt C_{cc}.	\end{equation*}
	The term on the right is positive, since condition (\ref{ass_nu}) guarantees that
	\begin{equation*}
	\nu \ge \frac{8\delta}{\nu}\wt C_{cc}\ge 2(1-\wt A_{ia})\wt C_{cc},
	\end{equation*}
	where the last inequality is due to the definition of $J_1$. This concludes the proof.  
\end{proof}
Lemma \ref{lem_anchor} remains valid under the conditions of Corollary \ref{cor_anchor} in which case $J_1 = \emptyset$
and we only need $\nu > 4\delta$ to prove (\ref{ineq2}).

\begin{proof}[Proof of Theorem \ref{thm_anchor}]
	We work on the event $\E$ throughout the proof.
	Without loss of generality, we assume that the label permutation $\pi$ is the identity. We start by presenting three claims which are sufficient to prove the result.  Let $\wh I^{(i)}$ be defined in step 5 of Algorithm \ref{alg1} for any $i\in [p]$.
	\begin{enumerate}
		\item[(1)] For any $i\in J\setminus J_1$, we have $\wh I^{(i)} \notin \wh \I$.
		\item[(2)] For any $i \in I_a$ and $a\in [K]$, we have $i\in \wh I^{(i)}$, $I_a\subseteq \wh I^{(i)}$ and $\wh I^{(i)}\bl I_a \subseteq J_1^a$.
		\item[(3)] For any $i\in J_1^a$ and $a\in [K]$, 
		we have $I_a \subseteq \wh I^{(i)}$.
	\end{enumerate}
	If we can prove these claims, then (1) and the \textsc{Merge} step in Algorithm \ref{alg1} guarantees that $\wh I\cap(J\setminus J_1)    =\emptyset$ and thus enable us to focus on $i\in I\cap J_1$. For any $a\in [K]$, (2) implies that there exists $i\in I_a$ such that $I_a\subseteq \wh I_{a}$ and $\wh I_{a} \bl I_a \subseteq J_1^a$ with $\wh I^{(i)} := \wh I_a$. Finally, (3) guarantees that none of anchor words will be excluded by any $i\in J_1$ in the \textsc{Merge} step. Thus, $\wh K = K$ and $\wh \I = \{\wh I_1,\ldots,\wh I_K\}$ is the desired partition. Therefore, we proceed to prove (1) - (3) in the following. Recall that $\wh a_i := \argmax_{1\le j\le p} \wh R_{ij}$ for all $1\le i\le p$.
	
	To prove (1), let $i\in J\setminus J_1$ be fixed. We first prove that $\wh I^{(i)} \notin \wh \I$ when $\wh I^{(i)} \cap I \ne \emptyset$.  From steps 8 - 10 of Algorithm \ref{alg1}, it suffices to show that, there exists $j\in \wh I^{(i)}$ such that the following does not hold for any $k\in \wh a_j$: 
	\begin{equation}\label{keyeq2}
	\left|\wh R_{ij} - \wh R_{jk}\right|\le  \delta_{ij}+\delta_{jk}.
	\end{equation}
	Let $\wh I^{(i)} \cap I \ne \emptyset$ such that there exists $j\in I_b\cap \wh I^{(i)}$ for some $b\in [K]$. For this $j$, we have $R_{ij}=\sum_{a} \wt A_{ia} \wt C_{ab}$ and
	\begin{equation}\label{left2}
	\wh R_{ij} \overset{(\ref{left})}{\le} \wt C_{bb}-(1-\wt A_{ib})\nu+\delta_{ij}.
	\end{equation}
	On the other hand, for any $k\in \wh a_j$ and $k'\in I_b$, using the definition of $\wh a_j$ gives
	\begin{equation}\label{right}
	\wh R_{jk} \ge \wh R_{jk'} \overset{\E}{\ge} R_{jk'}+\delta_{jk'} = \wt C_{bb} - \delta_{jk'}.
	\end{equation}
	Combining (\ref{left2}) with (\ref{right}) gives
	$$
	\wh R_{jk} - \wh R_{ij} \ge (1- \wt A_{ib})\nu - \delta_{ij} - \delta_{jk'}.
	$$ 
	The definition of $J_1$ and (\ref{def_delta}) with $\nu> 4\delta$ give
	\[
	\wh R_{jk} - \wh R_{ij} > \delta_{jk} + \delta_{ij}.
	\]
	This shows that for any $i\in J\setminus J_1$, if $\wh I^{(i)}\cap I \ne \emptyset$, $\wh I^{(i)} \notin \wh \I$. Therefore, to complete the proof of (1), we show that $\wh I^{(i)} \cap I = \emptyset$ is impossible if $i\in J\setminus J_1$.
	For fixed $i\in J\setminus J_1$ and $j\in \wh a_i$, we have
	\begin{equation*}
	R_{ij}  = \sum_{b} \sum_{a}\wt A_{ia}\wt A_{jb}\wt C_{ab}
	\le  \max_{1\le b\le K} \sum_{a}\wt A_{ia}\wt C_{ab}  = \sum_{a}\wt A_{ia}\wt C_{ab^*}= R_{ik}
	\end{equation*}
	for some $b^*$ and any $k\in I_{b^*}$. 
	Therefore,  
	\begin{equation*}
	\wh R_{ij} - \wh R_{ik }  \overset{\E}{ \le } R_{ij}- R_{ik} +\delta_{ij}+\delta_{ik} \le  \delta_{ij}+\delta_{ik}
	\end{equation*}
	On the other hand, assume $\wh I^{(i)} \cap I=\emptyset$. Since $k \in I_{b^*}$, we know $k\notin \wh I^{(i)}$, which implies 
	\begin{equation*}
	\wh R_{ij} - \wh R_{ik} > \delta_{ij}+\delta_{ik},
	\end{equation*}
	from step 5 of Algorithm \ref{alg1}. The last two displays contradict each other, and
	we conclude that, for any $i\in J\setminus J_1$, $\wh I^{(i)} \cap I \ne \emptyset$. This completes the proof of (1).
	
	From (\ref{ineq2}) in Lemma \ref{lem_anchor}, given step 5 of Algorithm \ref{alg1}, we know that, for any $j\in \wh I^{(i)}$, $j\in I_a\cup J_1^a$. Thus, we write	$\wh I^{(i)} = (\wh I^{(i)} \cap I_a) \cup (\wh I^{(i)}\cap J_1^a)$. For any $j \in \wh I^{(i)} \cap I_a$, by the same reasoning, $\wh a_j$ is either $I_a$ or $J_1^a$. For both cases, since $i,j\in I_a$ and $i\ne j$, (\ref{ineq1}) and (\ref{ineq4}) in Lemma \ref{lem_anchor} guarantee that (\ref{keyeq2}) holds. On the other hand, for any $j\in \wh I^{(i)} \cap J_1^a$, (\ref{ineq3}) in Lemma \ref{lem_anchor} implies that (\ref{keyeq2}) still holds. Thus, we have shown that, for any $i\in I_a$, $i\in \wh I^{(i)}$. To show $I_a\subseteq \wh I^{(i)}$, let any $j\in I_a$ and observe that  $\wh a_i$ can only be in $I_a\cup J_1^a$. In both cases, (\ref{ineq1}) and (\ref{ineq4}) imply $j\in \wh I^{(i)}$. Thus, $I_a\subseteq \wh I^{(i)}$. Finally, $\wh I^{(i)}\bl I_a\subseteq J_1^a$ follows immediately from (\ref{ineq2}).
	
	We conclude the proof by noting that (3) directly follows from (\ref{ineq3}). 
\end{proof}

\section{Proofs of Section \ref{sec_A}}
\subsection{Proofs of Lower bounds in Section \ref{sec_lowerbound}}
\begin{proof}[Proof of Theorem \ref{thm_lb}]
	We first show the result of the matrix $\ell_1$ norm.  Let 
	\begin{equation}\label{def_A0}
	A^{(0)} = \begin{bmatrix}
	e_1 \otimes \bm{1}_{g_1} &  e_2 \otimes  \bm{1}_{g_2} & \cdots & e_K \otimes  \bm{1}_{g_K}\\
	\bm{1}_{|J|} & \bm{1}_{|J|} & \cdots & \bm{1}_{|J|}
	\end{bmatrix}\times \begin{bmatrix}
	{1\over g_1 + |J|} & &\\
	& \ddots & \\
	& &  {1\over g_K + |J|} 
	\end{bmatrix}
	\end{equation}
	with $g_k := |I_k|$ for any $1\le k \le K$ and $|g_k - g_{k+1}|\le 1$. We use $e_k$ to denote the canonical basis vectors in $\RR^K$ and use $\bm{1}_d$ and $\otimes$  to denote, respectively, the $d$-dimensional vector with entries equal to $1$ and the kronecker product.  
	We start by constructing a set of ``hypotheses'' of $A$. Assume $g_k + |J|$ is even for $1\le k\le K$. Let 
	\[
	M := \{0,1\}^{(|I|+|J|K)/2}.
	\]
	Following the Varshamov-Gilbert bound in Lemma 2.9 in \cite{np_sasha}, there exists $w^{(j)}\in M$ for $j=0,1,\ldots, T$, such that 
	\begin{equation}\label{eq_w}
	\left\|w^{(i)} - w^{(j)}\right\|_1 \ge {|I|+K|J|\over 16},\quad \text{for any } 0\le i\ne j\le T,
	\end{equation}
	with $w^{(0)} = 0$ and	
	\begin{equation}\label{eq_T}
	\log (T) \ge {\log (2) \over 16}(|I|+K|J|).
	\end{equation}
	For each $w^{(j)} \in \RR^{(|I|+K|J|)/2}$, we divide it into $K$ chunks as
	$
	w^{(j)} = \left(w^{(j)}_1,w^{(j)}_2, \ldots, w^{(j)}_K \right)
	$ 
	with $w^{(j)}_k \in \RR^{(g_k + |J|)/2}$. 
	For each $w^{(j)}_k$, we write $\wt w^{(j)}_k \in \RR^{p}$ as its augumented counterpart such that 
	$\wt w^{(j)}_k(S_k) = [w^{(j)}_k, - w^{(j)}_k]$ and $\wt w^{(j)}_k(\ell) = 0$ for any $\ell \notin S_k$, where $S_k := \textrm{supp}(A^{(0)}_k)$.
	For $1\le j\le T$, we choose $A^{(j)}$ as 
	\begin{equation}\label{eq_Aj}
	A^{(j)} = A^{(0)}+\gamma\begin{bmatrix}
	\wt w^{(j)}_1 & \cdots & \wt w^{(j)}_K
	\end{bmatrix}
	\end{equation}
	with 
	\begin{equation}\label{eq_gamma}
	\gamma = \sqrt{\log (2)\over 16^2c^2(1+c_0)}\sqrt{ K^2 \over nN(|I|+ K|J|)}
	\end{equation}
	for some constant $c_0, c>0$. Under $|I| + K|J| \le c(nN)$, it is easy to verify that $A^{(j)} \in \A(K, |I|, |J|)$ for all $0\le j\le T$.
	
	In order to apply Theorem 2.5 in \cite{np_sasha},  we need to check the following  conditions:
	\begin{itemize}
		\itemsep0.5em
		\item[(a)]  $\KL(\PP_{A^{(j)}},\PP_{A^{(0)}}) \le \log(T)/16$, for each $i= 1,\ldots, T$.
		\item[(b)] $L_{1}\left(A^{(i)},A^{(j)}\right) \ge c_1K\sqrt{(|I|+K|J|)/(nN)}$, for $0\le i<j\le T$ and some positive constant $c_1$.
		\item[(c)] $L_{1}(\ \cdot\ )$ satisfies the triangle inequality. 
	\end{itemize}
	The expression of Kullback-Leibler divergence between two multinomial distributions is shown in  \cite[Lemma 6.7]{Tracy}. For   completeness, we include it here.
	\begin{lemma}[Lemma 6.7 \cite{Tracy}]\label{lem_KL}
		Let $D$ and $D'$ be two $p\times n$ matrices such that each column of them is a weight vector. Under model (\ref{model}), let $\PP$ and $\PP'$ be the probability measures associated with $D$ and $D'$, respectively. Suppose $D$ is a positive matrix. Let 
		$$\eta = \max_{1\le j\le p, 1\le i\le n}{|D'_{ji} - D_{ji}| \over D_{ji}}$$ and assume $\eta <1$. There exists a universal constant $c_0>0$ such that 
		\[
		\KL(\PP',\PP) \le (1+c_0\eta)N\sum_{i =1}^n\sum_{j =1}^p{|D'_{ji} - D_{ji}|^2 \over D_{ji}}.
		\]
	\end{lemma}

	Fix $1\le j\le T$ and choose $D^{(j)} = A^{(j)}W$
	with 
	\begin{equation*}
	W \in\RR_+^{K\times n}=  W^0 + {1\over nN}\bm{1}_K\bm{1}_K^T - {K\over nN}\bm{I}_K
	\end{equation*}
	where $W^0$ is defined in (\ref{def_W0}).
	The above choice of $W$ perturbs $W^0$ to avoid $D_{ji} \ne 0$ for any $j\in I$ and $i\in [n]$, in the presence of anchor words. Since there exists some large enough constant $c>0$ such that $g_k \le c|I|/K$, it then follows that
	\begin{equation}\label{eq_d0}
	D^{(0)}_{\ell i} = \sum_{k =1}^KA^{(0)}_{\ell k} W_{ki} = \left\{\begin{array}{ll}
	W_{ki} / (g_k+|J|) \ge cW_{ki}/ (|I|/K + |J|)& \text{ if }\ell \in I_k, k\in [K]\\
	\sum_k W_{ki} / (g_k+|J|)\ge c/(|I| / K+|J|) & \text{ if }\ell \in J
	\end{array}\right.
	\end{equation}
	for any $1\le i\le n$, where we also use  $\sum_k W_{ki} = 1$.
	Similarly, we have 
	\begin{equation}\label{eq_djd0}
	\left|D^{(j)}_{\ell i}- D^{(0)}_{\ell i}\right| =   \gamma\left|\sum_{k =1}^K \wt w^{(j)}_k(\ell)W_{ki}\right| \le
	\left\{\begin{array}{ll}
	W_{ki} \gamma& \text{ if }\ell \in I_k, k\in [K]\\
	\gamma& \text{ if }\ell \in J
	\end{array}\right.
	\end{equation}
	for all $1\le i\le n$, 
	where $\wt w^{(j)}_k(\ell)$ denotes the $\ell$th element of $\wt w^{(j)}_k$. Thus, by choosing proper $c_0$ in (\ref{eq_gamma}) and using $|I| + K|J| \le c(nN)$, we have
	$$
	\max_{1\le \ell\le p, 1\le i\le n}{|D^{(j)}_{\ell i} - D_{\ell i}| \over D_{\ell i}} \le c\gamma (|I|/K + |J|) < 1, 
	$$
	for any $1\le j\le T$. Invoking Lemma \ref{lem_KL} gives 
	\begin{align*}
	\KL(\PP_{A^{(j)}}, \PP_{A^{(0)}}) &~\le~ \left(1+c_0\right)N\sum_{i =1}^n\sum_{\ell =1}^p{|D^{(j)}_{\ell i} - D^{(0)}_{\ell i}|^2 \over D^{(0)}_{\ell i}}\\
	&~\le ~
	c\left(1+c_0\right)N\left({|I|\over K}+|J|\right)\sum_{i =1}^n\left\{\gamma^2|J| +\sum_{k=1}^K\gamma^2g_kW_{ki}\right\}.\\
	&~\le~ c\left(1+c_0\right)nN\left(|I|+K|J| \right)^2\gamma^2/K^2\\
	&\overset{(\ref{eq_T})}{\le} {1\over 16}\log T, 
	\end{align*}
	where the second inequality uses (\ref{eq_d0}) and (\ref{eq_djd0}) and the third inequality uses $g_k \le c|I|/K$ and $\sum_k W_{ki} = 1$. This verifies (a). 
	
	To show (b), from (\ref{eq_Aj}), it gives
	\begin{align*}
	L_{1}(A^{(j)}, A^{(\ell)}) &= \sum_{k =1}^K\left\|A^{(j)}_{\cdot k}-A^{(\ell)}_{\cdot k}\right\|_1\\
	&= 2\gamma \sum_{k =1}^K\left\|w^{(j)}_k-w^{(\ell)}_k\right\|_1\\
	&= 2\gamma \left\|w^{(j)}-w^{(\ell)}\right\|_1\\
	&\overset{(\ref{eq_w})}{\ge} {\gamma \over 8}(|I| + K|J|).
	\end{align*}
	Plugging into the expression of $\gamma$ verifies (b). Since (c) is already verified in \cite[page 31]{Tracy}, invoking  \cite[Theorem 2.5]{np_sasha} concludes the proof for even $g_k + |J|$. The complementary case is easy to derive with slight modifications. Specifically, denote by $\S_{odd} := \{1\le k\le K: g_k + |J| \text{ is odd}\}$. Then we change 
	$
	M:= \{0, 1\}^{Card}
	$
	with 
	$$
	Card = \sum_{k \in \S_{odd}} {g_k + |J| - 1\over 2} + \sum_{k \in \S_{odd}^c} {g_k + |J|\over 2}.
	$$
	For each $w^{(j)}$, we write it as $w^{(j)} = (w^{(j)}_1, \ldots, w^{(j)}_K)$ and each $w^{(j)}_k$ has length $(g_k + |J|- 1) / 2$ if $k\in \S_{odd}$ and $(g_k + |J|) / 2$ otherwise. 	We then construct $
	A^{(j)}_k = A_k^{(0)} + \gamma \wt w^{(j)}_k
	$
	where $\wt w^{(j)}_k\in \RR^{p}$ is the same augumented ccounterpart of $w^{(j)}_k$. The result follows from the same arguments.\\
	
	To conclude the proof, we show the lower bound of $\|\wh A-A\|_\c$. Let $k^* := \arg\max_{k} g_k$. Then we can repeat the above arguments by only changing the $k^*$th  column of $A^{(0)}$. Specifically, assume $|J| + g_{k^*}$ is even and we draw $w^{(j)}$ from $M = \{0,1\}^{(|J|+g_{k^*})/2}$ for $1\le j\le T$ with 
	$$
	\log(T) \ge \log(2) (|J| + g_{k^*})/16.
	$$
	Let $\wt w^{(j)}$ be its augumented counterpart and choose $A^{(j)}_k = A^{(0)}_k + \gamma \wt w^{(j)}$ with 
	$$
	\gamma = \sqrt{\log (2)\over 16^2c^2(1+c_0)}\sqrt{K \over nN(g_{k^*}+ |J|)}
	$$
	and  $\|w^{(j)} - w^{(\ell)}\|_1 \ge (g_{k^*}+|J|)/16$ for any $0\le j\ne \ell \le T$.
	Then we need to check
	\begin{itemize}
		\itemsep0.5em
		\item[(a')]  $\KL(\PP_{A^{(j)}},\PP_{A^{(0)}}) \le \log(T)/16$, for each $i= 1,\ldots, T$.
		\item[(b')] $L_\c\left(A^{(i)},A^{(j)}\right) \ge c_1\sqrt{K(g_{k^*}+|J|) /( nN)}$, for $0\le i<j\le T$ and some positive constant $c_1$.
		\item[(c')] $L_\c(\ \cdot\ )$ satisfies the triangle inequality. 
	\end{itemize}
	Different from (\ref{eq_d0}) and (\ref{eq_djd0}), we have, for all $1\le i\le n$,
	\begin{equation}\label{eq_d0_1}
	D^{(0)}_{\ell i} = \sum_{k =1}^KA^{(0)}_{\ell k} W_{ki} = \left\{\begin{array}{ll}
	W_{ki} / (g_k+|J|) \ge W_{ki}/ (g_{k^*} + |J|)& \text{ if }\ell \in I_k, k\in [K]\\
	\sum_k W_{ki} / (g_k+|J|)\ge 1/(g_{k^*}+|J|) & \text{ if }\ell \in J
	\end{array}\right.
	\end{equation}
	and
	\begin{equation}\label{eq_djd0_1}
	\left|D^{(j)}_{\ell i}- D^{(0)}_{\ell i}\right|  =
	\left\{\begin{array}{ll}
	\gamma W_{k^*i}  \wt w^{(j)}_\ell\le \gamma W_{k^*i}, &  \text{ if }\ell \in I_{k^*}\cup J\\
	0, &\text{ otherwise}.
	\end{array}\right..
	\end{equation}
	Using the same arguments, one can derive
	\begin{align*}
	\KL(\PP_{A^{(j)}}, \PP_{A^{(0)}}) &~\le~ \left(1+c_0\right)N\gamma^2(g_{k^*}+|J|) \sum_{i =1}^n\left(
	W_{k^*i}^2|J|+W_{k^*i}g_{k^*}
	\right).
	\end{align*}
	Since there exists some constant $c>0$ such that 
	\[
	\sum_{i =1}^n \left(
	W_{k^*i}^2|J|+W_{k^*i}g_{k^*}
	\right) = \sum_{a=1}^n \sum_{i \in n_a} \left(
	W_{k^*i}^2|J|+W_{k^*i}g_{k^*}
	\right) \le c (|J| + g_{k^*})n / K,
	\]
	we conclude that 
	\[
	\KL(\PP_{A^{(j)}}, \PP_{A^{(0)}}) \le c(1+c_0)\gamma^2 Nn (g_{k^*}+|J|)^2 / K \le \log(T)/16.
	\]
	To show (b'), observe that
	\[
	L_\c(A^{(j)}, A^{(\ell)}) = 2\gamma  \max_{1\le k\le K}\left\|w^{(j)}_k-w^{(\ell)}_k\right\|_1 \ge {\gamma (g_{k^*} + |J|)\over 8}.
	\]
	Finally, we verify (c') by showing that $L_\c(\cdot)$ satisfies the triangle inequality. Consider $(A, \wt A, \wh A)$ and observe that
	\begin{eqnarray*}
		L_\c(A,\wt A) &=& \min_{P\in \H_K}\|AP-\wt A\|_\c~= ~\min_{P,Q\in \H_K}\|AP-\wt AQ\|_\c \\
		&\le &  \min_{P,Q\in \H_K}\left(\|AP-\wh A\|_\c+\|\wh A-\wt AQ\|_\c\right)\\
		& = & \min_{P\in \H_K}\|AP-\wh A\|_\c + \min_{Q\in \H_K}\|\wh A-\wt AQ\|_\c\\
		& = & L_\c(A,\wh A)  + L_\c(\wt A,\wh A).
	\end{eqnarray*}
	The proof is complete.
\end{proof}

\subsection{Proofs of upper bounds in Section \ref{sec_upperbound}}\label{sec_proof_upper}
We work on the event $\E$. Recall that under (\ref{ass_signal}), we have $P(\E) \ge 1-M^{-3}$.  From Theorem \ref{thm_anchor}, we have $\wh K = K$ and $\wh I = I$. Without loss of generality, we assume that the label permutation matrix $\Pi$ is the identity matrix that aligns the topic words with the estimates ($\wh \I=\I$). In particular,  this implies that  any chosen $\wh L \subset I$  has the correct partition, and  we have  $\wh L =  L$. We first give a crucial lemma to the proof of Theorem \ref{thm_rate_Ahat}. From (\ref{def_eta}), recall that
\begin{align}\label{def_eta2}
\eta_{j\ell} &\asymp  \left( \sqrt{m_j}+\sqrt{m_\ell} \right)\sqrt{\Theta_{j\ell}\log M\over npN}+{(m_j+m_\ell) \log M \over npN} + \sqrt{(\u_j + \u_{\ell})(\log M)^4 \over npN^3},
\end{align}
for all $j,\ell \in [p]$.

\begin{lemma}\label{lem_rate_lbd}
	Under the conditions of Theorem \ref{thm_rate_Ahat}, we have
	\begin{eqnarray}\label{rate_lbd}
	\max_{i\in L}\sum_{j\in L}\eta_{ij} &\lesssim &\sqrt{\oa_I^3\og} \sqrt{\log M \over np^3N},\\\label{rate_eta1}
	\max_{i\in L}\sum_{j\in J}\eta_{ij} &\lesssim  & \sqrt{\oa_I \og S_J}\sqrt{\log M\over Knp^2N},
	\end{eqnarray}
	with $S_J = |J|\oa_I + \sum_{j\in J}\alpha_j$, for any $i\in I_k$ and $k\in [K]$.
	In addition, if $\sum_{k'\ne k}\sqrt{C_{kk'}} = o(\sqrt{C_{kk}})$ for any $1\le k\le K$, then 
	\begin{equation}\label{rate_lbd_fast}
	\max_{i\in L}\sum_{j\in L}\eta_{ij} \lesssim \sqrt{\oa_I^{3}\og}\sqrt{\log M\over Knp^3N}.
	\end{equation}
\end{lemma}
\begin{proof}
	We first simplify the expression of $\eta$ defined in (\ref{def_eta}).  To show (\ref{rate_lbd}), observe that, for any $i\in I_a$, $j\in I_b$  and $a,b\in [K]$,  we have
	\begin{equation*}\label{eq_sigma_ii}
	\Theta_{ij} = {1\over n}\sum_{\ell = 1}^n\M_{i\ell}\M_{j\ell} = A_{ia}A_{jb}{1\over n}\sum_{\ell = 1}^nW_{a\ell}W_{b\ell } \overset{(\ref{def_uag})}{=} {\alpha_i\alpha_j\over p^2}\cdot{1\over n}\langle W_{a\cdot}, W_{b\cdot}\rangle.
	\end{equation*}  
	As a result, plugging (\ref{eq_mu}) - (\ref{eq_m}) and the above display into (\ref{def_eta}) yields 
	\begin{align}\label{eq_eta_pure}
	\eta_{ij} &~ \asymp ~\sqrt{{1\over n}\langle W_{a\cdot}, W_{b\cdot}\rangle}\sqrt{\oa_I^{3}\log M\over np^3N} + {\oa_I\log M\over npN}+ \sqrt{\oa_I\og(\log M)^4\over KnpN^3}.
	\end{align}
	Since
	\[
	\sum_{b=1}^K{1\over n}\langle W_{a\cdot}, W_{b\cdot}\rangle = {1\over n}\sum_{\ell=1}^{n}W_{a\ell}\sum_{b =1}^KW_{b\ell} \overset{(\ref{orig_sum_to_one})}{=} {1\over n}\sum_{\ell=1}^{n}W_{a\ell} \overset{(\ref{def_uag})}{=} {\g_a \over K},
	\]
	summing (\ref{eq_eta_pure}) over $j \in L$, and using the Cauchy-Schwarz inequality to obtain the first term, gives
	\begin{align}\label{drie}
	\max_{i\in L}\sum_{j\in L}\eta_{ij} &\ \lesssim\  \sqrt{\oa_I  \log M \over npN}\left[ \sqrt{\oa_I^2 \og \over p^2} +\sqrt{ \oa_I K^2\log M \over npN}+\sqrt{ K\og(\log M)^3\over N^2}\right].
	\end{align}
	Note that Assumption \ref{ass_pd} implies $n \ge K$. Since  (\ref{ass_signal}) and (\ref{eq_mu_I}) together with the fact $\og \ge 1\ge \ug$ give
	\begin{equation}\label{cond_rate_1}
	{\oa_I \og \over K}\ge	{\oa_I \over K}\ge {\oa_I \ug \over K} \ge \min_{i\in I}\mu_i \ge {2p\log M \over 3N},
	\end{equation}
	these two bounds imply the first term on the right in (\ref{drie}) is greater than the second term on the right of (\ref{drie}). Moreover,  (\ref{eq_m}), (\ref{cond_rate_m_pN}) and (\ref{cond_rate_1}) imply
	\[
	{\oa_I^2 \over 
		Kp^2}\ge {2\log M\over 3N}{\oa_I \over p}\ge   {2\log M\over 3N} {m_{\min} \over p} \ge {6(\log M)^3 \over N^2}.
	\]
	Thus, the first term on the right of (\ref{drie}) is also greater than the third term on the right of (\ref{drie}). This finishes the proof of (\ref{rate_lbd}). 
	
	We proceed to show (\ref{rate_eta1}). For any $i\in I_a$ and $j\in [p]$,  observe that
	\begin{equation*}\label{eq_sigma_ij}
	\Theta_{ij} = {1\over n}\sum_{\ell = 1}^n\M_{i\ell}\M_{j\ell} = A_{ia}{1\over n}\sum_{\ell = 1}^nW_{a\ell}\M_{j\ell} \overset{(\ref{def_uag})}{=} {\alpha_i\over p}\cdot{1\over n}\langle W_{a\cdot}, \M_{j\cdot}\rangle.
	\end{equation*}
	Plugging (\ref{eq_mu}) -- (\ref{eq_m}) and the above display into (\ref{def_eta}) yields
	\begin{align}\label{eq_eta_ij}
	\eta_{ij} ~ \lesssim ~\sqrt{{1\over n}\langle W_{a\cdot}, \M_{j\cdot}\rangle}\sqrt{\alpha_i (\alpha_i+\alpha_j)\log M\over np^2N} + {(\alpha_i + \alpha_j)\log M\over npN} +  \sqrt{(\alpha_i+\alpha_j)(\log M)^4\over npN^3}.
	\end{align}
	Since $\sum_{k}W_{k\ell} = 1$ and (\ref{def_uag}) yield
	\begin{align*}
	\sum_{j\in J}{1\over n}\langle W_{a\cdot}, \M_{j\cdot}\rangle & =
	\sum_{k=1}^K {1\over n}\sum_{\ell =1}^nW_{a\ell}W_{k\ell}\sum_{j\in J}A_{jk} \le {\g_a\over K}  \|A_{J}\|_{1,\i} \le {
		\g_a \over K},
	\end{align*}  
	summing (\ref{eq_eta_ij}) over $j\in J$ and using Cauchy-Schwarz yield
	\begin{align*}
	\max_{i\in L}\sum_{j\in J} \eta_{ij} & \lesssim  \sqrt{S_J\log M \over npN}\left(\sqrt{\oa_I\og\over Kp} + \sqrt{S_J\log M \over npN}+\sqrt{|J|(\log M)^3\over N^2}\right).
	\end{align*}
	To conclude the proof of (\ref{rate_eta1}), it suffices to show the first term on the right in the display above dominates  the other two. This is true by noting that
	\[
	{\oa_I \og   N^2 
		\over K|J|p(\log M)^3}  \ge   {\mu_{\min} N^2 \over |J|p (\log M)^3} \gtrsim 1
	\]
	and 
	\[
	{ \og nN \over K |J|\log M} \gtrsim 1, \qquad {\oa_I \og n N \over K \sum_{j\in J}\alpha_j \log M} \gtrsim 1
	\]
	from the same arguments to prove (\ref{rate_lbd}) together with $|J| \le p$ and $\sum_{j\in J}\alpha_j \le Kp$.

	Finally, to show (\ref{rate_lbd_fast}), invoking the additional condition and using (\ref{eq_eta_pure}) yield
	\begin{align*}
	\max_{i\in L}\sum_{j\in L}\eta_{ij}  &\lesssim \max_{a\in [K]}\sqrt{C_{aa}}\sqrt{\oa_I^{3}\log M\over np^3N} + {K\oa_I\log M\over npN}+ \sqrt{K\oa_I\og(\log M)^4\over npN^3}\\
	&\le \sqrt{\oa_I^{3}\og \log M\over Knp^3N} + {K\oa_I\log M\over npN}+ \sqrt{K\oa_I\og(\log M)^4\over npN^3}
	\end{align*}
	where we use $C_{aa}\le n^{-1}\|W_{a\cdot}\|_1 \lesssim \og/K$ to derive the second inequality. Repeating the same arguments, one can show that on the right-hand-side the first term dominates the second and third terms. This finishes the proof.
\end{proof}

\subsubsection{Proof of Theorem \ref{thm_rate_Ahat}}\label{sec_proof_thm_rate}
On the event $\E$, Theorem \ref{thm_anchor} guarantees that $\wh I = I$, $\wh J= J$ and $\wh L = L$. We work on
this event for the remainder of the proof.
Our proof consists of three parts for any $k\in [K]$: To obtain the error rate of (1)   $\|\wh B_{I k}-B_{I k}\|_1$,
(2)  $\|\wh B_{J k}-B_{J k}\|_1$ and (3)   $\|\wh A_{\cdot k} - A_{\cdot k}\|_1$.\\

For step (1), recall (\ref{eq_iden_BI}) and (\ref{est_BI}). Then, for any $1\le k\le K$, it follows 
\begin{align*}
\|\wh B_{Ik} - B_{Ik}\|_1 = \sum_{i \in I_k} \left|{\bigl\|X_{i\cdot}\bigr\|_1 \over \bigl\|X_{ i_k\cdot}\bigr\|_1} - {\bigl\|\M_{i\cdot}\bigr\|_1 \over \bigl\|\M_{ i_k\cdot}\bigr\|_1}\right|
\end{align*}
and we obtain
\begin{align*}
&\|\wh B_{Ik} - B_{Ik}\|_1\\
&~ = ~ {1\over \|X_{i_k\cdot}\|_1\|\M_{i_k\cdot}\|_1} \sum_{i\in  I_k}\Bigl|\|\M_{i_k\cdot}\|_1\|X_{i\cdot}\|_1-\|X_{i_k\cdot}\|_1\|\M_{i\cdot}\|_1\Bigr|\\
&~\le ~ {1\over \|X_{i_k\cdot}\|_1\|\M_{i_k\cdot}\|_1} \left[\|\M_{i_k\cdot}\|_1\sum_{i\in  I_k}\Bigl|\|X_{i\cdot}\|_1-\|\M_{i\cdot}\|_1\Bigr|+ \Bigl|\|X_{i_k\cdot}\|_1-\|\M_{i_k\cdot}\|_1\Bigr|\sum_{i\in  I_k}\|\M_{i\cdot}\|_1\right]\\
&~=~ {1\over \|X_{i_k\cdot}\|_1}\sum_{i\in  I_k}
\frac1n \left| \sum_{j=1}^n \eps_{ij} \right| 
+ {\sum_{i\in  I_k}\|\M_{i\cdot}\|_1\over \|X_{i_k\cdot}\|_1\|\M_{i_k\cdot}\|_1}\frac1n \left| \sum_{j=1}^n \eps_{i_kj} \right|.
\end{align*}
In the last step we 		used  that $X_i$ and $\M_i$ have nonnegative entries only so that indeed
\begin{align*}\Bigl|\|X_{i\cdot}\|_1-\|\M_{i\cdot}\|_1\Bigr| &= \frac{1}{n} \left| \sum_{j=1}^n| X_{ij} | - \sum_{j=1}^n |\M_{ij}| \right|\\
&= \frac{1}{n} \left| \sum_{j=1}^nX_{ij}  -\sum_{j=1}^n \M_{ij} \right|\\
&= \frac{1}{n} \left| \sum_{j=1}^n \eps_{ij} \right|.
\end{align*}
Invoking
Lemma \ref{lem_t1}, $ n^{-1} \left| \sum_{j=1}^n \eps_{ij} \right|  
\lesssim \sqrt{\u_ip\log M / (nN)}$ on the event $\E_1\supset \E$.
Note further that, for any $i\in [p]$,  
$${1 \over \|X_{i\cdot}\|_1} =  (D_X)^{-1}_{ii} \overset{(\ref{eq_dx})}{\lesssim} {p \over n\u_i}$$
and $\|\M_{i\cdot}\|_1 = n\u_i / p$ by the definition in (\ref{def_uag}). Now deduce that
\begin{align*}
\|\wh B_{Ik} - B_{Ik}\|_1	&~\lesssim ~ {1\over \u_{i_k}} \sum_{i\in  I_k}\sqrt{\u_ip\log M \over nN}+{\sum_{i\in  I_k}\u_i \over \u_{i_k}^2}\sqrt{\u_{i_k}p\log M \over nN}.
\end{align*}
Recall that $\u_i = \alpha_i\g_k / K$ for any $i\in I_k$ in (\ref{eq_mu_I}). We further have
\begin{equation*}
{1\over \u_{i_k}} \sum_{i\in  I_k}\sqrt{\u_ip\log M \over nN}+{\sum_{i\in  I_k}\u_i \over \u_{i_k}^2}\sqrt{\u_{i_k}p\log M \over nN} =  {1\over \alpha_{i_k}}\sqrt{pK\log M \over \g_knN}\left[\sum_{i\in  I_k}\sqrt{\alpha_i}+{\sum_{i\in  I_k}\alpha_i \over \sqrt{\alpha_{i_k}}}\right]
\end{equation*}
and we conclude that, on the event $\E$, 
\begin{align}\label{rate_Delta1}
\|\wh B_{Ik} - B_{Ik}\|_1~ \lesssim  ~ {\sum_{i\in I_k}\alpha_i \over \alpha_{i_k}\sqrt{\ua_I\g_k}}\sqrt{pK\log M \over nN}.
\end{align}

We proceed to show   step (2). Recall that $\wh \Omega = (\wh \omega_1, \ldots,  \wh \omega_K)$ is the optimal solution from (\ref{obj_omega}), $\wh B_J = \bigl(\wh \Theta_{JL}\wh \Omega\bigr)_+$ and $B_J =  \Theta_{JL}\Omega$. Fix any $1\le k\le K$ and denote the canonical basis vectors in $\RR^K$ by $e_1,\ldots,e_K$. Since $B$ has only non-negative entries,
we have
\begin{align*}
&\|\wh B_{J k} - B_{J k}\|_1\\ &= \|\bigl(\wh \Theta_{JL}\wh\Omega_{\cdot k} \bigr)_+ -   \Theta_{JL}\Omega_{\cdot k}\|_1\\
&\le \|\wh \Theta_{JL}\wh\Omega_{\cdot k} - \Theta_{JL}\Omega_{\cdot k}\|_1\\ 
&\le \|(\wh \Theta_{JL} - \Theta_{JL})\wh\Omega_{\cdot k}\|_1 +  \|\Theta_{JL}\wh\Omega_{\cdot k}- B_{J k}\|_1\\
&\le \|\wh\Omega_{\cdot k}\|_1\|\wh \Theta_{JL} - \Theta_{JL}\|_\c + \|B_J\Theta_{LL}\wh\Omega_{\cdot k}- B_Je_k\|_1\\ 
&\le  \|\wh\Omega_{\cdot k}\|_1\|\wh \Theta_{JL} - \Theta_{JL}\|_\c + \left[\|\wh \Theta_{LL}\wh \Omega_{\cdot k}- e_k\|_1+ \|(\wh \Theta_{LL}- \Theta_{LL})\wh \Omega_{\cdot k}\|_1\right]\|B_J\|_\c\\
&\le\|\wh\omega_k\|_1\|\wh \Theta_{JL} - \Theta_{JL}\|_\c + \left[\|\wh \Theta_{LL}\wh \Omega_{\cdot k}- e_k\|_1+ \|\wh \omega_k\|_1\|\wh \Theta_{LL}- \Theta_{LL}\|_\r\right]\|B_J\|_\c.
\end{align*}
We first study the property of $\wh \Omega$. Notice that $\omega_k := \Omega_{\cdot k}$ is feasible of (\ref{obj_omega}) since, on the event $\E$,
\[
\|\wh \Theta_{LL}\omega_k- e_k\|_1 \le \|\omega_k\|_1 \|\wh \Theta_{LL} - \Theta_{LL}\|_\r \overset{\E}{\le}  C_0\|\omega_k\|_1 \max_{i\in L}\sum_{j\in L}\eta_{ij} = \|\omega_k\|_1 \lambda,
\]
by Proposition \ref{prop_sigma}.
The optimality and feasibility of $\wh\omega_k$ imply
\begin{equation*}
\|\wh\omega_k\|_1 \le \|\omega_k\|_1,\qquad
\| \wh \Theta_{LL}\wh \omega_k - e_k\|_1 \le \|\wh \omega_k\|_1 \lambda \le \|\omega_k\|_1 \lambda.
\end{equation*} 
Hence, on the event $\E$, we have
\[
\|\wh B_{J k} - B_{J k}\|_1 \le C_0\|\omega_k\|_1 \max_{i\in L}\sum_{j\in J}\eta_{ij} + 2\|\omega_k\|_1\lambda \|B_J\|_\c.
\]
Recall that $B_J = A_JA_L^{-1}$ and $A_L = \diag(\alpha_{i_1}/p, \ldots, \alpha_{i_K}/p)$, and deduce
\begin{align}\label{eq_B_l1}
\|B_J\|_\c = \max_{1\le k\le K}{p \over \alpha_{i_k}}\sum_{j\in J}A_{jk} \le {p \over \ua_I}\|A_J\|_{1,\i}
\end{align}
Recall that $\lambda = \max_{i\in L}\sum_{j\in L}\eta_{ij}$ and notice that
\begin{equation}\label{eq_omega_inv}
\|\omega_k\|_1 = \|\Omega_{\cdot k}\|_1 = \|\Theta_{LL}^{-1}e_k\|_1  = \|A_L^{-1}C^{-1}A_L^{-1}e_k\|_1 \le {p^2 \over  \alpha_{i_k}\ua_I}\|C^{-1}\|_\r.
\end{equation}
Collecting (\ref{eq_B_l1}) -- (\ref{eq_omega_inv}) gives
\begin{equation*}
\|\wh B_{Jk}-B_{Jk}\|_1
~\le~ C_0{p^2 \over  \alpha_{i_k}\ua_I}\|C^{-1}\|_\r\left[  \max_{i\in L}\sum_{j\in J}\eta_{ij} + 2p{\|A_J\|_{1,\i} \over \ua_I}\max_{i\in L}\sum_{j\in L}\eta_{ij}\right].
\end{equation*}
Invoking (\ref{rate_lbd}) -- (\ref{rate_eta1}) in Lemma \ref{lem_rate_lbd} yields, on the event $\E$,
\begin{align}\label{rate_B_1_inf}\nonumber
\|\wh B_{Jk}-B_{Jk}\|_1
&\lesssim {p^2\|C^{-1}\|_\r \over  \alpha_{i_k}\ua_I}\left[\sqrt{\oa_I \og S_J}\sqrt{\log M\over Knp^2N}+{\|A_J\|_{1,\i}\over \ua_I}\sqrt{\oa_I^3\og} \sqrt{\log M \over npN} \right]\\
& \le  {p\over \alpha_{i_k}}\|C^{-1}\|_\r {\oa_I \over \ua_I} \sqrt{\og\log M\over KnN}\left[\sqrt{|J| + \sum_{j\in J}{\alpha_j  \over \oa}} +  {\oa_I \over \ua_I}\sqrt{K\sum_{j\in J}{\alpha_j \over \oa_I}}\right]
\end{align}
where we recall that $S_J = |J|\oa_I + \sum_{j\in J}\alpha_j$ and use $\|A_J\|_{1,\i} \le 1$, $\|A_J\|_{1,\i}\le \sum_{j\in J}\alpha_j / p$ in the second inequality.
Combining (\ref{rate_Delta1})  and (\ref{rate_B_1_inf}) yields, on the event $\E$, 
\begin{align}\label{eq_rate_B}\nonumber
\|\wh B_{\cdot k} - B_{\cdot k}\|_1 &\lesssim {p\over\alpha_{i_k}} {\sum_{i\in I_k}\alpha_i \over \sqrt{\ua_I\g_k}}\sqrt{K\log M \over npN}.\\
&\quad +{p\over \alpha_{i_k}}\|C^{-1}\|_\r {\oa_I \over \ua_I} \sqrt{\og\log M\over KnN}\left[\sqrt{|J| + \sum_{j\in J}{\alpha_j  \over \oa}} +{\oa_I \over \ua_I}\sqrt{K\sum_{j\in J}{\alpha_j \over \oa_I}}\right].
\end{align}

It remains to prove step (3). For any $k\in[ K]$, from (\ref{est_A}), we have
\begin{align*}
|\wh A_{jk} - A_{jk}|  &= \left|{\wh B_{jk} \over \|\wh B_{\cdot k}\|_1} -  { B_{jk} \over \| B_{\cdot k}\|_1}\right|
\\
& \le  {\wh B_{jk} \over \|\wh B_{\cdot k}\|_1 \|B_{\cdot k}\|_1}\left| \|\wh B_{\cdot k}\|_1 -  \|B_{\cdot k}\|_1\right| + {|\wh B_{jk} - B_{jk}| \over \|B_{\cdot k}\|_1}\\
& = {\left| \|\wh B_{\cdot k}\|_1 -  \|B_{\cdot k}\|_1\right|\over \|B_{\cdot k}\|_1}\wh A_{jk} + {|\wh B_{jk} - B_{jk}| \over \|B_{\cdot k}\|_1}
\end{align*}
by using $\wh A_{jk} =\wh B_{jk} / \|\wh B_{\cdot k}\|_1$ in the last equality.
Since $\|\wh A_{\cdot k}\|_1 = 1$, summing over $j\in[p]$ yields 
\begin{align*}
\|\wh A_{\cdot k} - A_{\cdot k}\|_1 &\le   {2\|\wh B_{\cdot k} - B_{\cdot k}\|_1 \over \|B_{\cdot k}\|_1} = {2\alpha_{i_k} \over p} \|\wh B_{\cdot k} - B_{\cdot k}\|_1.
\end{align*}
The  equality uses $\|B_{\cdot k}\|_1 = p/\alpha_{i_k}$ from $B = AA_L^{-1}$. Invoking (\ref{eq_rate_B}) concludes the proof.\qed

\subsubsection{Proof of Corollary \ref{cor_opt_rate}}
To prove Corollary \ref{cor_opt_rate}, we need the following lemma which controls $\|C^{-1}\|_\r$. 
\begin{lemma}\label{lem_C_inv}
	Let $\og$ and $\ug$ be defined in (\ref{def_alpha_gamma}).
	Assume $\og \asymp \ug$ and $\sum_{k'\ne k}\sqrt{C_{kk'}} = o(\sqrt{C_{kk}})$ for any $1\le k\le K$. Then $\|C^{-1}\|_\r = O(K)$.
\end{lemma}
\begin{proof}
	From the definition of $(\r)$-norm and the symmetry of $C$, one has 
	\[
	\|C^{-1}\|_\r = \sup_{v\ne 0}{\|C^{-1}v\|_1 \over \|v\|_1} = \sup_{v\ne 0}{\|v\|_1 \over \|Cv\|_1}  = \left(
	\inf_{ \|v\|_1 = 1} \|Cv\|_1
	\right)^{-1}.
	\]
	It suffices to lower bound $\inf_{ \|v\|_1 = 1} \|Cv\|_1$. We have 
	\begin{align*}
	\inf_{ \|v\|_1 = 1} \|Cv\|_1 &= \inf_{ \|v\|_1 = 1} \sum_{a=1}^K\left|C_{aa}v_a + \sum_{b\ne a}C_{ab}v_b\right|\\
	&\ge  \inf_{ \|v\|_1 = 1}\sum_{a=1}^K\left(C_{aa}|v_a| - \sum_{b\ne a}C_{ab}|v_b|\right)\\
	&= \inf_{ \|v\|_1 = 1}\sum_{a=1}^KC_{aa}|v_a| - \sum_{b=1}^K|v_b|\sum_{a\ne b}C_{ab}.
	\end{align*}
	Note that 
	$\sum_{k'\ne k}\sqrt{C_{kk'}} = o(\sqrt{C_{kk}})$ implies 
	\begin{equation}\label{disp_C_offdiag}
	\sum_{k'\ne k}{C_{kk'}}\le \left(\sum_{k'\ne k}\sqrt{C_{kk'}}\right)^2 = o(C_{kk}) ,\quad \text{for any }k\in [K]
	\end{equation}
	and $C_{kk} \le n^{-1}\|W_{k\cdot}\|_1 = \g_k / K = O(1/K)$ by using $\og\asymp \ug\asymp1$. We further obtain
	\[
	\inf_{ \|v\|_1 = 1} \|Cv\|_1 \gtrsim \inf_{ \|v\|_1 = 1}\sum_{a=1}^K\left(C_{aa}|v_a| - |v_b|o(1/K)\right) \gtrsim  \min_a C_{aa}.
	\]
	Observe that 
	\[
	{1\over K} - C_{kk}\asymp {\g_k \over K} - C_{kk} = \sum_{k' = 1}^KC_{kk'} - C_{kk}=\sum_{k'\ne k}C_{kk'}\le \sum_{k'\ne k}\sqrt{C_{kk'}} = o(C_{kk}).
	\]
	This implies $C_{kk} \asymp 1/K$ which concludes $\inf_{ \|v\|_1 = 1} \|Cv\|_1\gtrsim 1/K$ and completes the proof.
\end{proof}

\begin{proof}[Proof of Corollary \ref{cor_opt_rate}]
	From (\ref{rate_lbd_fast}) in Lemma \ref{lem_rate_lbd}, under  $\sum_{k'\ne k}\sqrt{C_{kk'}} = o(\sqrt{C_{kk}})$ for any $1\le k\le K$, one has
	\[
	\max_{i\in L}\sum_{j\in L}\eta_{ij} \lesssim \sqrt{\oa_I^{3}\og \log M\over Knp^3N}
	\]
	which improves the rate in (\ref{rate_lbd}) by a factor of $\sqrt{1/K}$. Repeating the previous arguments for proving Theorem \ref{thm_rate_Ahat}, one can show that 
	\begin{align*}
	\text{Rem}(J, k) & ~\lesssim~  \sqrt{K\log M \over nN}\cdot {\og^{1/2}\|C^{-1}\|_\r  \over K}\cdot {\oa_I \over \ua_I} \left(\sqrt{|J| + \sum_{j\in J}\rho_j} +{\oa_I \over \ua_I}\sqrt{\sum_{j\in J}\rho_j}\right).
	\end{align*}
	Invoking conditions (i) -- (ii) of Corollary \ref{cor_opt_rate} together with $\|C^{-1}\|_\r = O(K)$ from Lemma \ref{lem_C_inv}, one can obtain
	\[
	\sum_{k=1}^K Rem(I,k) \lesssim \sum_{i\in I}\sqrt{\alpha_i K\log M \over npN} \le K\sqrt{|I| \log M \over nN}
	\]
	by using Cauchy-Schwarz inequality with $\sum_{i\in I}\alpha_i \le \sum_{i=1}^p\alpha_i \le pK$,
	and 
	\[
	Rem(J,k) \lesssim \sqrt{|J|K \log M \over nN}.
	\]
	This yields the optimal rate of $L_1(A, \wh A)$. For $L_{\c}(A, \wh A)$, observe that $\sum_{i\in I_k}\alpha_i / p = \sum_{i\in I_k}A_{ik} \le 1$. This yields
	\[
	Rem(I, k) \lesssim \sum_{i\in I_k}\sqrt{\alpha_i K\log M \over npN}\le \sqrt{|I_k|K\log M \over nN},
	\]
	which, in conjunction with the rate of $Rem(J,k)$, concludes the result of $L_{\c}(A, \wh A)$. 
\end{proof}

\subsubsection{Proof of the statment in Remark \ref{rem-noise-control}}\label{sec_proof_rem_noise_control}
We prove the result of Theorem \ref{thm_rate_Ahat} when condition (\ref{ass_signal}) is replaced by (\ref{ass_signal_weak}) and (\ref{ass_signal_weak_mu}). First note that, similar as (\ref{cond_rate_pN}) and (\ref{cond_rate_m_pN}), condition (\ref{ass_signal_weak_mu}) implies 
\begin{equation}\label{cond_rate_pN_I}
\min_{j\in I}{\u_j \over p} = \min_{j\in I}{1\over n}\sum_{i =1}^n \M_{ji} \ge {c\log M \over N},\quad 
\min_{j\in I}{m_j \over p} = \min_{j\in I}\max_{1\le i\le n}\M_{ji} \ge {c'(\log M)^2 \over N}.
\end{equation}	
We will work on the event $\E'$ defined in the proof of Corollary \ref{cor_delta} which holds with probability greater than $1-O(M^{-1})$ under condition (\ref{ass_signal_weak}). To prove Theorem \ref{thm_rate_Ahat}, observe that its proof in Section \ref{sec_proof_thm_rate} only uses the result of  Lemma \ref{lem_t1}, Lemma \ref{lem_rate_lbd} and (\ref{eq_dx}). Since Lemma \ref{lem_t1} and (\ref{eq_dx}) is guaranteed by $\E'$, it suffices to prove that Lemma \ref{lem_rate_lbd} holds for $\eta_{j\ell}$ defined in (\ref{def_eta_wt}). This is indeed true since the only different terms in the expression of $\eta_{j\ell}$ in (\ref{def_eta_wt}) than (\ref{def_eta2}) are 
\[
{m_j + m_{\ell} \over p} \vee {\log^2 M\over N},\qquad 	{\u_j + \u_{\ell} \over p} \vee {\log M\over N}
\]
for any $j\in I_k$ and $\ell \in [p]$, and hence invoking  (\ref{cond_rate_pN_I}) yields 
\[
{m_j + m_{\ell} \over p} \vee {\log^2 M\over N}\lesssim {m_j + m_{\ell} \over p},\quad 	{\u_j + \u_{\ell} \over p} \vee {\log M\over N}\lesssim 	{\u_j + \u_{\ell} \over p}.
\]
Therefore, Lemma \ref{lem_rate_lbd} follows and so does Theorem \ref{thm_rate_Ahat}. \qed 

\section{Discussions of Assumptions \ref{ass_pd} and \ref{ass_w}}\label{app_assumps}
\subsection{Examples}\label{app_ex}
The following example shows that  Assumption \ref{ass_pd} does not imply Assumption \ref{ass_w}.

\begin{example}
	
	In the first two instances (\ref{app_eg_1}) and (\ref{app_eg_2}) below, Assumptions \ref{ass_pd} and \ref{ass_w} both hold. These instances give insight into 
	why  Assumption \ref{ass_w} may fail, while Assumption \ref{ass_pd} holds, which is shown in the third instance (\ref{app_eg_3}).

	We consider the following matrix
	\begin{equation}\label{app_eg_1}
	W = \begin{bmatrix}
	0.5 & 0.4 & 0 & 0 &  \bm{0.4} \\ 
	0.2 & 0.6 & 0.5 & 0.5 &  \bm{0} \\ 
	0.3 & 0 & 0.5 & 0.5 &  \bm{0.1}\\ 
	\bm{0} & \bm{0} & \bm{0} & \bm{0} & \bm{0.5} \\ 
	\end{bmatrix}\quad \Longrightarrow \quad 
	\wt C = \begin{bmatrix}
	0.34 & 0.15 & 0.1 & 0.31 \\ 
	0.15 & 0.28 & 0.22 & 0 \\ 
	0.1 & 0.22 & 0.31 & 0.07 \\ 
	0.31 & 0 & 0.07 & 1 \\ 
	\end{bmatrix}.
	\end{equation}
	Clearly, each diagonal entry of $\wt C$ dominates the non-diagonal ones, row-wise.  From (\ref{def_nu2}), Assumption \ref{ass_w} holds.  Assumption \ref{ass_pd} can be verified as well.  
	We see that topic 4 is {\em rare}, in that it only occurs in document 5. However, the probability that the rare topic occurs is not small ($W_{45} = 0.5$). Since each column sums to $1$, the closer  $W_{45}$ is to $1$, the smaller the other entries in the $5$th column must be, and the more uncorrelated topic $4$ will be with other topics. Hence, the larger $W_{45}$, the more likely it is that Assumption \ref{ass_w} holds.  Suppose the rare topic $4$ also has a small probability, for instance
	\begin{equation}\label{app_eg_2}
	W = \begin{bmatrix}
	0.5 & 0.4 & 0 & 0 & \bm{0.3} \\ 
	0.2 & 0.6 & 0.5 & 0.5 & \bm{0.2} \\ 
	0.3 & 0 & 0.5 & 0.5 & \bm{0.4} \\ 
	\bm{0} & \bm{0} & \bm{0} & \bm{0} & \bm{0.1} \\ 
	\end{bmatrix}\quad \Longrightarrow\quad 
	\wt C = \begin{bmatrix}
	0.35 & 0.17 & 0.13 & 0.25 \\ 
	0.17 & 0.24 & 0.19 & 0.1 \\ 
	0.13 & 0.19 & 0.26 & 0.24 \\ 
	0.25 & 0.1 & 0.24 & 1 \\ 
	\end{bmatrix}.
	\end{equation}
	While the rare topic $4$ has small non-zero entry $W_{45} = 0.1$,   none of $W_{i5}$ for $i = 1, 2, 3$,  is dominating and  Assumption \ref{ass_w} is still met. However, it fails in the following scenario:
	\begin{equation}\label{app_eg_3}
	W = \begin{bmatrix}
	0.5 & 0.4 & 0 & 0 & \bm{0.9} \\ 
	0.2 & 0.6 & 0.5 & 0.5 & \bm{0} \\ 
	0.3 & 0 & 0.5 & 0.5 & \bm{0} \\ 
	\bm{0} & \bm{0} & \bm{0} & \bm{0} & \bm{0.1}
	\end{bmatrix}\quad \Longrightarrow \quad 
	\wt C = \begin{bmatrix}
	\bm{0.38} & 0.1 & 0.06 & \bm{0.5} \\ 
	0.1 & 0.28 & 0.24 & 0 \\ 
	0.06 & 0.24 & 0.35 & 0 \\ 
	\bm{0.5} & 0 & 0 & 1 \\ 
	\end{bmatrix}.
	\end{equation}
	Here, a frequent topic (topic $1$) has probability  $W_{15}=0.9$, that dominates all other entries in the fifth column. Topic $4$ is hard to detect since it occurs with probability 0.1, while the frequent topic $1$ occurs with probability 0.9, in the single   document that may reveal  topic 4.
	We emphasize that $W$ in all three cases above satisfy Assumption \ref{ass_pd}.\\
\end{example} 

We give an example below to show that Assumption \ref{ass_w} does not imply Assumption \ref{ass_pd}, that is, $W$ satisfies Assumption \ref{ass_w} but does not have rank $K$. 
\paragraph{Example 3} Let $K = n = 5$ and consider 
\[
W = \begin{bmatrix}
0.3 & 0.05 & 0.3 & 0 & 0 \\ 
0.25 & 0.04 & 0.25 & 0.32 & 0.1 \\ 
0 & 0.5 & 0.15 & 0.1 & 0.3 \\ 
0.055 & 0.257 & 0.13 & 0.466 & 0.28 \\ 
0.395 & 0.153 & 0.17 & 0.114 & 0.32 \\ 
\end{bmatrix}.
\]
Note that $W_{4\cdot} = -0.9 W_{1\cdot} + 1.3 W_{2\cdot} + 0.5 W_{3\cdot}$ which implies $\textrm{rank}(W) < 5$. However, it follows that
\[
\wt C = n\wt W\wt W^T = \begin{bmatrix}
\bm{2.16} & 1.22 & 0.51 & 0.44 & 1.18 \\ 
1.22 & \bm{1.3} & 0.59 & 1.02 & 0.98 \\ 
0.51 & 0.59 & \bm{1.69} & 1.12 & 0.87 \\ 
0.44 & 1.02 & 1.12 & \bm{1.35} & 0.83 \\ 
1.18 & 0.98 & 0.87 & 0.83 & \bm{1.22} \\ 
\end{bmatrix}.
\]
Therefore, Assumption 3 is satisfied as it is equivalent with 
$\wt C_{ii} \wedge \wt C_{jj} > \wt C_{ij}$ for any $i, j\in [K]$, whereas  $W$ does not have  full rank.

\subsection{Statements and proofs of results invoked in  Remark \ref{remark_nu}}

We first   show  the equivalence between $\nu > 4\delta$ and 
\begin{equation}\label{eq_ass_W^prime}
\cos\left( \angle( W_{i\cdot},   W_{j\cdot} ) \right)< \left(\frac{\zeta_i}{\zeta_j} \wedge  \frac{\zeta_j}{\zeta_i}\right)\left(1 -  {4\delta \over n\zeta_i  \zeta_j}
\right),\quad \text{for all $1\le i \ne j\le K$}.
\end{equation}
\begin{lemma}\label{lem_fact}
	Let $\nu$ be defined in (\ref{def_nu2}) with $\wt C = n\wt W\wt W^T$ and $\wt W = D_W^{-1} W$. Then,
	the inequality  $\nu > 4\delta$ is equivalent with (\ref{eq_ass_W^prime}).
	In particular, $\nu > 0$ is equivalent with Assumption \ref{ass_w}.
\end{lemma}
\begin{proof}
	Starting with the definition of $\nu$, we have 
	\begin{align*}
	& \nu > 4\delta\\
	& \iff  {n\|W_{i\cdot}\|_2^2 \over \|W_{i\cdot}\|_1^2} \wedge {n\|W_{j\cdot}\|_2^2 \over \|W_{j\cdot}\|_1^2}  -{n \langle W_{i\cdot}, W_{j\cdot}\rangle \over \|W_{i\cdot}\|_1\|W_{j\cdot}\|_1} > 4\delta, \quad \textrm{for all }i\ne j \\
	&\iff n\zeta_i^2  \wedge n\zeta_j^2  - n\zeta_i \zeta_j \cos(\angle (W_{i\cdot}, W_{j\cdot})) > 4\delta, \qquad \textrm{for all }i\ne j
	\end{align*}
	by using $\zeta_i = \|W_{i\cdot}\|_2 / \|W_{i\cdot}\|_1$.
	The result follows after rearranging terms, and taking $\delta=0$ we obtain the second assertion as well. 
\end{proof}

The following lemma provides an upper bound of $\delta$ defined in (\ref{def_delta}). 
\begin{lemma}\label{lem_delta}
	Assume condition (\ref{ass_signal}) holds. Then, there exists some constant $c>0$ such that
	\[
	\delta  ~\le~ c\left\{\max_{\ell \in [p]}{m_\ell \over \u_\ell }\sqrt{1\over n} + \sqrt{\log M \over n}\right\} := \eps_n.
	\]
	If, in addition, $\log M = o(n)$ and condition (\ref{cond_word_balance}) hold, we have 
	\[
	{\delta \over n \min_i \zeta_i^2} = o(1)
	\]
	with $\zeta_i = \|W_{i\cdot}\|_2 / \|W_{i\cdot}\|_1$.
\end{lemma}
\begin{proof}
	First, we notice that $\delta = \max_{1\le j\le p}\delta_{jj}$. For any $1\le j\le p$, the expressions in (\ref{def_eta}) -- (\ref{delta}) yield
	$$
	\delta_{jj} \asymp  {p^2\over \mu_j^2}\left\{ \sqrt{m_j\Theta_{jj}\log M\over npN}+{m_j\log M \over npN} + \sqrt{\u_j(\log M)^4 \over npN^3} + \Theta_{jj}\sqrt{p\log M \over \mu_j n N}\right\}.
	$$
	Using
	\begin{equation}\label{eq_Theta_jj}
	\Theta_{jj} = {1\over n}\sum_{i = 1}^n\M_{ji}^2\le {1\over n}\sum_{i = 1}^n\M_{ji} \max_{1\le i\le n}\M_{ji} \overset{(\ref{def_uag})}{=} {\mu_j m_j \over p^2},
	\end{equation}
	together with (\ref{cond_rate_pN}), we obtain 
	\begin{align}\label{eq_delta_jj}\nonumber{}
	\delta_{jj} &\lesssim   {m_j \over \u_j}\sqrt{p\log M\over \mu_jnN}+{m_jp\log M \over \mu_j^2 nN} + \sqrt{p^3(\log M)^4 \over \mu_j^3nN^3}\\
	&\lesssim {m_j \over \u_j}\sqrt{1\over n} + \sqrt{\log M \over n}
	\end{align}
	Taking the maximum over $ j\in [p]$ concludes the proof of the upper bound of $\delta$. The second part follows by observing that 
	$$
	\min_{1\le i \le K}n\zeta_i^2  = \min_{1\le i \le  K}{n\|W_{i\cdot}\|_2^2
		\over  \|W_{i\cdot}\|_1^2}\ge 1,
	$$
	by the Cauchy-Schwarz inequality.
\end{proof}

\subsection{The Dirichlet distribution of $W$}\label{app_dir}
Suppose the columns of $W = (W_1, \ldots, W_n)$ are i.i.d. samples from the Dirichlet distribution parametrized by the vector $d = (d_1, \ldots, d_K)\in \RR^K$. Let $d_0 := \sum_k d_k$. It is well known that, for any $1\le k\le K$,
\begin{equation}\label{dir_mean_var}
a_k := \EE[W_{ki}] = {d_k \over d_0}, \qquad \sigma_{kk}^2 := Var(W_{ki}) = {d_k (d_0 - d_k) \over d_0^2 (d_0 + 1)}
\end{equation}
and, for any $1\le k\ne \ell \le K$,
\begin{equation}\label{dir_sec_moment}
s_{kk} := \EE[W_{ki}^2] = {d_k(1+d_k) \over d_0(1+d_0)},\qquad s_{k\ell} := \EE[W_{ki}W_{\ell i}] = {d_k d_\ell \over d_0(1+ d_0)}.
\end{equation}
%
Further denote 
$$
\wh a_k := {1\over n}\sum_{i = 1}^nW_{ki}, \quad \wh s_{k\ell } = {1\over n}\sum_{i = 1}^n W_{ki}W_{\ell i},\quad \text{for any } 1\le k,\ell \le K
$$
and the event 
\begin{equation}\label{def_event_dir}
\E_{dir} = \left\{
|\wh a_k - a_k| \le \eps_1^k, \quad 
|\wh s_{k\ell} - s_{k\ell} | \le \eps_2^{k\ell },\quad \text{for all }1\le k, \ell \le K
\right\}.
\end{equation}
From Lemma \ref{lem_fact}, condition $\nu > 4\delta$ discussed in Remark \ref{remark_nu} is equivalent with
\begin{equation}\label{eq_sample_cond}
{\wh s_{kk} \over \wh a_k^2} \wedge {\wh s_{\ell\ell } \over \wh a_\ell^2} - {\wh s_{k\ell }  \over \wh a_k \wh a_\ell } > 4\delta,\quad \text{for any $1\le k<\ell \le K$}.
\end{equation}
In particular, Assumption \ref{ass_w} corresponds to $\delta = 0$ in (\ref{eq_sample_cond}). The following lemma provides sufficient conditions that guarantee Assumption \ref{ass_w} and $\nu > 4\delta$. Recall that $M := n \vee p \vee \max_{i\in [n]}N_i$.

\begin{lemma}\label{lem_dir}
	Assume the columns of $W$ are i.i.d. from a Dirichlet distribution with parameter $d = (d_1, \ldots, d_K)$. Set $d_0 := d_1 + \cdots + d_K$, $\od := \max_k d_k$ and $\ud := \min_k d_k$.
	%
	Assumption \ref{ass_w} holds with probability greater than $1- O(M^{-1})$,  provided	\begin{equation}\label{cond_dir}
	{1\vee \od \over \ud} d_0(1+d_0)  \le c{n \over \log M}
	\end{equation}
	holds, 
	for some constant $c>0$ sufficiently small. 
	Additionally, if
	\begin{equation}\label{cond_dir_popu}
	{\od(1+d_0)\over \ud } \le c\sqrt{n \over \log M},
	\end{equation}
	holds
	for some constant $c>0$ small enough,
	then  $\PP\{\nu > 4\delta\} \ge 1- O(M^{-1})$.
\end{lemma}	
\begin{proof}
	Display (\ref{eq_delta_jj}) together with 
	$$
	{m_j \over p} = \max_{1\le t\le n}\M_{jt} \le \|A_{j\cdot}\|_1,\quad {\mu_j \over p}  = {1\over n}\sum_{t=1}^n \M_{jt} \ge \|A_{j\cdot}\|_1 \min_k {1\over n}\|W_{k\cdot}\|_1 = \|A_{j\cdot}\|_1 \min_k \wh a_k$$
	from (\ref{eq_mu}) and (\ref{eq_m}) yield
	\begin{equation}\label{bd_delta}
	\delta \lesssim {1\over \min_k\wh a_k}\sqrt{1\over n}+ \sqrt{\log M \over n}.
	\end{equation}
	%
	We work on the event $\E_{dir}$ with $\eps_1^k$ and $\eps_{2}^{k\ell}$ chosen as 
	\begin{equation}\label{rate_eps}
	\eps_1^k = c_1\sqrt{d_k \over d_0(1+d_0)}\sqrt{\log M  \over n}, \quad  \eps_2^{k\ell} =  c_2\left\{\sqrt{d_kd_{\ell}\over d_0(1+d_0)}+\sqrt{ \log M \over n}\right\}\sqrt{\log M \over n},
	\end{equation}
	for any $1\le k\le \ell \le K$ where $c_1, c_2$ are some positive constants such that, from Lemma \ref{lem_dir_error}, $\PP(\E_{dir}) \ge 1-O(M^{-1})$. Note that condition (\ref{cond_dir}) and $\E_{dir}$ imply
	\begin{align}\label{bd_ak}
	\wh a_k &\le a_k +\eps_1^k \le {d_k \over d_0} +  c_1\sqrt{d_k \over d_0(1+d_0)}\sqrt{\log M  \over n} \le c_1' a_k;\\\label{low_bd_ak}
	\wh a_k &\ge a_k -\eps_1^k \le {d_k \over d_0} - c_1\sqrt{d_k \over d_0(1+d_0)}\sqrt{\log M  \over n} \ge c_1'' a_k
	\end{align}
	for all $k\in [K]$.
	In the following, we prove (\ref{eq_sample_cond}) by only showing
	\begin{equation}\label{eq_ratio}
	{\wh s_{kk} \over  \wh a_k^2} - {\wh s_{k\ell} \over \wh a_k \wh a_\ell} =  {\wh s_{kk} \wh a_\ell - \wh s_{k\ell}\wh a_k \over \wh a_k^2 \wh a_\ell}> 4\delta 
	\end{equation}
	since the other part can be deduced by using similar arguments. We first show that,  on the event $\E_{dir},$ condition (\ref{cond_dir}) guarantees 
	\[
	{\wh s_{kk} \wh a_\ell - \wh s_{k\ell}\wh a_k \over \wh a_k^2 \wh a_\ell} \ge  \rho\cdot { d_0 \over d_k (1+d_0)} 
	\]
	for some constant $\rho \in [0, 1)$. We   bound the numerator of the term on the left from below by
	\begin{align*}
	\wh s_{kk} \wh a_\ell - \wh s_{k\ell}\wh a_k &~\ge~ \left(s_{kk} -\eps_2^{kk}\right)\left(a_\ell-\eps_1^{\ell}\right) - \left(s_{k\ell }+ \eps_2^{k\ell}\right)\left( a_k+\eps_1^k\right)\\
	&~=~ s_{kk}a_\ell - s_{k\ell}a_k   - \eps_2^{kk}c_1'a_\ell- \eps_1^{\ell} s_{kk} - \eps_2^{k\ell}c_1'a_k - - \eps_1^ks_{k\ell}\\
	&~=~ {d_k d_\ell \over d_0^2(1+d_0)}- c_1'\eps_2^{kk}{d_\ell \over d_0}- \eps_1^{\ell} {d_k(1+d_k) \over d_0(1+d_0)} - c_1'\eps_2^{k\ell}{d_k \over d_0} -  \eps_1^k {d_kd_\ell \over d_0(1+d_0)}.
	\end{align*}
	We used (\ref{bd_ak}) in the second line and used (\ref{dir_mean_var}) -- (\ref{dir_sec_moment}) in the third line. We then show that
	\[
	\max\left\{c_1'\eps_2^{kk}{d_\ell \over d_0}, ~\eps_1^{\ell} {d_k(1+d_k) \over d_0(1+d_0)}, ~ c_1'\eps_2^{k\ell}{d_k \over d_0},~ \eps_1^k {d_kd_\ell \over d_0(1+d_0)}\right\} \le {d_k d_\ell \over 5d_0^2(1+d_0)}.
	\]
	From (\ref{dir_mean_var}) -- (\ref{dir_sec_moment}), it suffices to show the individual inequalities
	\[
	5c_1'\eps_2^{kk} < {d_k \over d_0(1+d_0)},\quad  5\eps_1^\ell < {d_l \over d_0 (1+d_k)},\quad  5c_1'\eps_2^{k\ell} < {d_{\ell} \over d_0(1+d_0)}, \quad5\eps_1^k  < {1\over d_0}.
	\]
	These inequalities follow by invoking condition (\ref{cond_dir}) and plugging in the expression of $\eps_1^k$ and $\eps_2^{k\ell}$ in (\ref{rate_eps}). We thus have 
	\[
	\wh s_{kk} \wh a_\ell - \wh s_{k\ell}\wh a_k \ge {1\over 5}{d_kd_\ell \over d_0^2(1+d_0)}.
	\]
	In conjunction with (\ref{bd_ak}), we conclude that 
	\[
	{\wh s_{kk} \wh a_\ell - \wh s_{k\ell}\wh a_k \over \wh a_k^2 \wh a_\ell} \ge {1\over 5(c_1')^3} {d_0 \over d_k (1+d_0)} 
	\]
	which further yields
	\[
	{\wh s_{kk} \over \wh a_k^2} \wedge {\wh s_{\ell\ell } \over  \wh a_\ell^2} - {\wh s_{k\ell }  \over \wh a_k \wh a_\ell }\ge {1\over 5(c_1')^3} {d_0 \over \od (1+d_0)}. 
	\]
	When $\delta = 0$, the first statement follows. To prove the statement of $\nu > 4\delta$, 
	it suffices to show 
	\[
	{1\over 5(c_1')^3} {d_0 \over \od (1+d_0)} > 4\delta.
	\]
	This follows by invoking (\ref{cond_dir_popu}), (\ref{bd_delta}) and (\ref{low_bd_ak}). The proof is complete.
\end{proof}

\begin{remark}
	For a symmetric Dirichlet distribution, that is, $d_1 = \cdots = d_K = d$, condition (\ref{cond_dir_popu}) and (\ref{cond_dir}) becomes, respectively, 
	$$
	1 + Kd \le c\sqrt{n\over \log M},\qquad 
	(1\vee d)K(1+Kd)\le c {n\over \log M}.$$ 
	The larger $K$ is, the smaller $d / n$ needs to be. Note that small $d$ encourages the sparsity of $W$. 
\end{remark}

\begin{lemma}\label{lem_dir_error}
	Under condition (\ref{cond_dir}), assume columns of $W$ are i.i.d. sample from the Dirichlet distribution. Then $\PP(\E_{dir}) \ge 1- M^{-1}$ with $\eps_1^k$ and $\eps_2^{k\ell}$ chosen as (\ref{rate_eps}). 
\end{lemma}
\begin{proof}[Sketch of the proof]
	Since $W_{ki}$ and $W_{ki}W_{\ell i}$ are bounded by $1$ and displays (\ref{dir_mean_var}) -- (\ref{dir_sec_moment}) imply $Var(W_{ki}W_{\ell i}) \le d_k d_{\ell} / (d_0 (1+d_0))$, using Bernstein's inequality in Lemma \ref{lem_bernstein} with similar arguments in Lemmas \ref{lem_t1} and \ref{lem_t4} together with condition (\ref{cond_dir}) yield the desired result.
\end{proof}

\section{Cross-validation procedure for selecting $C_1$ in Algorithm \ref{alg2}}
We give a practical cross-validation procedure for selecting the proportionally constant $C_1$ used in Algorithm \ref{alg2}. Starting with a chosen fine grid $\mathcal{C}$, we randomly split  the data into a training set $\mathcal{D}_1$ and a validation set $\mathcal{D}_2$. For each $c \in \mathcal{C}$,  we apply Algorithm \ref{alg2} with $T=1$, $C_0 = 0.01$ and $C_1 = c$ on $\mathcal{D}_1$ to obtain $\wh A_c$ and $\wh C_c$, and then calculate the out-of-sample error
\[
L(c) :=   \left\| \wh\Theta^{(2)} - \wh A_c \wh C_c \wh A_c^T \right\|_1.
\]
Here $\wh\Theta^{(2)}$ is obtained as in (\ref{est_Theta}) by using $\mathcal{D}_2$ and $\|\cdot \|_1$ denotes the matrix $\ell_1$ norm. The selected $\wh C_1$ is defined as 
\[
\wh C_1 = \arg\min_{c\in \mathcal{C}} L(c).
\]
Using   estimates $\wh I$ and $\wh A$ and motivated by the structure $\Theta_{II} = A_I C A_I^T$, we can estimate $C$ by 
\[
\wh C = (\wh A_{\wh I}^T \wh A_{\wh I})^{-1} \wh A_{\wh I}^T \wh \Theta_{II} \wh A_{\wh I}(\wh A_{\wh I}^T \wh A_{\wh I})^{-1}.
\]

\section{Additional simulation results}\label{app_sim}
\subsection*{Sensitivity to topics number $K$}\label{sec_sim_K}
We study the performance of \textsc{Recover-L2} and \textsc{Recover-KL} for different $K$.
We show that even in a favorable low dimensional setting ($p = 400$ and true $K_0 = 15$) with $|I_k| = 1$, $\xi = 1/p$ and large sample sizes 
($N = 800$, $n = 1000$), the estimation error is seriously affected by a wrong choice of thenumber of topics  $K$.

We generated $50$ datasets according to our data generating mechanism and applied \textsc{Recover-L2}, \textsc{Recover-KL} by using different $K$ to each dataset to obtain $\wh A_K$. To quantify the estimation error, we   use the criterion
\[
\bigl\|\wh A_K\wh A_K^T - AA^T\bigr\|_1
\]
to evaluate  the overall fit of the word by word membership matrix  $AA^T\in \RR^{p\times p}$. We averaged this loss over $50$ datasets. To further benchmark the result, we use a random guessing method \textsc{Uniform} which randomly draws $p\times K$ entries from the Uniform$(0,1)$ distribution and normalizes each column to obtain an ``estimate'' that is independent of the data.
The performance of \textsc{Recover-L2}, \textsc{Recover-KL} and \textsc{Uniform} is shown in Figure \ref{fig_AA}. It clearly shows that both \textsc{Recover-L2} and \textsc{Recover-KL} are very sensitive indeed to correctly specified $K$. When $K$ differs from the true value $K_0$ by more than $2$ units, the performance is close to random guessing! This phenomenon continues to hold for various settings and we conclude that  specifying $K$ is critical for both \textsc{Recover-L2} and \textsc{Recover-KL}. 
\begin{figure}[ht]
	\centering
	\includegraphics[width = .5\textwidth]{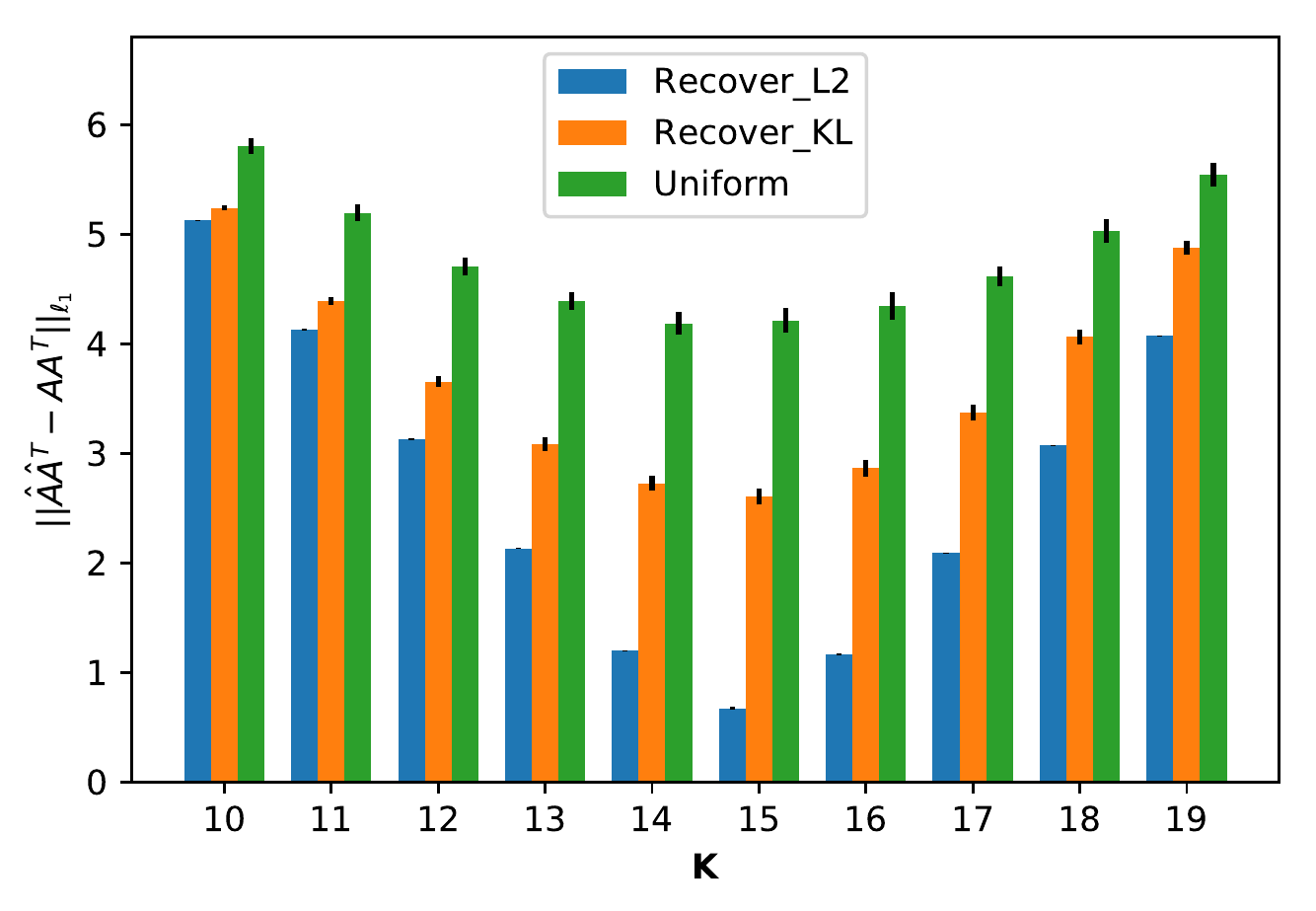}
	\caption{Plot of $\|\wh A\wh A^T - AA^T\|_1$ for using different specified $K$. The black vertical bars are standard errors over 50 datasets.}
	\label{fig_AA}
\end{figure}
\subsection*{Comparison of different algorithms for small $K$}\label{sec_sim_smallK}
We compare  \textsc{Top10}, \textsc{Recover-L2}, \textsc{Recover-KL} and \textsc{T-score}  in  the benchmark setting $N = n = 1500$, $p = 1000$ and $\xi = 1/p$, but for small values of $K$.   \textsc{T-score} is the procedure of \cite{Tracy}, who kindly made the code available to us. Since $K$ is small, the cardinality of each column $W_i$ of $W$ is randomly sampled from $\{1, 2, 3\}$.

In the first experiment, we considered $K = 5$ and varied $|I_k|$ within $\{2, 3, 4, 5, 6\}$. The   estimation error $L_1(\wh A,A) /K$, averaged  over $50$ generated datasets, is shown in Figure \ref{fig_small_K}. We see that the performance of \textsc{Recover-KL} is as good as \textsc{Top10} and even better for $|I_k| \le 3$. However, as already verified in Section \ref{sec_sim}, \textsc{Recover-KL} has the worst performance for large $K$ and is computationally demanding. The plot also demonstrates that \textsc{T-score} needs at least $4$ anchor words per group in order to have comparable performance with \textsc{Top10} and \textsc{Recover-KL}. This is as expected since $|I_k|$ needs to grow as $p^2\log^2(n)/(nN)$ in \cite{Tracy}.

In the second experiment, we set $|I_k| = p/100$ and varied $K$ from $\{5, 6, \ldots, 10\}$. We considered this range of values as, unfortunately, in the implementation of \textsc{T-score} the authors made available to us,  it often crashes  for $K=11$. The average overall $L_1$ estimation error  in Figure \ref{fig_small_K} shows that \textsc{Top10} has the smallest error over all $K$. \textsc{T-score} has similar performance when $K \le 8$ but becomes worse than \textsc{Top10} when $K$ becomes larger. 

\begin{figure}[ht]
	\centering
	\begin{tabular}{cc}
		\includegraphics[width=.35\textwidth]{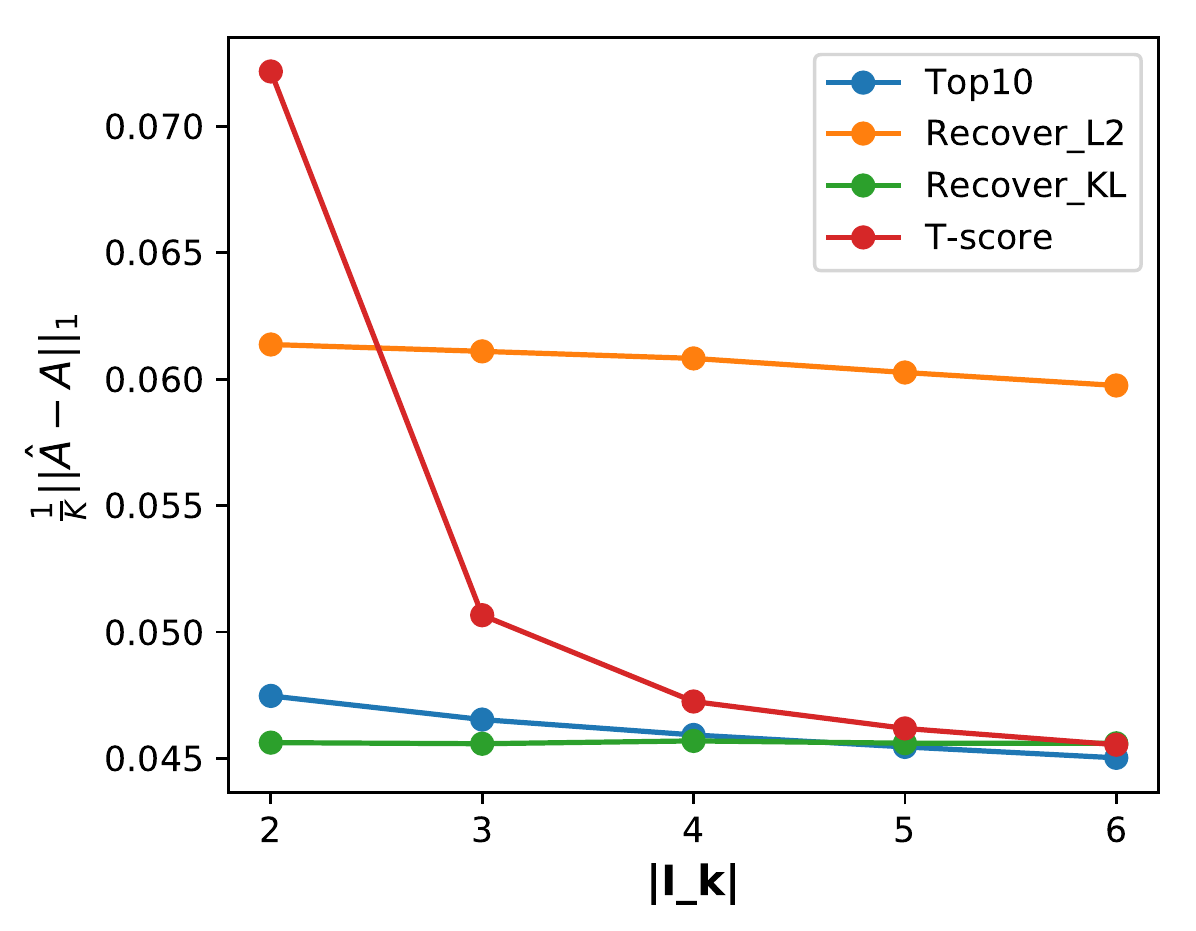} & \hskip 15pt
		\includegraphics[width=.35\textwidth]{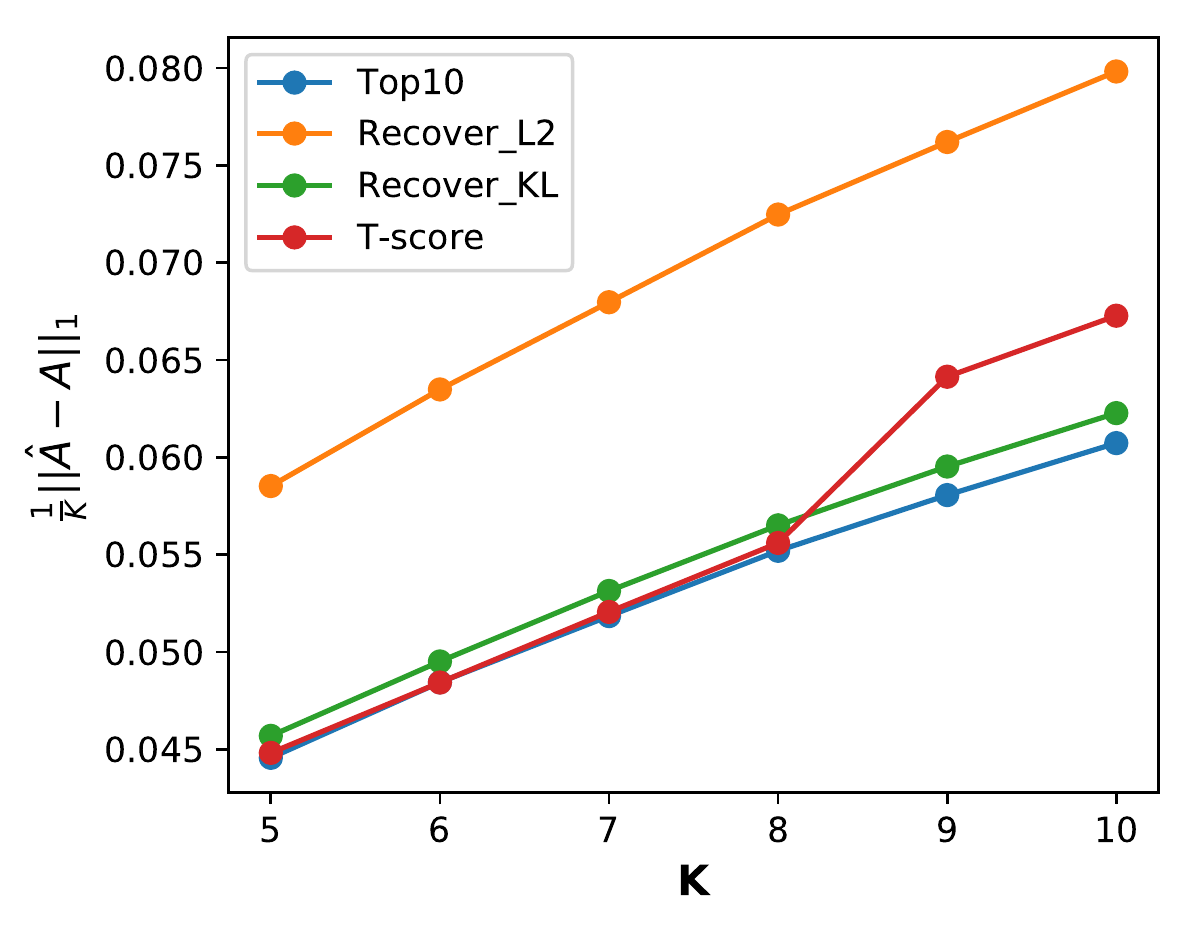}
	\end{tabular}
	\caption{The left plot is the averaged \emph{overall estimation error} by varying $|I_k|$ and setting $K = 5$. The right one is for varying $K$ and setting $|I_k| = p/100$.}
	\label{fig_small_K}
\end{figure}

\subsection*{The New York Times (NYT) dataset}
Similar to  the procedure for preprocessing the NIPs dataset, we removed common stopping words and rare words occurring in less than 150 documents. The dataset after preprocessing has $n = 299419$, $p = 3079$ and $N = 210$. Since \cite{arora2013practical} mainly considers $K = 100$, we tune our method to obtain a  similar number of topics for our comparative study, for  both the semi-synthetic data and the real data.

\paragraph{Semi-synthetic data comparison.}
To generate the semi-synthetic data, we first apply {\sc Top1} to the preprocessed dataset with $C_0 = 0.01$ and $C_1 = 15.5$ \footnote{The value of $C_1$ is chosen such that the estimated number of topics is around $100$} and obtain the estimated word-to-topic matrix $\wt A$ with $\wh K = 101$ and $206$ anchor words. Using this $\wt A$, we generate $W$ from the previous three distributions (a) -- (c), separately, with $n=\{40000, 50000, 60000, 70000\}$, and then generate $X$ based on $\wt A W$ with $N = 250$. For each setting, we generate $15$ datasets and 
the averaged {\it overall estimation error} and {\it topic-wise estimation error} of {\sc Top}, {\sc Recover-L2}, {\sc Recover-KL} and {\sc LDA} are reported in Figure \ref{fig_nyt}. The running time of different algorithm is shown in Table \ref{table_time_nyt}. We specify the true number of topics for {\sc Recover-L2}, {\sc Recover-KL} and {\sc LDA} while {\sc Top} estimates it. 

The results have essentially the same pattern as those from the NIPs dataset. For the Dirichlet distribution with parameter 0.03, {\sc Recover-KL, Recover-L2} and {\sc Top} are comparable, whereas {\sc Top} outperforms the other two when the topics become more correlated. {\sc LDA} has the worst performance, especially when $n$ is relatively small. {\sc Top} and {\sc Recover-L2} run much faster than {\sc Recover-KL} and {\sc LDA}.

\begin{figure}[ht]
	\centering
	\includegraphics[width = \textwidth]{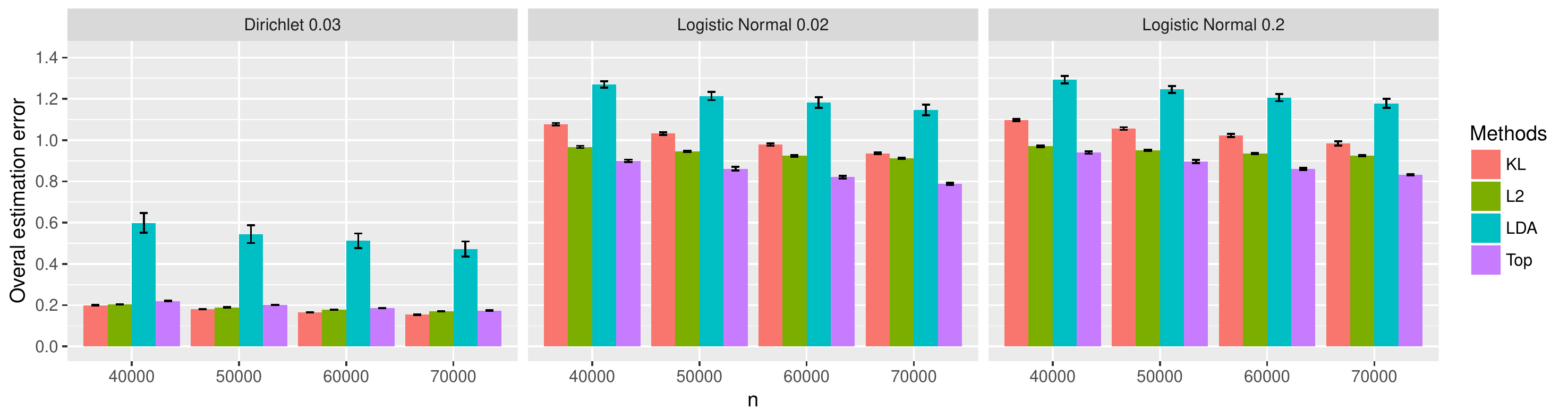}
	\includegraphics[width = \textwidth]{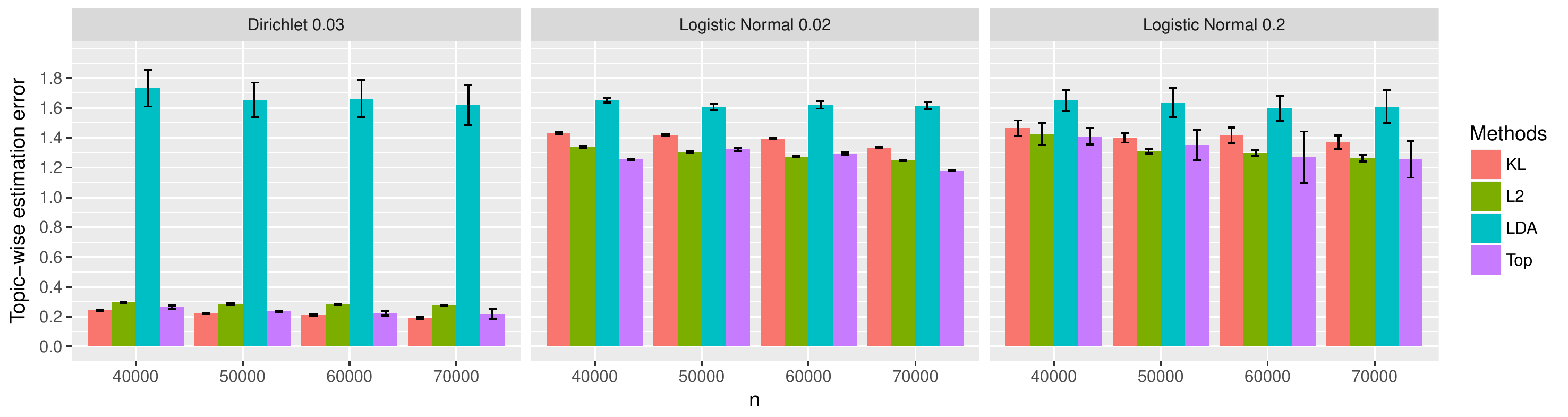}
	\caption{Plots of averaged \emph{overall estimation error} and \emph{topic-wise estimation error} of {\sc Top}, {\sc Recover-L2} (L2), {\sc Recover-KL} (KL) and LDA from the NYT dataset. {\sc Top} estimates $K$ while the other algorithms use the true $K$ as input. The bars denote one standard deviation.}
	\label{fig_nyt}
\end{figure}

\begin{table}[ht]
	\centering
	\caption{Running time (seconds) of different algorithms on the NYT dataset}
	\label{table_time_nyt}
	\begin{tabular}{ccccc}
		\hline
		& {\sc Top} & {\sc Recover-L2} & {\sc Recover-KL} & {\sc LDA} \\ 
		\hline
		$n = 40000$ & 471.5 & 795.6 & 4681.9 & 35913.7 \\ 
		$n = 50000$ & 517.2 & 841.5 & 4824.0 & 44754.1 \\ 
		$n = 60000$ & 555.3 & 894.1 & 4945.3 & 53833.8 \\ 
		$n = 70000$ & 579.2 & 957.2 & 5060.8 & 62548.8 \\ 
		\hline
	\end{tabular}
\end{table}

\paragraph{Real data comparison from the NYT dataset.}
We applied {\sc Top}, {\sc Recover-L2}, {\sc Recover-KL} and {\sc LDA} to the preprocessed NYT dataset. In order to ensure that all algorithms make use of a similar number of topics, {\sc Top} uses $C_0 = 0.01$ and $C_1 = 15.5$, which gives an estimated  $\wh{K} = 101$, and the other three algorithms use $K = 101$ as input. Evaluation of algorithms on the real data is difficult since we do not know the ground truth. However, there exists some standard ways of evaluating the fitted model. Since our main target is the word-topic matrix $A$, we consider two widely used metrics of evaluating the semantic quality of estimated topics \footnote{We do not consider the metric of {\em held-out-probability} since it requires the estimation of both $A$ and $W$ while neither {\sc Top} nor {\sc Recover} estimates $W$.} \citep{arora2013practical, Mimno-coherence,stevens-coherence}:
\begin{enumerate}
	\item[(1)] {\em Topic coherence.} 
	Coherence is a measure of the co-occurrence of the high-frequency words for each individual estimated topics. ``This metric has been shown to correlate well with human judgments of topic quality. When the topics are perfectly recovered,  all the high-probability words in a topic should co-occur frequently, otherwise, the model may be mixing unrelated concepts.'' \citep{arora2013practical}. Given a set of words $\mathcal{W}$, coherence is calculated as
	\begin{equation}\label{def_coh}
	Coherence(\mathcal{W}) := \sum_{w_1,w_2 \in \mathcal{W}}\log {D(w_1,w_2)+\eps \over D(w_2)}
	\end{equation}
	where $D(w_1)$ denotes the number of documents that the word $w_1$ occurs, $D(w_1, w_2)$ denotes the number of documents that both word $w_1$ and word $w_2$ occur \citep{Mimno-coherence} and the parameter $\eps = 0.01$ is used to avoid taking the log of zero \citep{stevens-coherence}. 
	
	In the NYT dataset, for each algorithm, we use its estimated word-topic matrix $\wh A$ and calculate $Coherence(\mathcal{W}_k)$ for $1\le k\le \wh K$, where $\mathcal{W}_k$ is the set of 20 most frequent words in topic $k$. The averaged coherence metrics and its standard deviation across topics of {\sc Top}, {\sc Recover} and {\sc LDA} are reported in Table \ref{table_coh}.\footnote{We only report one of {\sc Recover-L2} and {\sc Recover-KL} since they have the same results.} {\sc Top} has the largest coherence suggesting a better topic recovery while {\sc Recover} has the smallest coherence.

	\item[(2)] {\em Unique words:} Since coherence only measures the quality of each individual topic, we consider the metric, {\em unique words} \citep{arora2013practical}, which reflects the redundancy between topics. For each topic and its $T$ most probable words, we count how many of those $T$ words do not appear in any $T$ most probable words of the other topics.  Some overlap across topics is expected due to semantic ambiguity, but lower numbers of unique words indicate less useful models. We consider $T = 100$ and the averaged unique words and the standard deviation across topics are reported in Table \ref{table_coh}. {\sc Top} and {\sc LDA} have more unique words than {\sc Recover}. In fact, the unique words from {\sc Recover} are only the anchor words and the other most probable words of different topics are all overlapped, indicating the redundancy between the estimated topics. 
	
\end{enumerate}

\begin{table}[ht]
	\centering
	\caption{Coherence and unique words of {\sc Top, Recover} and {\sc LDA} on the NYT dataset. The numbers in the parentheses are the standard deviations across topics.}
	\label{table_coh}
	\renewcommand*\arraystretch{1.3}{
		\begin{tabular}{llll}
			\hline
			Metric & {\sc Top} & {\sc Recover} & {\sc LDA}\\\hline
			{\em Coherence}    & -328.8(51.3) & -464.1(3.8) & -340.3(44.8) \\
			{\em Unique words} & 6.7(4.4) & 1.0(0.0) & 6.3(3.1) \\
			\hline
	\end{tabular}}
\end{table}

\end{document}